\documentclass[a4,11pt]{article}

\usepackage[latin1]{inputenc}

\usepackage{latexsym}
\usepackage{amssymb}
\usepackage{amsmath}   
\usepackage{epsfig}

\usepackage[colorlinks,pagebackref,linktocpage]{hyperref}

\usepackage{amsthm} 
\newtheorem{theorem}{Theorem}
\newtheorem{proposition}[theorem]{Proposition}

\newtheorem{definition}{Definition}

\usepackage{color}
\definecolor{orange}{RGB}{255,127,0}
\definecolor{brown}{RGB}{150,70,0}
\definecolor{green}{RGB}{127,255,127}
\definecolor{darkgreen}{RGB}{0,127,0}
\definecolor{blue}{RGB}{127,127,255}
\definecolor{lightblue}{RGB}{150,150,255}
\definecolor{darkblue}{RGB}{0,0,127}
\definecolor{red}{RGB}{255,90,90}
\definecolor{grey}{RGB}{127,127,127}
\definecolor{pink}{RGB}{255,180,180}

\usepackage{xspace}	

\newcommand{\xaxis}{$x$-axis\xspace}
\newcommand{\yaxis}{$y$-axis\xspace}

\newcommand{\Haction}{{h}}                
\newcommand{\Hpolicy}{{\hbar}}            

\newcommand{\Rmax}{{R_{\text{max}}}}
\newcommand{\Rmin}{{R_{\text{min}}}}
\newcommand{\Rmean}{{R_{\text{mean}}}}

\newcommand{\argmin}{{\text{argmin}}}
\newcommand{\qpos}{{q_{\text{pos}}}}
\newcommand{\qneg}{{q_{\text{neg}}}}
\newcommand{\Npos}{{N_{\text{pos}}}}
\newcommand{\Nneg}{{N_{\text{neg}}}}
\newcommand{\pipos}{{{\pi}^*_{\text{pos}}}}
\newcommand{\pineg}{{{\pi}^*_{\text{neg}}}}
\newcommand{\Dpos}{{D_{\text{pos}}}}
\newcommand{\Dneg}{{D_{\text{neg}}}}
\newcommand{\Repos}{{\Re_{\text{pos}}}}
\newcommand{\Reneg}{{\Re_{\text{neg}}}}
\newcommand{\wall}{{w_{\text{all}}}}
\newcommand{\kmax}{{k_{\text{max}}}}
\newcommand{\kmin}{{k_{\text{min}}}}

\long\def\comment#1{}

\begin{document}

\begin{titlepage}
\begin{center}
\vspace*{1cm}
\newcommand{\HRule}{\rule{\linewidth}{0.5mm}}
\HRule \\[0.7cm]
{ \huge \bfseries Complexity distribution of agent policies}\\[0.4cm]
\HRule \\[1.2cm]
\begin{minipage}{0.75\textwidth}
\begin{flushleft} 
Jos\'e Hern\'andez-Orallo\hfill(jorallo@dsic.upv.es)\\
Departament de Sistemes Inform\`atics i Computaci\'o\\
Universitat Polit\`ecnica de Val\`encia, Spain\\[12pt]
\end{flushleft}
\vspace{2cm}
\begin{center}
January 11, 2013
\end{center}
\end{minipage}

\vfill
\end{center}

\begin{abstract}
We analyse the complexity of environments according to the policies that need to be used to achieve high performance. The performance results for a population of policies leads to a distribution that is examined in terms of policy complexity and analysed through several diagrams and indicators. The notion of environment response curve is also introduced, by inverting the performance results into an ability scale. We apply all these concepts, diagrams and indicators to a minimalistic environment class, agent-populated elementary cellular automata, showing how the difficulty, discriminating power and ranges (previous to normalisation) may vary for several environments.

\vspace{0.3cm}


\end{abstract}

\end{titlepage}
\newpage
\tableofcontents
\newpage

\section{Introduction}

Many scientific disciplines must evaluate natural and artificial systems as a daily basis, such as psychometrics, animal cognition and artificial intelligence. A fundamental issue in these disciplines is then to determine the difficulty and the discriminating power of a task or problem instance. The capabilities of a system are given by the kinds and difficulty of the tasks it can solve. In this sense, tasks become discriminating if they can be solved by some agents but not by others. As a result, it is of utmost importance to precisely determine not only the types of tasks that are used but also their difficulty and their discriminating power.
Informally, difficulty is usually associated with the expected result for a set of individuals while discriminative power is usually associated with the variability of the results for a set of individuals.

Psychometrics has dealt with the analysis of task difficulty and discriminating power. Item Response Theory (IRT) \cite{embretson2000item}, for instance, is a well-founded approach in psychometrics where the difficulty (in terms of response curves) for each item is used to construct more effective tests, in order to derive scores and to obtain reliability measures. 
The ``difficulty'' of an item or task is generally (but not always) calculated with the results of previous tests on the same population, and not as a result of an underlying theory of the intrinsic problem difficulty of the task at hand. In artificial intelligence, the emphasis has not usually been put on evaluation, but on the development of intelligent systems. Nonetheless, some domain-specific tests and benchmarks have usually been designed by including problem instances of various difficulty levels.
In the end, the use of tasks of previously known difficulty and discriminating power is crucial to make testing effective, and get an accurate result with a minimum amount of tasks and time.

In more general terms, the concepts of complexity and difficulty appear in almost every field of science: physics, biology, psychology, (algorithmic) information theory, complex system theory, evolutionary computation, cryptography and many other fields. The analyses and measures are so diverse that any comprehensive account is far beyond of any survey, although we will give some pointers in the next section. The reason of such a diversity is due to the very different elements whose complexity we may be interested in (physical phenomena, life forms, bit strings, algorithms, reasoning processes, etc.) and also the different purposes (understanding phenomena, developing new methods or just measuring individuals). In fact, it is when we talk about agents (or any other kind of cognitive system) when we use the word `difficulty' instead of `complexity', as suggesting that `difficulty' is a relative (cognitive) issue. In other words, complexity is usually associated with the problem (statement and solution) and difficulty is usually associated with the way or method to go from statement to solution. 
It is thus not strange that a very complex task (in terms of how statement and solution are expressed) may have an easy solution. Conversely, some sets of very simple rules (e.g., games, automata or mathematical conjectures) develop into (or emerge to) extremely intricate phenomena and solutions.

In this paper, we focus on the analysis of difficulty (and discriminating power) for a particular (but ultimately general\footnote{Many other problems can be formulated using restricted environment classes (see, e.g., the hierarchy of environments in \cite[ch. 3]{Legg08}).}) setting. We concentrate on an interactive scenario, where an agent interacts with an environment through actions and observations. Instead of considering tasks with a particular episodic goal, we consider (possibly infinite) tasks where performance is evaluated in terms of rewards. This is in line with many approaches for cognitive system evaluation in psychology, animal cognition and artificial intelligence, most especially in the reinforcement learning setting. In the context of agent-environment interaction, the task is given by the environment and the good `solutions' are given by sequences of actions that maximise rewards. These actions are the result of a policy, which determines a behaviour in general. The central idea of this paper is to put the emphasis on behaviours, rather than actions. We highlight the issue that the aggregated rewards for the possibly infinite set of policies may be distributed in many different ways. As we will see, the study of these distributions plays a crucial role in the estimation of task difficulty. Another distinctive feature of our approach is that we will assess environment difficulty using algorithmic information theory \cite{Li-Vitanyi08} (by estimating the Kolmogorov complexity of each policy) in order to (1) evaluate the information content of the policy, (2) make the complexity measure more independent of the representation language, due to the invariance theorem\footnote{The Kolmogorov complexity of the same object using two different description mechanism is the same, up to a constant that is independent of the object \cite{Li-Vitanyi08}.}, and (3) consider all (or many of) the alternative (but equivalent) expressions of the same policy, so making the estimation more robust, as with Solomonoff's algorithmic probability (this is exploited by the coding theorem method \cite{delahaye2011numerical,zenilPhD}). 

So, the main idea of this paper is that we see  this difficulty from the standpoint of the policies, i.e., the agents, by calculating their complexity. We do not estimate the complexity of the environment. We do not estimate the complexity of the solution either (the sequence of actions). Instead, we estimate the complexity of the description of the policies, using a policy description language (an agent language) and compressing each program. As a result, we evaluate the difficulty of an environment by observing the distribution of policies in terms of their Kolmogorov complexity and their aggregated reward. The most straightforward analysis of the distribution is just looking at the frequency of simple policies achieving good results. In other words, a starting point can be the following `maxim': ``{\em If the environment has many simple policies achieving good results then it is easy. Otherwise, it is difficult}''. From this starting point, we will also derive other more robust indicators, functions and graphical tools from this distribution, in order to say when an environment is more or less difficult. Also, 
we will also investigate whether the environment is discriminating, not only in the beginning but after a random walk or after other policies.

The rest of this paper is distributed as follows. Section \ref{sec:background} shortly reviews some of the many different ways of measuring complexity appeared in the literature, distinguishing those that look at the problem (statement or solution), at the search space or at the population of solvers. Also, several notions of difficulty and discriminating power are introduced.
     Section \ref{sec:complexity} sets the focus on environments and gives definitions of their Kolmogorov complexity, the complexity of optimal action sequences and the complexity of their policies, and some connections between them.
 		Section \ref{sec:distribution} analyses the whole distribution of policy complexity and aggregated reward, using graphical tools, summary indicators and adapting the notion of response function, which finally leads to some general indicators of complexity and discriminating power.
    Section \ref{sec:experiments} introduces minimalistic environments based on {\em agent-populated elementary cellular automata} and a very simple agent policy language. Using them, we show the distribution of rewards according to policy complexity, calculate their indicators and plot their response functions.
		Section \ref{sec:discussion} discusses how this approach relates to other kinds of complexity, analyses the applicability and limitations of the approach and explores avenues for future work.


\section{Where to look at?}\label{sec:background}


There is a plethora of terms around the notion of `problem difficulty', such as task difficulty, problem hardness or simply complexity. We will shortly review some of them. Instead of arranging them according to their origin: psychometrics, computer science, complexity theory, biology, etc., we will group the approaches in the literature according to what they look at in order to establish parameters such as difficulty or discriminating power.

\subsection{Looking at the problem}

The first distinction that we need to do is between problem {\em class} difficulty and problem {\em instance} difficulty. 
In computational complexity theory \cite{du2011theory}, the focus is usually put on class complexity, e.g., the behaviour of an algorithm for all possible inputs, in terms of the time (or space) that an algorithm takes to solve a problem or to decide whether an object belongs to a set. Nonetheless, the notion of instance complexity has also been considered, related to the concept of average-case computational complexity,
developed by Leonid Levin in the 1980s \cite{levin1986average}, which is of course related to the more general notion
of average-case performance of algorithms (see, e.g., \cite{knuth1973sorting}). However, the point of view is the time (or space) that a {\em given algorithm} requires to solve a problem instance, such as quicksort being faster for an almost sorted instance than a randomly sorted one, not about the {\em intrinsic} complexity of the problem instance.

Here we are interested in instance {\em difficulty}, from the point of a view of cognitive agent evaluation (running whatever algorithm or policy). This stance has been taken by psychology and other cognitive sciences, where there is an old history of approaches using the concept of {\em working memory} or the number of elements that must be considered at the same time in order to solve a problem \cite{miller1956magical,Ashcraft1992,hardman1995problem,barch1997dissociating}. In this line, and closely related to the emerging field of artificial intelligence, Simon and Kotovsky explored, for decades, ``the problem space of difficulty'' \cite{SimonKotovsky1963,KotovskySimon1990}.
However, most of these views of difficulty have been anthropocentric. The use of these measures for non-human subjects (especially for machines) may be misleading since it has been shown that many tasks that are very difficult for humans are very easy for machines and vice versa \cite{IQnotformachines}. 

A more mathematical (and computational) approach for associating complexity with the number (or the size) of items that are necessary to explain a concept (or solve a problem) is now known as algorithmic information theory (also known as Kolmogorov complexity \cite{Li-Vitanyi08}). As a result, in the past forty decades, the difficulty of solving a particular task has been occasionally related to its Kolmogorov complexity. This relation has been especially advocated for in inductive inference, 
because it makes sense to look at the length of the shortest pattern explaining the evidence, and this length seems to determine the difficulty of the problem. However, the relation is, at most, unidirectional\footnote{Some repetitive, but long, patterns (hence having high Kolmogorov complexity) may just require memory to see and store the repetition.}, and not always intuitive\footnote{True random strings are incompressible, so leading to no `solution' at all, either easy or difficult.}. 
Several alternatives, such as Levin's $Kt$ \cite{Levin73}, logical depth \cite{Chaitin1977}, effective complexity \cite{Ay-etal08}, computational depth \cite{Antunes2006}, sophistication \cite{chedid2010sophistication}, and others (see \cite[chap. 7]{Li-Vitanyi08}), have been proposed instead, where not only the solution of the problem (the pattern) is analysed but also how difficult is to extract that solution from the evidence. These proposals are generally conceived for evaluating objects but some of them can also be used for quantifying difficulty in sequential inductive problems. For instance, a variant of $Kt$, known as intensional complexity, accurately captured the difficulty humans found on IQ test series problems \cite{HernandezOrallo-MinayaCollado98,HernandezOrallo00a}. 
Outside induction, the idea of instance complexity \cite{orponen1994instance}\cite[sec. 7.4]{Li-Vitanyi08} has also been developed with the use of algorithmic information theory, in front of the classical 
view of problem class complexity mentioned above. For other kinds of problems, such as general inductive and deductive problems, \cite{thesis99,HernandezOrallo00d,HernandezOrallo00b} also explored the use of algorithmic information theory, by considering the information gain from the problem to the solution.

The above approaches have mostly focussed on static or (at most) sequential problems. 
When we move to the evaluation of natural or artificial agents, we need to address the estimation of the difficulty of {\em interactive} tasks, or {\em environments}, where subsequent input is affected by the agent's actions. Here, we usually consider other elements as well, such as a score or a reward. In reinforcement learning, for instance, the difficulty of a problem depends on the goal and how rewards are assigned (according or not to this goal).
Also, rewards are rarely {\em normalised}, and the aggregation of results for several environments could be biassed in favour of those environments with higher reward magnitudes.
Another problem of an interactive setting is that each action leads to a different (sub)environment, so the results may be very sensitive to an early wrong action. This is why some works have advocated for ergodic environments \cite{Legg08}, where the agent can always recover from a local `hell' or `heaven'. Yet another issue is that estimating the reward of a policy may take many steps in the environment, so the mere checking of a solution takes a lot of time. Also, the choice of the aggregate reward function is crucial, especially if time is taken into account \cite{HernandezOrallo10c}. Finally, the evaluated agent can evolve during the process and become a different agent \cite{hernandez2013potential}.

Despite all these problems, many works on reinforcement learning, agents and robotics have considered specific ways of approximating difficulty.
Typically, the mere `size' of the problem (e.g., maze size, \cite{Zatuchna-Bagnall09}) or the state space \cite{madden2004transfer}, have been used as approximations of the `difficulty' of an environment. However, these approaches are usually domain-specific and are prone to a series of problem because difficulty is just handled in an informal or ad-hoc way.

Alternatively, the application of Kolmogorov complexity to estimate the difficulty of an environment has also been explored, 
but becomes much more cumbersome. In \cite{HernandezOralloDowe2010}, another variant of Kolmogorov complexity ($Kt^{\max}$) was suggested as a measure of complexity for the environments, but it was already stated that this approximation was unidirectional, since some very complex environments might be easy, i.e., high rewards could be obtained by very simple policies (\cite[sec. 4.1]{HernandezOralloDowe2010}). Also, the problem of whether some environments are more {\em discriminating} than others is also discussed there, and some notions (such as reward-sensitiveness) were introduced \cite[sec. 4.2]{HernandezOralloDowe2010}. Following these ideas, an environment class was introduced in \cite{HernandezOrallo10b}, where the difficulty of an environment can be approximated by the Kolmogorov complexity of the rules (as a Markov algorithm) that describe the environment. The class is used to create intelligence tests in \cite{AGI2011Evaluating}, but the approximation of difficulty is not satisfactory. In fact, some other related works also mention the importance of being able to assess the difficulty of the environment, such \cite{AGI2011Compression} and to achieve discriminating power \cite{HernandezOrallo09aTR}.

Even more challenging, though, is when we consider other agents in the environment. In games (and multi-agent systems in general), the difficulty of a problem depends on the rules, but most especially on the opponents. For general multi-agent systems and multi-agent reinforcement learning \cite{busoniu2008comprehensive} in particular, the difficulty of the system depends on the opponents and cooperators \cite{DBLP:conf/agi/Insa-CabreraBH12}, and their intelligence \cite{AGI2011DarwinWallace,Turing100}.

A more minimalistic approach, starting with very basic components is the study of the emergence of complexity (and other properties) in cellular automata and other artificial life universes (see, e.g., \cite{wolfram2002new}). Algorithmic information theory has also been applied here; an interesting approach is the calculation of the complexity of the patterns that an elementary (1-dimensional) cellular automaton generates \cite{zenil2012two}. However, the purpose of Zenil et al. is not to measure the difficulty for an agent (not even a `glider'), but whether the 2-dimensional patterns have low or high complexity. In fact, it is very unlikely (or only likely after billions of generations) that in any of these settings or any other artificial life universe, we find `agents' whose cognitive abilities may be interesting to evaluate. 


\subsection{Looking at the search space}

Instead of looking at the complexity of the environment, a natural way to look at the difficulty of a problem relies on analysing the search space. For instance, in a maze, the search space is the set of trajectories that can be performed in the maze. Each trajectory is a sequence of actions. Analysing how likely the solutions are in the space of trajectories may work because, in this particular case, the environment is static (the maze does not change after the agent's actions). However, things are different for other kinds of environment, where the environment really reacts to the agent's actions. In general, the same ideas and principles used in the areas of optimisation and heuristics apply here. For instance, it is said that a problem is difficult if there is a ``high density of well-separated almost solutions (local minima)'' \cite{cheeseman1991really} or we have a ``rugged solution space'' \cite{hoos1999sat}.

Evolutionary computation is an area where the search space and its relation to difficulty has been studied more exhaustively.
While fitness and reward functions could be considered parallel, most approaches consider {\em a} solution as achieving the problem goal, while we have advocated above for a case where have an aggregated reward, not a final goal. 
Nonetheless, the first big difference with a reinforcement learning case is that in evolutionary computation we consider the notion of operator, which converts (mutates) one solution into a different solution. The solutions are the nodes of the search space and the operators are edges. This makes up the `landscape' in evolutionary approaches, but it is not clear how to adapt this to other settings.
Another difference between environments and evolutionary problems is that for many approaches the size of the individual codes (genotype) is constant. Only when the genotype size is variable (e.g., evolutionary programming), the size of the solution is a factor of study.
Taking these (and other) differences) into account, it is still useful to learn from the experience in this field. \cite{he2007note} includes an account of approaches for problem difficulty measures in evolutionary computation, appeared in the past twenty years.
There are several interesting concepts such as ``rugged fitness landscape'' (problems where the solution space is full of local minima and maxima), 
the `needle in a haystack'' metaphor, 
the existence of one (or a few) way to the solution, 
{\em deceptive} ways (promising ways finally leading nowhere),
the appearance of ``epistatic interactions'' between parts of the solution (gene co-influence),
the notions of building blocks, and many others. 

\subsection{Looking at the population of solvers}

In psychometrics, items are usually classified into several difficulty categories, and a variety of items of different difficulty are used in order to cover a wide range of the ability that is to be measured. Item response theory (IRT) \cite{embretson2000item} is a paradigm for the study of items (tasks) and a well-grounded way of designing tests and other instruments that measure abilities, especially in the area of (computerised) adaptive testing. 
IRT is based on mathematical functions and associated probability and informativeness estimations for each item, according to several models. One very common model for discrete-score problems is the three-parameter logistic model, where the item response function (or curve) corresponds to the probability that an agent with ability $\theta$ gives a correct response to an item. This model is characterised as follows:
\[p(\theta) \triangleq c + \frac{1-c}{1+e^{-a(\theta-b)}} \]
\noindent where $a$ is the {\em discrimination} (the maximum slope of the curve), $b$ is the {\em difficulty} or item location (the value of $\theta$ leading to a probability half-way between $c$ and 1), and $c$ is the chance or asymptotic minimum (the value that is obtained by {\em random} guess, i.e., $(1+c)/2$. as in multiple choice items). 
The zero-ability expected result is given when $\theta=0$, which is exactly $z = c + \frac{1-c}{1+e^{ab}}$. 
Figure \ref{fig:irc} (left) shows an example of a logistic item response curve. 
For continuous score items, a very frequent approach is the linear model \cite{mellenbergh1994,ferrando2009difficulty}:
\[X(\theta) \triangleq z + \lambda \theta + \epsilon \]
\noindent where $z$ is the intercept (zero-ability expected result), $\lambda$ is the loading or slope, and $\epsilon$ is the measurement error. Again, the slope $\lambda$ is positively related to most measures of discriminating power \cite{ferrando2012discriminating}.
Figure \ref{fig:irc} (right) shows and example of a linear item response curve.
Note that for continuous score items, if they are bounded, the logistic model may be more appropriate.

\begin{figure}
	\centering
		\vspace{-1.2cm}
		\hspace{-0.5cm} 
		\includegraphics[width=0.50\textwidth]{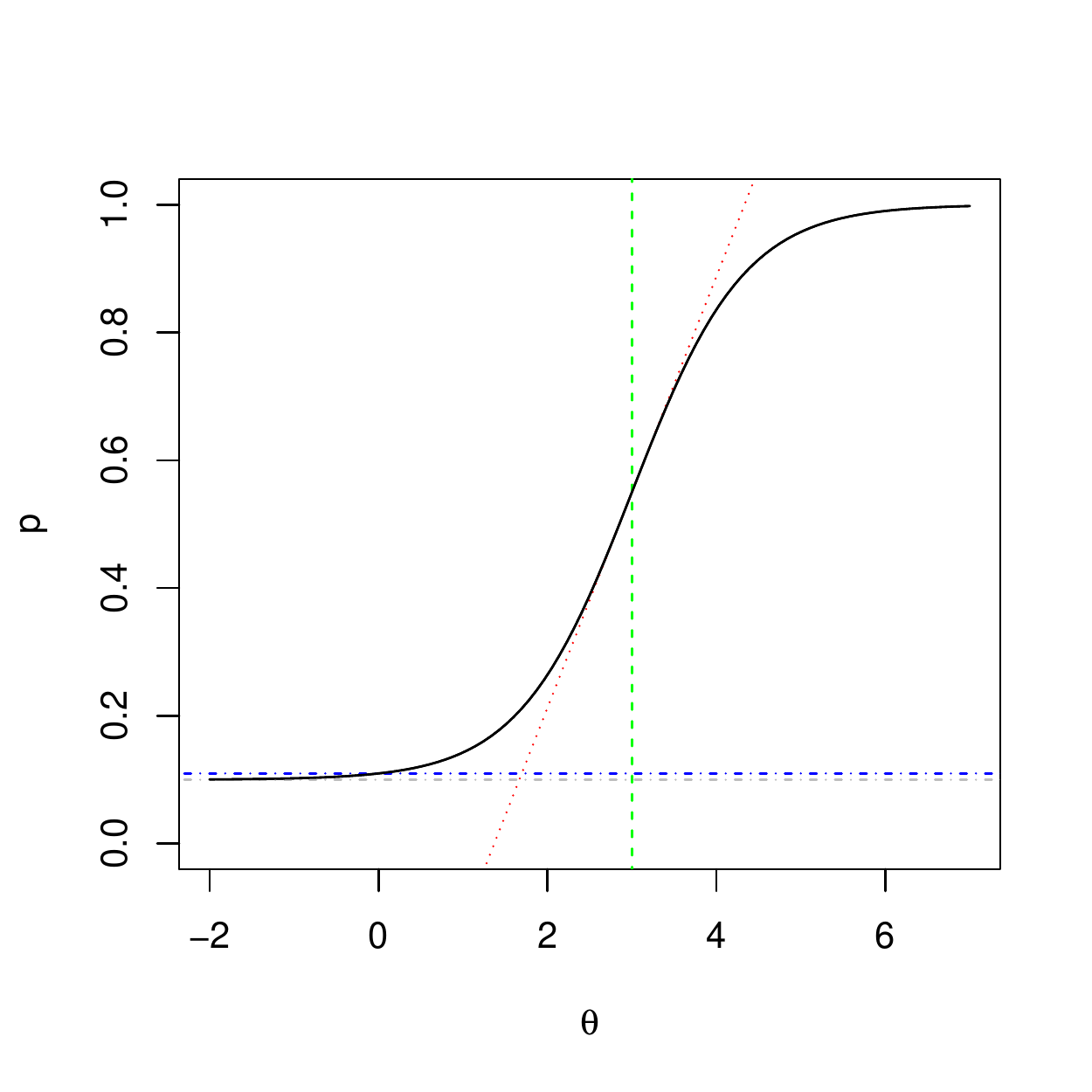} \hspace{-0.3cm}
		\includegraphics[width=0.50\textwidth]{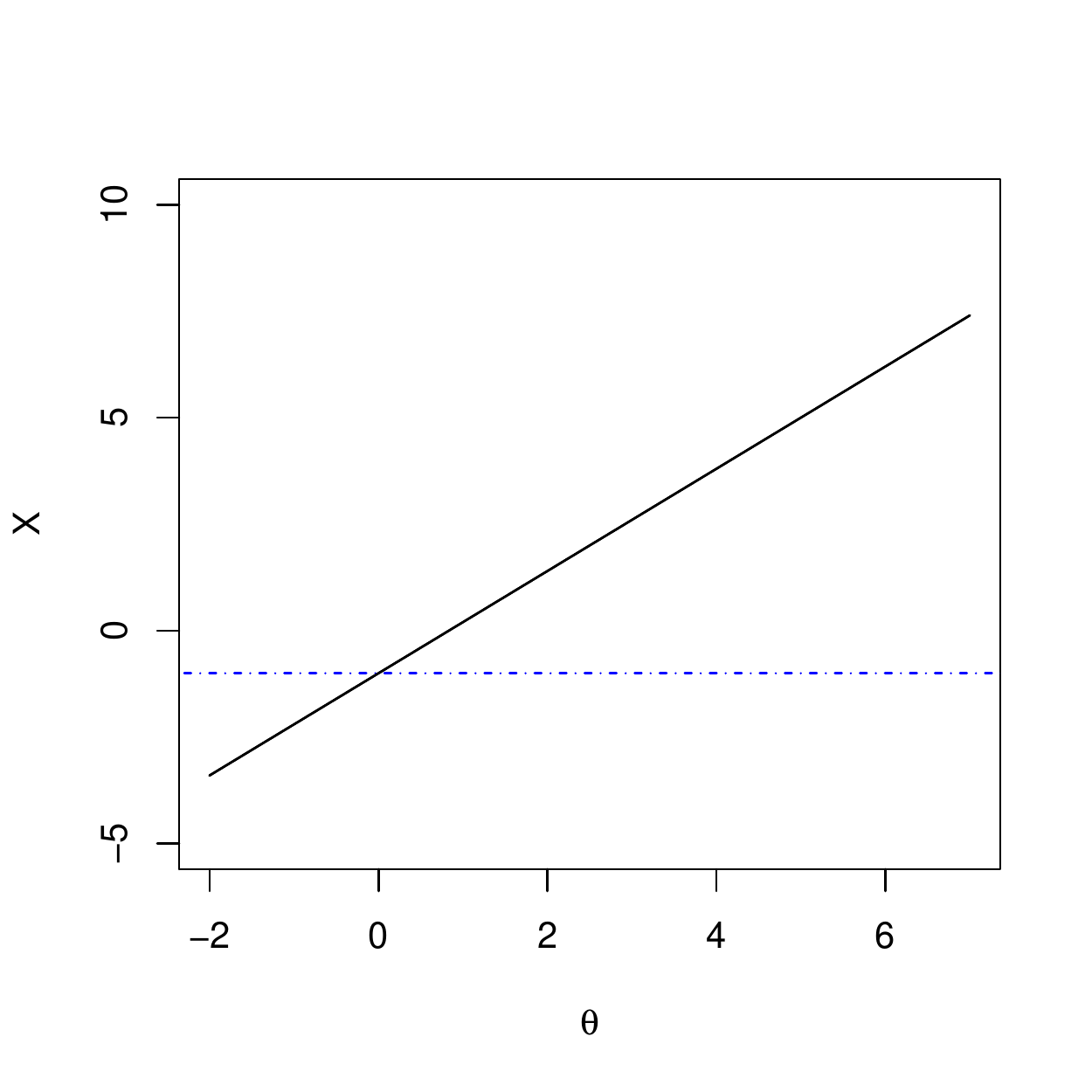} \hspace{-0.7cm}
		\vspace{-0.5cm}
	\caption{Left: item response function (or curve) for a binary score item with the following parameters for the logistic model: discrimination $a=1.5$, item location $b=3$, and chance $c=0.1$. The discrimination is shown by the slope of the curve at the midpoint: $a(1-c)/4$ (in dotted red), the location is given by $b$ (in dashed green) and the chance is given by the horizontal line at $c$ (in dashed-dotted grey), which is very close to the zero-ability expected result $p(\theta)=z$ (here at 0.11). Right: A linear model for a continuous score item with parameter $z=-1$ and $\lambda=1.2$. The dashed-dotted line shows the zero-ability expected result.}
	\label{fig:irc}
\end{figure}

Working with item response models is very useful for the design of tests, because if we have a collection of items, we can choose the most suited one for the subject (or population) we want to evaluate. According to the results the subject has obtained on previous items, we may choose more difficult items if the subject has succeeded on the easy ones,  we may look for those items that are more discriminative in the area we have doubts at, etc. Note that discrimination is not a global issue: a curve may have a very high slope at a given point, so it is highly discriminating in this area, but the curve will almost be flat when we are far from this point. Conversely, if we have a low slope, then the item covers a wide range of difficulties but the result of the item will not be so informative as for a higher slope. 

One important question about these curves is how the parameters are estimated. This is usually done by applying the item to a population of subjects for which we already know their ability. Since the ability of subjects is usually determined by tests, leads to certain circularity, which is usually sorted out in an incremental way in psychometrics: the first tests were devised by comparing on informal assessments of both the ability of subjects and the difficulty of tasks, the next attempt refined them using the results from the previous ones, etc. It is important to note that there is no theoretical assessment of the intrinsic difficulty or discriminating power of items: this is done empirically by looking at the population.

Overall, there have been approaches to derive problem difficulty (and other parameters) from the intrinsic complexity of the problem statement (or the solution), from the complexity of the way from the problem statement to the solution, i.e., the search space, and, finally, extrinsically, from a population. Some of them are informal and other are more rigorously defined in terms of the representation of problem statement, solution or search space. The main drawback of any of these, but most especially, the ones based on the search space, is that the measure of difficulty highly depends on the representation. For instance, if two evolutionary settings use different solution representation and different operators, the estimated difficulties for the {\em same} problem could highly differ. 
The approaches based on algorithmic information theory are less prone to this drawback because the use of Kolmogorov complexity may reduce this dependency due to the invariance theorem. Let us see next how to apply algorithmic information theory to the general case when an agent interacts with an environment, and how it leads to a population-oriented approach.




\section{The complexity of environments, actions and policies}\label{sec:complexity}


The main goal of this paper is to study the difficulty that {\em one} agent faces on a given environment. In order to consider a very general setting (including games, for instance), we will consider that there might be {\em other} agents in the environment. Multi-agent environments are usually based on some common space, transition rules and reward system, which are shared by all of them. A more formal definition follows:

\begin{definition} \label{def:environment}
A multi-agent environment $\mu$ is a tuple $\left\langle \sigma, \tau, \omega, \rho, \Pi \right\rangle$ where $\sigma$ is a (state) space, $\tau$ is a state transition function, $\Pi$ is a set of agents $\pi_1, \pi_2, \dots, \pi_n$ as policy functions and $\omega$ and $\rho$ are observation and reward functions (respectively) for all the agents.
\end{definition}

The previous definition does not completely specify how the transition rules or the reward system work and the topology of the space, the possible contents and the movements of the agents. We will see a specific case in section \ref{sec:experiments}. For the moment, it is enough to analyse the issue of difficulty in a general, but still formal way, with the following definitions and notation. Given any of the agents $\pi_i$ running for $t$ steps on a deterministic environment, we denote by $\chi_i^t$ its interaction history, which is a sequence of tuples of the form $\left\langle a, o, r \right\rangle$, representing action, observation and reward respectively, and where $a$ and $o$ are elements in finite discrete sets ${\cal{A}}$ and ${\cal{O}}$ respectively, and $r$ is a bounded rational number. The policy $\pi_i$ is then a function $\pi_i(\chi_i^t) \rightarrow a_i^{t+1}$. Similarly, the transition function takes the previous state and the actions of all the agents leading to another state: $\tau(\sigma^t, a_1^t, a_2^t, \dots, a_n^t) \rightarrow \sigma^{t+1}$, from which observations and rewards are produced for all of them: $\omega(s^{t+1}, i) = o_i^{t+1}$ and $\rho(s^{t+1}, i) = r_i^{t+1}$.
Given an agent, we can just project over its history $\chi_i^t$ to get just the sequence of actions, denoted by $\alpha_i^t$. We assume all transition function to be deterministic.  
We denote by $R_i^t(\pi, \mu)$ an aggregation function of rewards which is applied to the sequence of rewards until time $t$ (e.g., the average of all rewards so far for agent $i$). 

As discussed in the previous section, we can try to estimate the difficulty of the environment in different ways: by looking at the problem, the solution or the search space. Here, the problem is given by the environment $\mu$, the solution is a sequence of actions $\alpha$ and the search space is given by the exploration of all the sequences of actions. 

First, if we consider the problem statement (i.e., the environment), we can define its complexity $K$ as\footnote{Unless precisely specified, $K$ can be the Kolmogorov complexity, an approximation or other related function.}:
\[ K(\mu) \triangleq K(\left\langle\sigma, \tau, \omega, \rho, \Pi\right\rangle) \]
\noindent If we take the point of view (or role) of agent $i$, then the complexity of the environment is defined as:
\begin{equation}
\dot{K}_i(\mu) \triangleq K(\left\langle\sigma, \tau, \omega, \rho, \Pi - \{\pi_i\}\right\rangle) \label{eq:dotK}
\end{equation}
\noindent The use of $\dot{K}_i(\mu)$ as an approximation of the difficulty of the environment (from the point of view of role $i$) is then the first possibility and would boil down to estimating the difficulty of an environment as the complexity of its space, transition function, observation function, reward function and other agents.

As mentioned above, we can consider a second approach: calculating the complexity of the solution.
Focussing on one role $i$ and a given episode length $t$, let us define $A_i^t(\mu)$ as the set of sequences of actions made by any policy $\pi$ such that $R_i^t(\mu,\pi)$ is maximised for role $i$.


It follows that $|A_i^t(\mu)| \geq 1$ (there might be one or more sequences maximising the aggregated reward). From here, we can define the best-action difficulty (or hardness) $\Haction$ as follows:
\begin{equation}
\Haction_i^t(\mu) \triangleq \min_{\alpha \in A_i^t(\mu)} K(\alpha) \label{eq:Haction}
\end{equation}
\noindent  Here, $\Haction(\mu)$ represents the complexity of the simplest action sequence leading to optimal aggregated reward.

What is the relation of the two previous approaches? We can show that if $K$ is Kolmogorov complexity then we have:

\begin{proposition}\label{prop:upperbound1}
Assuming an environment $\mu$ with computable space, transition, observation and reward functions, and also a computable reward aggregation function, we have:
\[ \forall i \in \{1 \dots n\} \:\:\:  \Haction_i^t(\mu) \leq^+ \dot{K_i}(\mu) + K(t) \label{eq:upperbound1} \]
where $\leq^+$ means that the inequality holds up to a constant which is independent of both terms.
\end{proposition}
\begin{proof}
In this case, consider one of the shortest descriptions (coded as binary strings) $d$ for $\left\langle\sigma, \tau, \omega, \rho, \Pi - \{\pi_i\}\right\rangle$.
Clearly $l(d) = \dot{K}_i(\mu)$. If we know this shortest description we can {\em simulate} the environment $\mu$ and the interaction with all the agents (except $\pi_i$). 
 We now simulate (e.g., in parallel) all possible actions for agent $\pi_i$ until time $t$ and record its histories and rewards. We just select those which maximise $R_i^t$ and get exactly $A_i^t(\mu)$.
The description of how to make this simulation will have an overhead of $c + K(t)$, which is independent of the environment (it only depends on the aggregation function $R_i^t$, the descriptional language being used, and $t$). Note that any computable function can be described with a finite program. Consequently, describing $A_i^t(\mu)$ takes at most $l(d) + c + K(t)$ bits.
Naturally, $\Haction_i^t(\mu) = \min_{\alpha \in A_i^t(\mu)} K(\alpha) \leq l(d) + c + K(t) + c'$ bits, where $c'$ is the number of bits needed to describe what a minimum is and how to select it from the set $A_i^t$. Then:
$\forall i \in \{1 \dots n\} \:\:\: \Haction_i^t(\mu) \leq l(d) + c + K(t) + c'= \dot{K}_i(\mu) + c + K(t) + c'$, which completes the proof.
\end{proof}

Note that the previous rationale also leads to $K(A_i^t(\mu)|\mu) \leq^+ K(t)$.
The previous proposition implies that simple environments (with low $K$) cannot have complex optimal action sequences. However, the opposite implication is not true and we cannot find a lower bound. 
Also, the previous proposition relies on the use of $K$, which is incomputable. In fact, this upper bound would not be true in case we used some approximation of $K$ taking time into account (such as Levin's $Kt$). 
Nonetheless, in many cases, it may still be true that the cost of describing a sequence of actions leading to (and calculating) the optimal result is smaller (in general) than the `cost' of describing the environment and the rest of `ingredients' of the problem. In other cases, however, this will not be true, as, e.g., in environments with very complex emergence behaviour. For example, in chess, it is still not trivial (in computational terms) what to do when knowing the rules and the moves of the opponent in order to, e.g., finish the match in least moves. Given infinite time, however, the optimal solution can be found.

So we have analysed the complexity of the problem formulation (the environment) and the solution (the sequence of actions).
Now, 
we are going to consider agent policies. There are several reasons for this. First, the sequence of actions depends on the environment. For instance, consider a very complex environment which outputs a sequence of non-compressible observations and consider that the best reward is just attained by performing an action which is a simple function of the observation (e.g., if observations and actions were binary, just replicating the input as an output). Then the complexity of the sequence of actions would be high, but the solution to the problem is intuitively easy\footnote{In fact, this is true in general, at least if $K$ is Kolmogorov complexity, since $K(A_i^t(\mu)|\mu) \leq^+ K(t)$.}. Second, action sequences depend on $t$ and may become very long for long episodes. In contrast, policies are finite, which makes the distribution analysis much easier even with rough approximations of Kolmogorov complexity. Third, policies represent agents, i.e., behaviours, not sequences of actions. In the end, we are interested in knowing whether there are simple agents (i.e., simple policies) solving the problem. In fact, the same sequence of actions can be performed by two different agents, with very different complexities of their policies.

Before pushing forward the approach based on policies, we need a few definitions. 
Given a class of policies or agents $\Omega$, we define:
\begin{eqnarray}
 {\Rmax}^t_i(\mu) & \triangleq & \max_{\pi \in \Omega} R_i^t(\pi,\mu) \label{eq:Rmax}   
\end{eqnarray}
\begin{eqnarray}
 {\Rmin}^t_i(\mu) & \triangleq & \min_{\pi \in \Omega} R_i^t(\pi,\mu) \label{eq:Rmin}   
\end{eqnarray}
\begin{eqnarray}
 {\Rmean}^t_i(\mu) & \triangleq & \frac{1}{|\Omega|} \sum_{\pi \in \Omega} R_i^t(\pi,\mu) \label{eq:Rmean}
\end{eqnarray}
\noindent as the maximum, minimum, and average (respectively) aggregated reward attainable in $t$ steps. In other words, this is the best (respectively worst, or average) score for any possible agent performing as agent $\pi_i$ in the environment.

\noindent We can calculate the complexity of the best policy or, more precisely, the lowest complexity of any best policy, known as
the best-policy difficulty (or hardness) $\Hpolicy$ as follows:
\begin{equation}
\Hpolicy_i^t(\mu) \triangleq \min_{\pi : R_i^t(\pi) = \Rmax^t_i(\mu)} K(\pi) \label{eq:Kmax}
\end{equation}
\noindent And we get a similar result to the one above:

\begin{proposition}\label{prop:upperbound2}
Assuming an environment $\mu$ with computable space, transition and reward functions, and also a computable reward aggregation function:
\[ \forall i \in \{1 \dots n\} \:\:\: \Hpolicy_i^t(\mu) \leq^+ \dot{K}^t_i(\mu) + K(t) \label{eq:upperbound2} \]
where $\leq^+$ means that the inequality holds up to a constant which is independent of both terms.
\end{proposition}
\begin{proof}
The proof is very similar to the one for proposition \ref{prop:upperbound1}, but just choosing a policy that gets the best results given the environment.
\end{proof}

So we have a very loose upper bound to $\Hpolicy_i^t(\mu)$. Since the actual value depends on a minimum, this suggests that we need to explore many agent policies in terms of their aggregated reward in order to estimate this. We take a look at the whole agent policy distribution below.
From now on, we will drop $i$ when it is clear which role we are using (or there is only one agent in the environment, i.e., if it is not a multi-agent environment) and drop $t$  when clear from the context (or irrelevant for the matter at hand). Similarly, we will also drop $\mu$ when working with just one environment.


\section{A distribution of agent policies}\label{sec:distribution}


The value $\Hpolicy(\mu)$ seen in Eq. \ref{eq:Kmax} of the previous section looks like a good candidate to problem difficulty, as the lowest complexity of any optimal policy. If the best policy (or one of the best policies) has low complexity, the problem could be said to be easy. While this is a straightforward interpretation, it is too narrow because of several reasons. First, it might be the case that the best policy $\pi$ leads to $R(\pi)= r$ while the second best policy $\pi'$ leads to  $R(\pi')= r - \epsilon$. If $\epsilon$ is (comparatively) very small but $K(\pi') << K(\pi)$, we may even say that an almost equally good (but much simpler) policy can be found. Intuitively, the environment would look even easier. Second, focussing on just one value is too risky if we plan to do a robust approximation of the value of difficulty, especially taking into account that in practice we will be forced to analyse samples of the agent population $\Omega$, and not the whole population itself.

Hence, the idea is to analyse $R(\pi)$ against the complexity of $\pi$. This is what we see next.

\subsection{Graphical analysis and indicators}

In order to study $R(\pi)$ in terms of the complexity of $\pi$, we define {\em slices} of the whole distribution (conditional distributions actually) according to its complexity. Given a class of agents $\Omega$, we define:
\[\Omega[k] \triangleq \{ \pi \in \Omega : K(\pi) = k \} \]
\noindent From here, we can parametrise $R$ as follows:
\begin{definition}
The aggregated reward per difficulty $k$ is given by the following distribution:
\[ R[k] \triangleq \{R(\pi) : \pi \in \Omega[k]  \} \] 
\end{definition}
\noindent By varying $k$ we have a series of distributions. We will use the notation $[\leq k]$ to represent all values of complexity from $1$ to $k$. For instance $R[\leq k] = \sum_{j=1..k} R[j]$, known as the `accumulated'\footnote{An `accumulated' distribution is not a {\em cumulative} distribution.} version of $R[k]$.

What $R[k]$ looks like? This will highly depend on the environment so its shape can give us information about its difficulty. An example is shown in  Figure \ref{fig:r-easy}. 
If we look at the distributions in this figure, we see that the distribution typically widens with increasing values of $k$. This does not mean that the expected reward increases (it remains constant in this case), but that the higher variability given by more complex programs make it possible to consider more diverse policies and find some that may be better. 

\begin{figure}
	\centering
		\vspace{-1.2cm}
		\includegraphics[width=0.6\textwidth]{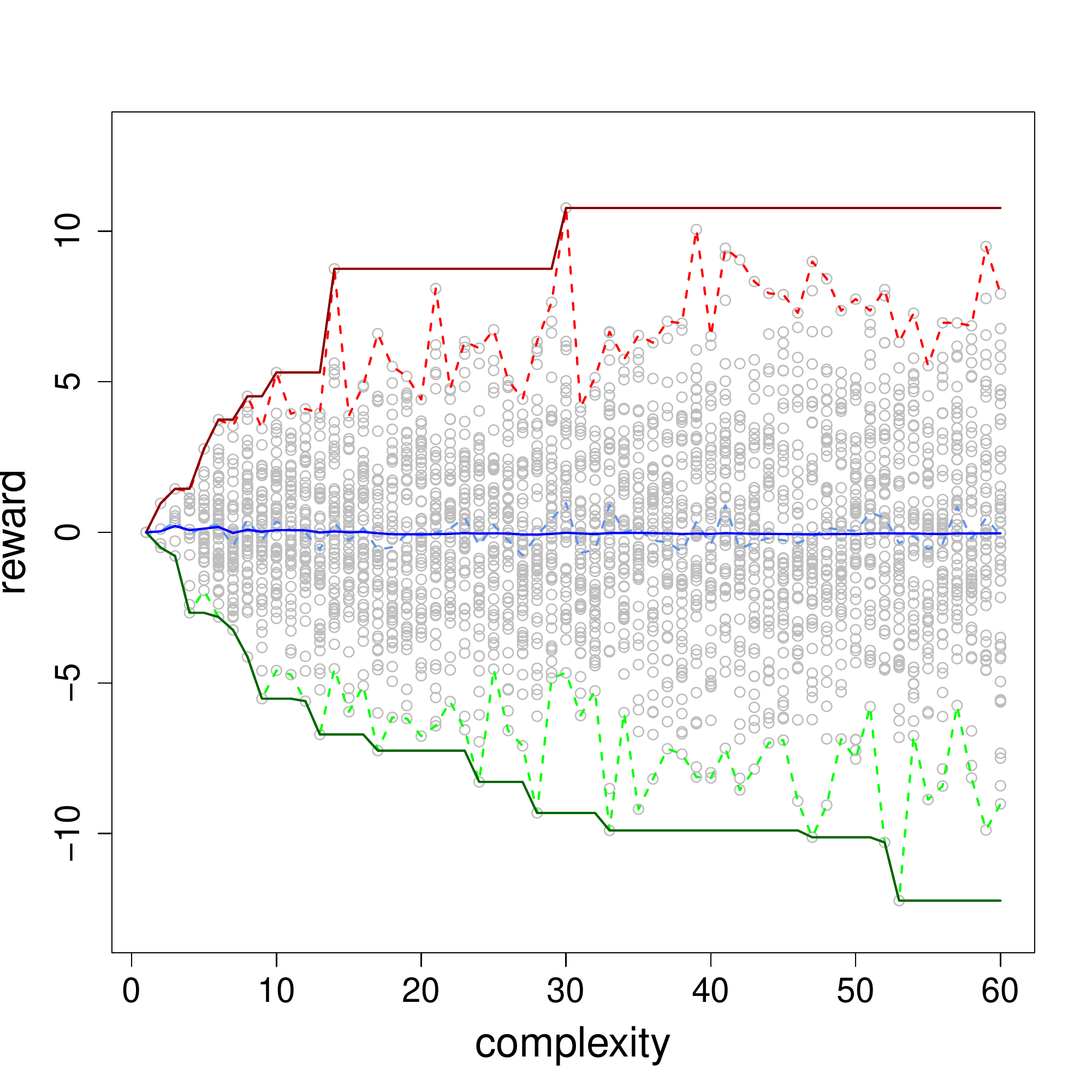}
		\vspace{-0.4cm}
	\caption{Figurative plot illustrating the difficulty of a balanced environment (an environment where random actions lead to an expected reward of 0, as in \cite[sec. 4.2]{HernandezOralloDowe2010,HernandezOrallo10b}). It shows a distribution of the expected aggregated reward $R$ (\yaxis) for an arbitrarily large number of steps, according to the agent's complexity $k$ (\xaxis). It does not show all the possible agents $\pi$ with $K(\pi) \leq 60$ but a sample containing about 50 agents for each value of $k$. The dashed red and green envelope curves show the maximum ($\Rmax[k]$) and minimum ($\Rmin[k]$), respectively, while the dashed blue curve shows the average ($\Rmean[k]$). The accumulated envelopes for the maximum ($\Rmax[\leq k]$) and minimum ($\Rmin[\leq k]$) are shown in solid dark red and dark green curves respectively, and the dark blue curve shows the accumulated average ($\Rmean[\leq k]$). }
	\label{fig:r-easy}
\end{figure}





We can derive some indicators of difficulty as a summarisation of the distribution.
The first idea for difficulty is to consider the maximum. The maximum envelope is just defined as:
\begin{equation}
 \Rmax[k] \triangleq \max_{\pi \in \Omega[k]}  R[k]   \label{eq:Rmaxslice}
\end{equation}
\noindent Clearly, $\Rmax$, as was already defined (see Eq. \ref{eq:Rmax}), can also be calculated as $\max_{k} \Rmax[k]$. In the example of Figure \ref{fig:r-easy}, this is 10.8. The previous Eq. \ref{eq:Rmaxslice} can be similarly defined for $\Rmin$ and $\Rmean$. 
By looking at $\Rmax$, $\Rmin$ and $\Rmean$ we can have an idea of the normalisation of the environment. If $\Rmean$ is not centred, it may indicate some kind of bias (i.e., the environment can be benevolent or malevolent, or simply that the rewards are not balanced).
This suggests that results need to be normalised, by getting a new plot where the mean is centred. Normalisation is important because it makes possible to aggregate the results of one (or more agents) on several environments, by making them commensurable. 

If we look for one simple indicator of difficulty, yet again our first approach of a measure of difficulty could be to choose where $\Rmax$ is found:
%
%
\[ \argmin_{k=1.. \infty} (\Rmax[k]) \] 
\noindent which is exactly equivalent to $\Hpolicy$, as seen in Eq. \ref{eq:Kmax}. 
%
 In the environment of Figure \ref{fig:r-easy}, we see that the maximum value is reached at complexity 30, so we have that $\Hpolicy = 30$. 
However, there are other policies with almost-as-good results at complexity 14. 
This shows that this indicator is not robust. We could also calculate the smallest value of $k$ for a given value of $R$ or other indicators.
Nonetheless, all these indicators based on a single number are prone to problems for real situations, where we can only get a sample of the points, and not all of them. 
An alternative option then is to calculate similar values taking some quantile of the data into account, e.g., 99\%, and get the lowest value of $k$ such that we get 1\% of the data above a given value (or, as a percentage of the maximum, e.g., 95\% of $\Rmax$). 
%
%
Many other statistical indicators can be used, including a comparison of the distributions. For instance, we could determine for which value of $k$ the distributions stabilise (consecutive distributions become very similar) in order to determine when a more robust maximum has been reached. 
Another interesting indicator is to calculate the correlation between reward and complexity. If we consider the correlation of all the values, i.e., $Cor_{\pi \in \Omega}(K(\pi),R(\pi))$,  
we should expect to have no correlation (as in Figure \ref{fig:r-easy}). 
Quite differently, we can calculate the correlation by slice using the maximum value, i.e., 
$Cor_{i=1..\kmax}(i, \Rmax[i])$ (with $\kmax$ being the maximum complexity which is considered),
which in this case is $0.73$ (Spearman correlation). Here, a high positive correlation is expected. Otherwise, the environment may even be penalising complex policies.
We will see the use of some of these indicators in section \ref{sec:experiments}.


Nonetheless, considering the envelope (either as given by the maximum points or by a quantile of the data) may not capture the difficulty of finding good rewards in general. This issue is related to how the points are distributed for each slice, which is, in turn, related to the notion of discriminating power. In order to clarify this, we need to be more precise about the meaning of being discriminating. This was initially stated, in an informal way, as a property of problems or tasks that can be solved by some agents but not by others, i.e., showing high variability in results. 
An interesting plot for understanding this would be to show the distribution by slice, or the distribution of $R$ for the whole data. The degree of dispersion (measured by kurtosis or other indicators) could be useful here. For instance, if the values between $\Rmax$ and $\Rmin$ are distributed almost uniformly we have that there are many policies closer to the optimal result. If it has a more peaked shape (for instance like a beta distribution with $\alpha=\beta=3$), then the extreme values may be more difficult to reach. While this may give clear findings occasionally (we will look at several cases in section \ref{sec:experiments}), this analysis may be unable, in general, to disentangle difficulty and discriminating power. Let us look for a different approach next.







\subsection{Environment response functions}\label{sec:erc}

As we mentioned in section \ref{sec:background}, item response theory (IRT) \cite{embretson2000item} is a paradigm for the study of items (tasks) and a more principled way of designing tests and other instruments that measure abilities, especially in the area of (computerised) adaptive testing. The most distinctive feature about IRT is that each item (task) is categorised and analysed according to its difficulty in terms of the response that subjects of different ability degrees may show at the task. 

We saw an example of the three-parameter logistic model and a linear model in Figure \ref{fig:irc}. Typically, for tasks which are led by bounded rewards, the curves should look something in between of the two models.
Nonetheless, we do not even need to define a parametric model: what we really need is to determine a function that returns the expected reward given an ability level. In order to do this, we need to look at our policy (agent) distribution again.


Consider a population of agents $\Omega$. We can define a probability distribution over them.


\begin{definition}\label{def:polprob}
Given a population of agents, $\Omega$, the {\em a priori policy probability}, denoted $w(\pi)$, is defined as a distribution over $\Omega$.
\end{definition}

This gives more or less weight to policies and represents the probability of finding, using or exploring a policy. We can get a sample $S$ of size $N$ (without replacement) from the set $\Omega$ by using the probability $w(\pi)$. This is denoted by $\Omega||_{w,N}$.

Now let us consider a tolerance value $\gamma$, with $0 \leq \gamma \leq 1$. Given a random sample according to $w$ and a tolerance value $\gamma$, the probability that the sample of size $N$ contains a value sufficiently closer ($\gamma$) to the maximum aggregated reward of the population $\Omega$ is given by:
\begin{equation}
\qpos(\gamma,\Omega,w,N) \triangleq \text{Pr}\left[\max_{\pi \in \Omega||_{w,N}}\{R(\pi) \} \geq (1-\gamma)(\Rmax - \Rmean) + \Rmean\right] \label{eq:qpos}
\end{equation}
\noindent There is a baseline probability, even for extreme tolerance $\gamma = 1$:
\begin{proposition}\label{prop:half}
If $\Rmean$ is equal to the median of the aggregated reward in $\Omega$ and $w(\pi)$ is independent of $R(\pi)$ then $q(1,\Omega,w,1) = 1/2$.
\end{proposition}
\begin{proof}
If $N=1$ we just draw one element, and with $\gamma=1$  we have:
\begin{eqnarray*}
\qpos(1,\Omega,w,1) & = & \text{Pr}\left[\max_{\pi \in \Omega||_{w,N}}\{R(\pi)\} \geq (1-1)(\Rmax - \Rmean) + \Rmean \right] \\
                    & = & \text{Pr}\left[\max_{\pi \in \Omega||_{w,N}}\{R(\pi)\} \geq \Rmean \right]
\end{eqnarray*}							
If the mean and the median match then there are exactly $|\Omega|/2$ elements above $\Rmean$ and the other half would be below, so the probability that taking just one is greater than or equal to $\Rmean$ is exactly 1/2. 
\end{proof}

Interestingly, we can define $\qneg$ in a very similar way to $\qpos$:
\[ \qneg(\gamma,\Omega,w,N) \triangleq \text{Pr}\left[\min_{\pi \in \Omega||_{w,N}}\{R(\pi)\} \leq \gamma(\Rmean - \Rmin) + \Rmin \right] \]
\noindent Similar results as proposition \ref{prop:half} can be obtained for $\qneg$.
Note that the maximum (or minimum) is unique but there might be more than one policy reaching that maximum (minimum). 

We now look for the minimum size of the sample such that the above probability (either $\qpos$ or $\qneg$) is at least $1/2$. In other words, we want to know how large the sample has to be to have probability $1/2$ (or larger) of getting the maximum (up to a tolerance level).
\begin{equation}
\Npos(\gamma, \Omega, w) \triangleq \min \{ N : \qpos(\gamma,\Omega,w,N) \geq 1/2 \} \label{eq:Npos}
\end{equation}
\begin{equation}
\Nneg(\gamma, \Omega, w) \triangleq \min \{ N : \qneg(\gamma,\Omega,w,N) \geq 1/2 \} \label{eq:Nneg}
\end{equation}
From here, we have:
\begin{proposition}\label{prop:one}
Under the same conditions of proposition \ref{prop:half}:
\[ \Npos(1, \Omega, w) = 1\]
\[ \Npos(0, \Omega, w) = |\Omega|/2 \]
For the second equality we also assume that the policy that produces $\Rmax$ is unique.
\end{proposition}
\begin{proof}
The first one follows directly from proposition \ref{prop:half}.
For $\gamma=0$, we have that $\qpos$ simplifies to $\text{Pr}(\Rmax \in \Omega||_{w,N})$. Since $w$ is independent of $\Rmax$, and there is only one policy for $\Rmax$, we need to draw half of the elements of $\Omega$ in order to have exactly probability $1/2$. 

\end{proof}
Similarly for $\Nneg$. 
By varying $\gamma$ between $0$ and $1$ we have a function of the expected value of policies that we need to explore to approach a given value (in this case the actual maximum reward of the population, i.e., the best result). 
Depending on the choice of $w$, the functions $\Npos$ and $\Nneg$ will have different shapes. Since $\Omega$, the set of agent policies, can be very large (or even infinite), and this set is discrete, we need to consider a distribution that is appropriate for infinitely many discrete objects. A good choice is a universal distribution, using some measure of (Kolmogorov) complexity $K$, as follows:
\begin{equation}
w(\pi) \triangleq \frac{2^{-K(\pi)}}{\sum_{\pi' \in \Omega}{2^{-K(\pi')}}}  \label{eq:w}
\end{equation}
\noindent This probability represents that simple policies will be sampled (explored) with much higher probability, which is a very reasonable choice as an a priori probability. Note that if $w(\pi)$ is independent of $R(\mu)$, as the assumptions of propositions \ref{prop:half} and \ref{prop:one}, we necessarily have that $Cor_{\pi \in \Omega}(K(\pi), R(\pi)) = 0$. However, the reverse is not true.
If we use the previous $w$ in Eq. \ref{eq:Npos} we have the following property:

\begin{proposition}\label{prop:mega}
Consider $\pipos$ as any of the shortest policies with maximum reward, i.e., any element in the set $\argmin_\pi \{ K(\pi) : R(\pi) = \Rmax \}$. We have that:
\[  \Npos(0, \Omega, w) \leq 2^{K(\pipos)}{\sum_{\pi' \in \Omega}{2^{-K(\pi')}}} \]
\end{proposition}
\begin{proof}
Let us use the shorthand notation $k^*= K(\pipos)$ and $\wall = {\sum_{\pi' \in \Omega}{2^{-K(\pi')}}}$.
The probability of $\pipos$ appearing in a sample of size $N$ is greater than or equal to $1-(1-w)^N$ (this would be the probability with replacement).
Then, since this policy gives the maximum in $\Omega$, we have that:
\begin{eqnarray*}
\qpos(0,\Omega,w,N) & =   & \text{Pr}\left[\max_{\pi \in \Omega||_{w,N}}\{R(\pi)\} \geq (1-\gamma)(\Rmax - \Rmean) + \Rmean\right] \\
                         & \ge & 1-(1-w)^N 
\end{eqnarray*}												
Since $\qpos$ is non-decreasing with $N$, in order to calculate $N$ we make $\qpos$ equal to $1/2$ in order to have $N_{min}$ :
\begin{eqnarray}
\qpos  = & 1/2                            & \geq 1-(1-w)^N \nonumber \\
     & (1-w)^N                        & \geq 1/2      \nonumber  \\
     & (1-{\frac{2^{-k^*}}{\wall}})^N   & \geq 1/2       \nonumber \\
     & N                              & \leq \log_{1-{\frac{2^{-k^*}}{\wall}}}(1/2) \nonumber \\ 
     & N                              & \leq \frac{\log_2(1/2)}{\log_2(1-{\frac{2^{-{k^*}}}{\wall}})} \nonumber\\
     & N                              & \leq \frac{-1}{\log_2(1-{\frac{2^{-{k^*}}}{\wall}})}           \label{eq:sandwich} 
\end{eqnarray}
Now we use the following variable change $t= 1 - {\frac{2^{-k^*}}{\wall}}$. 
Since $k^*>0$ we have that $0 < t < 1$. So the above inequality becomes:
\begin{equation}
N  \leq \frac{-1}{\log_2(t)}  \label{eq:change}
\end{equation} 
We see now that 
\begin{eqnarray*}
\frac{-1}{\log_2(t)}  & \leq & \frac{1}{1-t} \\
\log_2(t)             & \leq & -1+t \\
\frac{\ln(t)}{\ln(2)} & \leq & -1+t \\ 
\ln(t)                & \leq & (-1+t)\ln(2)   
\end{eqnarray*}
is true because for $0 < t < 1$ we have that $\ln(t) \leq -1+t$ and $(-1+t) < (-1+t)\ln(2)$ since $(-1+t)$ is negative.
Using this in Eq. \ref{eq:change} and undoing the variable change we have:
\begin{eqnarray*}
N & \leq & \frac{-1}{\log_2(t)} \leq \frac{1}{1-t} \\
N & \leq & \frac{1}{{\frac{2^{-{k^*}}}{\wall}}} \\
N & \leq & {(2^{k^*}){\wall}} \\
N & \leq & {2^{k^*}{{\sum_{\pi' \in \Omega}{2^{-K(\pi')}}}}}  
\end{eqnarray*}
In fact, it can also be shown that this bound would be tight (if we used replacement in the sample), because if $1/2 \geq (1-{\frac{2^{-(k^*-1)}}{\wall}})^N$ then Eq. \ref{eq:sandwich} would be an equality. Figure \ref{fig:sandwich} illustrates how these two functions would make tight bounds.

\end{proof}

\begin{figure}
	\centering
	  \vspace{-1.5cm}
		\includegraphics[width=0.6\textwidth]{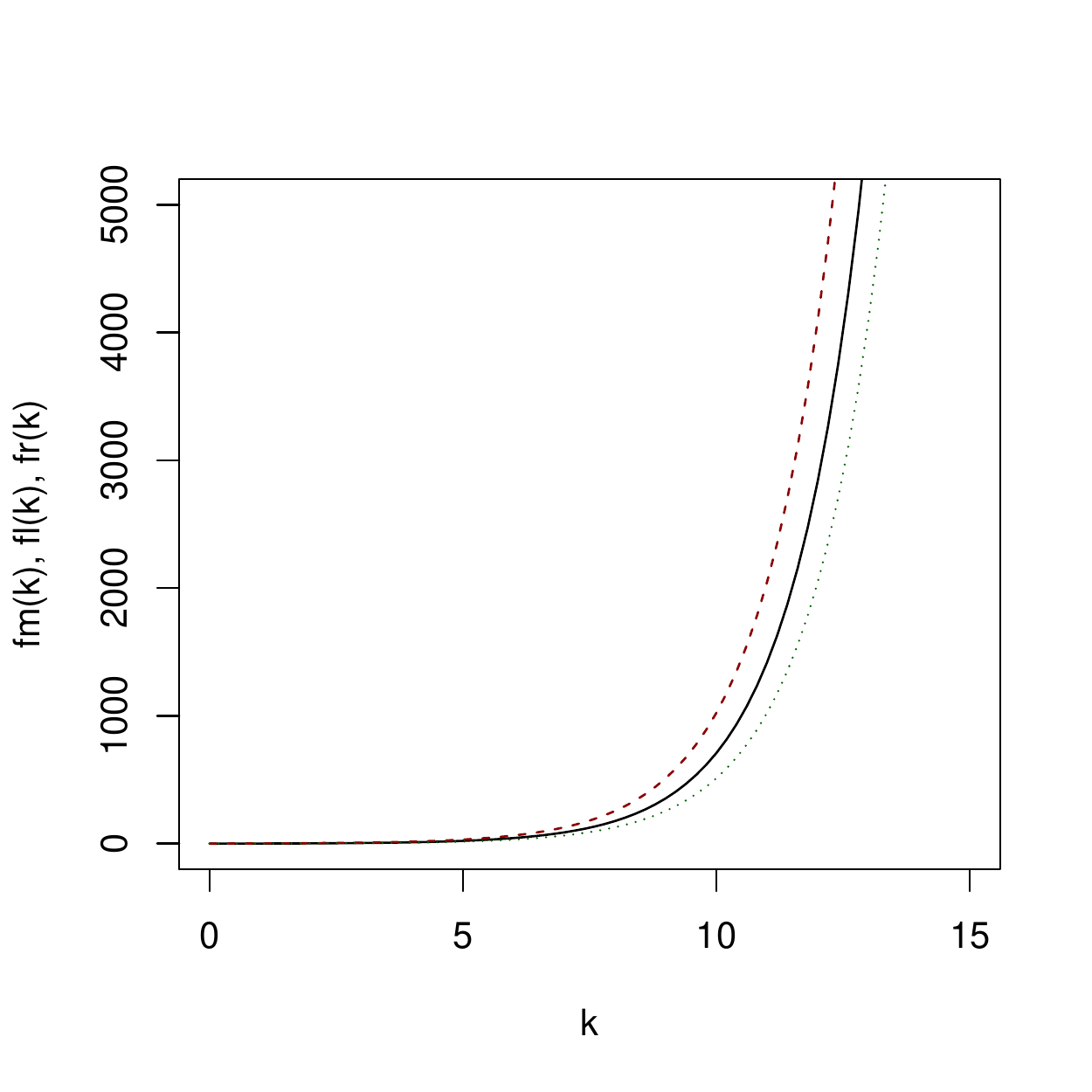}
	\vspace{-0.5cm}
		\caption{Example of the approximation made by Eq. \ref{eq:sandwich}. If instead of $fm = \frac{-1}{\log_2(1-{\frac{2^{-k}}{\wall}})} $, we use $fr= {(2^k){\wall}}$ and $fl = {(2^{k-1}){\wall}}$, $fm$ lies tightly in between. The plot shows these threes functions ($fm$ solid, $fr$ dashed, $fl$ dotted), all with $\wall=1$. }
	\label{fig:sandwich}
\end{figure}


Note that the normalisation factor is 1 when $K$ is a complexity function such that the universal measure derived from it is normalised. In any case, this is a constant value which only depends on the environment.

The previous result is also applicable to the shortest policies with minimum reward, $\pineg$.
In both cases, the expression suggests that in order to have a measure which is commensurable with the value of complexity $k$ from which the probability $w$ is used, we need to apply a logarithm to $N$. This leads us to propose the following metric:
\begin{equation}
\Dpos(\gamma, \Omega, w) \triangleq \log_2 \Npos(\gamma, \Omega, w) \label{eq:Dpos}
\end{equation}
\begin{equation}
\Dneg(\gamma, \Omega, w) \triangleq \log_2 \Nneg(\gamma, \Omega, w) \label{eq:Dneg}
\end{equation}
\noindent Assuming $w$ and $\Omega$ are fixed in what follows, this gives functions of difficulty in terms of $\gamma$, i.e., $\Dpos(\gamma)$ and $\Dneg(\gamma)$. Note that, from propositions \ref{prop:half} and \ref{prop:one}, we have that $\Dpos(1) = \log_2 1 = 0$, i.e., the difficulty with tolerance $\gamma=1$ is just 0. Similarly for $\Dneg$. 
The maximum value of difficulty under independence of $w(\pi)$ and $R(\pi)$ would be $\log_2 |\Omega|/2 = (\log_2 |\Omega|) - 1$. Note that we will usually calculate $\Dpos$ with a sample. If this sample is small, the estimation will be below the actual value of $\Dpos$. As the sample becomes larger, the value will converge to its actual value (note that rewards are bounded). 

Now, if we are able to estimate $\Dpos$, mapping $\gamma$ into $k$, we can derive $(\Dpos)^{-1}$, which returns a value of $\gamma$ that ensures 1/2 probability of finding that level of aggregated reward using a sample size related to $2^k$. Instead of interpreting the input as a difficulty, we can see as an ability of the agent, leading to curve very much like the item response curve. Also, we can just translate the output of $(\Dpos)^{-1}$ to an aggregated reward by applying the resulting tolerance expression, and defining $\Repos(\theta) \triangleq (1-(\Dpos)^{-1}(\alpha))(\Rmax - \Rmean) + \Rmean$, for $\theta \geq 0$ with output between $\Rmean$ and $\Rmax$. This non-decreasing function can be plotted in the form of {\em environment response curves}. We can do similarly for $\Dneg$ leading to $\Reneg(\theta) \triangleq (\Dneg)^{-1}(-\theta)(\Rmean - \Rmin) + \Rmin$, for $\theta \leq 0$ with output between $\Rmean$ and $\Rmin$.
Putting both things together gives the following function:
\begin{definition}
The {\em environment response function} (or curve) is defined as:
\begin{eqnarray*}
\Re(\theta) & \triangleq & \Reneg(\theta) \:\: \mbox{if} \:\: \theta < 0  \\
						&						 & \Repos(\theta)  \:\: \mbox{otherwise}  
\end{eqnarray*}
\end{definition}
Note that the above curves are not normalised, because we recover the original scale given by $\Rmax$, $\Rmin$ and $\Rmean$. By just setting these terms to $1$, $-1$ and $0$ respectively, we can attain a normalised environment response curve. An illustrative example is shown in Figure \ref{fig:erc}.

\begin{figure}
	\centering
	  \vspace{-1.5cm}
		\includegraphics[width=0.6\textwidth]{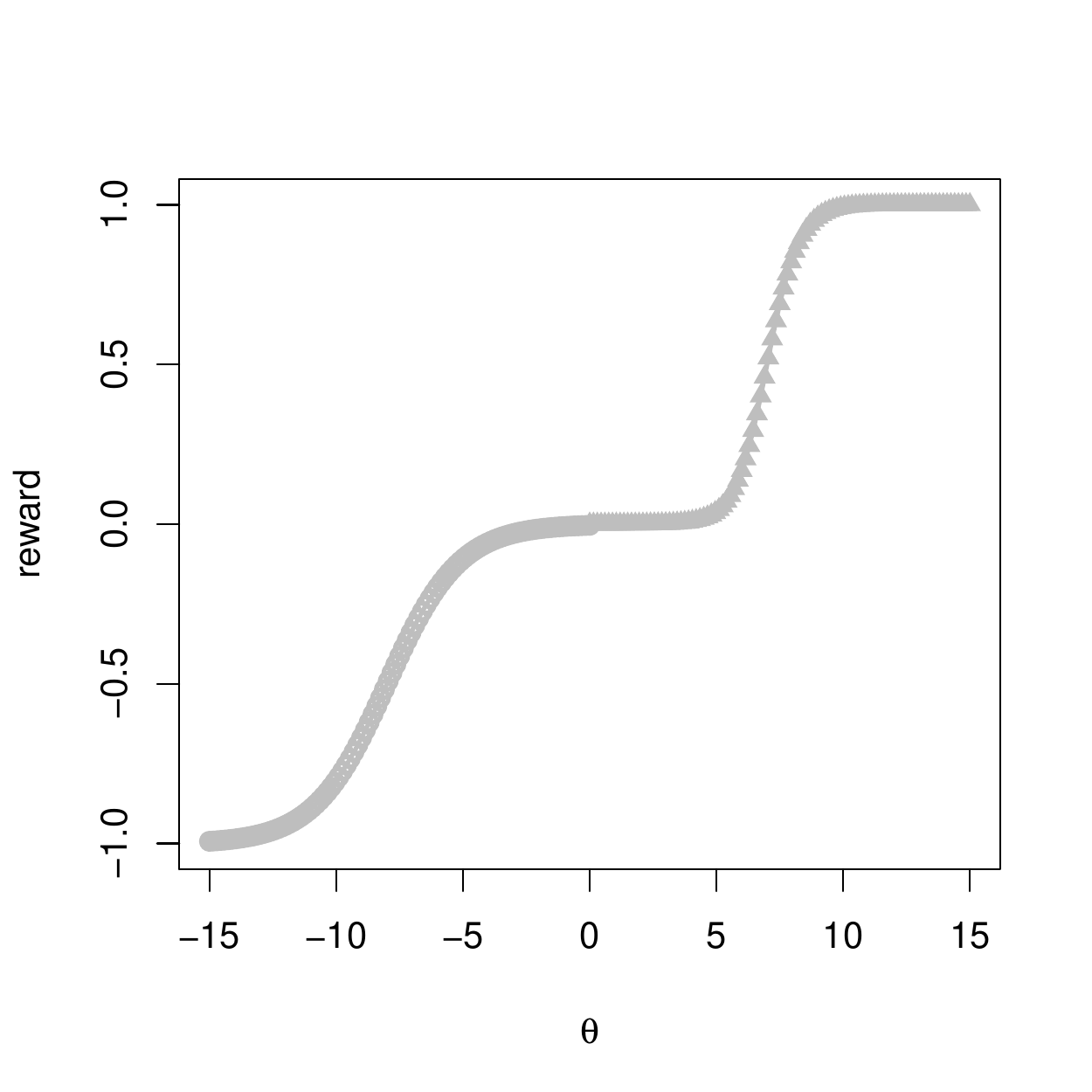}
	  \vspace{-0.5cm}
	\caption{Example of a normalised {\em environment response curve}. Given an expected ability $\theta$ on the \xaxis we get an expected aggregated reward on the \yaxis. If we only focus on the positive (top right) part of the plot, we can see that the plot saturates at about $\theta=8$. For different values of tolerance $\gamma$ we can also determine how much ability is needed. The bottom left part of the plot shows that any environment is dual, and by negating the rewards we could have an environment behaving as it is shown in the bottom left part of the plot.}
	\label{fig:erc}
\end{figure}

We have used $\theta$ for the ability, as in Item Response Theory. The use of a negative ability may look strange initially, but it has a good interpretation. If we see the value of $\theta=0$ as the ability that  is expected to score as the average of all policies, then performing worse than that can only happen by chance or because the agent is making an effort to get bad rewards. This leads to a duality in these plots.


\section{Case study: agent-populated elementary cellular automaton}\label{sec:experiments}


In theory, we can analyse difficulty and discriminating power for any possible kind of environment which is compatible with definition \ref{def:environment}, i.e., any discrete-time deterministic environment. In this section, we are going to choose a simplistic setting for practical reasons. First, we are interested in minimalistic environments where the number of observations and actions are extremely reduced, so having some relatively rich phenomena with very simple transition functions. Second, we are interested in simplistic policy languages in order to be able to evaluate a large amount of agents quickly. Third, we want to elaborate on settings that are well known and previously studied in terms of emergence and complexity. The configuration we present next follows these three conditions.

\subsection{Agent-populated elementary cellular automata: definition and examples}

The environments we will work with are will use elementary cellular automata (ECA) \cite{wolfram2002new} for the space $\sigma$ and the transition function $\tau$, 
and will let the agent see and modify part of the usual behaviour of the automaton. The following definition specifies the complete behaviour of this kind of the environment:

\begin{definition} \label{def:cellular}
A single-agent elementary cellular automaton ({\sffamily SAECA}) is a special kind of environment 
$\left\langle \sigma, \tau, \omega, \rho, \Pi \right\rangle$, as seen in definition \ref{def:environment}, with the following extra parameters $\left\langle \sigma^0, \nu, p^0\right\rangle$.
The state $\sigma$ is represented by a one-dimensional array of bits or cells $\sigma_1, \sigma_2, \dots, \sigma_n$, also known as {\em configuration}. We consider the array to be finite (length $|\sigma|=n$) but circular in terms of neighbourhood ($\sigma_{0} = \sigma_n$ and $\sigma_{n+1} = \sigma_1$).
There is an initial array $\sigma^0$, also know as {\em seed}. The transition function $\tau$ is given by $\nu$, as any of the $2^{2^3}=256$ rules that can be defined looking at each cell and its two neighbours according to the numbering scheme convention introduced in \cite{wolfram2002new}.
Given the behaviour of the space, we consider just one agent in $\Pi$. The agent is located at one cell (its position $p$) with $1 \leq p \leq n$, which is initially $p^0$. The set of observations ${\cal{O}}$ is given by three bits $\left\langle c, l, r \right\rangle$ representing the contents of the 
left and right neighbouring cells respectively, i.e., $\sigma_{p-1}$ and $\sigma_{p+1}$. 
The actions ${\cal{A}}$ are given by a pair of `move' and `upshot', denoted by $\left\langle V, U \right\rangle$. The ordered set of moves is given by $\{$ {\tt stay, right, left} $\}$, and the ordered set of upshots is $\{$ {\tt keep, swap, set0, set1} $\}$, which respectively mean that the content of the cell where the agent is does not change, the content of the cell is swapped ($0 \rightarrow 1$, $1 \rightarrow 0$), the content is set to 0 and the content is set to 1.
The rewards are calculated in the following way. If the agent is at position $p$ at time $t$, then we use this formula:
\[ r^t \leftarrow \sum_{i=1..\left\lfloor n/2 \right\rfloor} \frac{\sigma^t_{pos+i} + \sigma^t_{pos-i}}{2^{i+1}}  \] 
which counts the number of 1s which are in the neighbourhood of the agent, weighted by their proximity. It is easy to see that $0 \leq r^t \leq 1$.

The order of events for each step in the system is: observations are produced, actions are performed, the automaton is updated and finally, rewards are produced.
\end{definition}

Note that the environment is parametrised by the original contents of the array $\sigma^0$ (including its size), the ECA rule number $\nu$, and the original position of the agent $p^0$. Given an environment and a computable agent, the evolution of the system is computable and deterministic.

In order to explore the results for different policies, we need a language for expressing them. There are a few agent languages in the literature, but they are too oriented towards the architecture, are too focussed on Markov Decision Processes or are not sufficiently minimalistic for an exhaustive search over the policies (see, e.g., \cite{andre2003programmable}). As a result, we just a introduce a new language, {\sf APL}:

\begin{definition} \label{def:policy1}
The agent policy language {\sf APL} is given by a memory (or history) binary array $m$, initially empty (and not circular), and an ordered set of instructions ${\cal{I}}$ = $\{$ {\tt back}=0, {\tt fwd}=1, {\tt Xaddm}=2, {\tt Xadd1}=3, {\tt Yaddm}=4, {\tt Yadd1}=5 $\}$. 
A program or policy $\pi$ is a sequence of instructions $\iota_1, \iota_2, ..., \iota_m$ in ${\cal{I}}$.
The interpreter works on its memory by using two accumulators $V$ and $U$, and the action is given by the result of the accumulators at the end of the process. Namely:
\begin{enumerate}
\item Read the observation $\left\langle c, l, r \right\rangle$ and append it to the history array $m$. 
\item Place the memory pointer $b$ at the end of $m$.
\item $V \leftarrow$ {\tt stay}
\item $U \leftarrow$ {\tt keep}
\item forall $\iota_j$ 
\item  $\:\:$ case $\iota_j$:
\item  $\:\:\:\:\:\:$ {\tt back} $\:\:$     :$\:\:$ $b \leftarrow \max(b - 1, 1)$
\item  $\:\:\:\:\:\:$ {\tt fwd}  $\:\:\:\:$ :$\:\:$ $b \leftarrow \min(b + 1, |m|)$  
\item  $\:\:\:\:\:\:$ {\tt Vaddm}$\:$       :$\:\:$ $V \leftarrow (V + m_b) \: \text{mod} \: 3$ 
\item  $\:\:\:\:\:\:$ {\tt Vadd1}$\:$       :$\:\:$ $V \leftarrow (V + 1) \: \text{mod} \: 3$  
\item  $\:\:\:\:\:\:$ {\tt Uaddm}$\:$       :$\:\:$ $U \leftarrow (U + m_b) \: \text{mod} \: 4$ 
\item  $\:\:\:\:\:\:$ {\tt Uadd1}$\:$       :$\:\:$ $U \leftarrow (U + 1) \: \text{mod} \: 4$  
\item  $\:\:$ end case  
\item endfor 
\item return $\left\langle V, U \right\rangle$
\end{enumerate}
\end{definition}

The language is clearly not universal, and all programs end. The goal of this language is not to be easily programmable but to be able to express some simple policies that may be useful in the environment.

Let us see a few examples of how these environments and agents work.
First, Figure \ref{fig:manyrules} shows the evolution of several environments with seed ``010101010101010101010'', $p_0$= 11 and several values of $\nu$. We include an {\em inert} agent, i.e., an agent that does not affect the environment (i.e., an empty policy $\pi$). As a result, the resulting matrix after 200 iterations is the same as a classical elementary cellular automaton with each number $\nu$, with patterns that are well-known (see, e.g., \cite{wolfram2002new}).

\begin{figure}
	\centering
		\vspace{-1cm}
		\hspace{-1.5cm} 
		\includegraphics[width=0.24\textwidth]{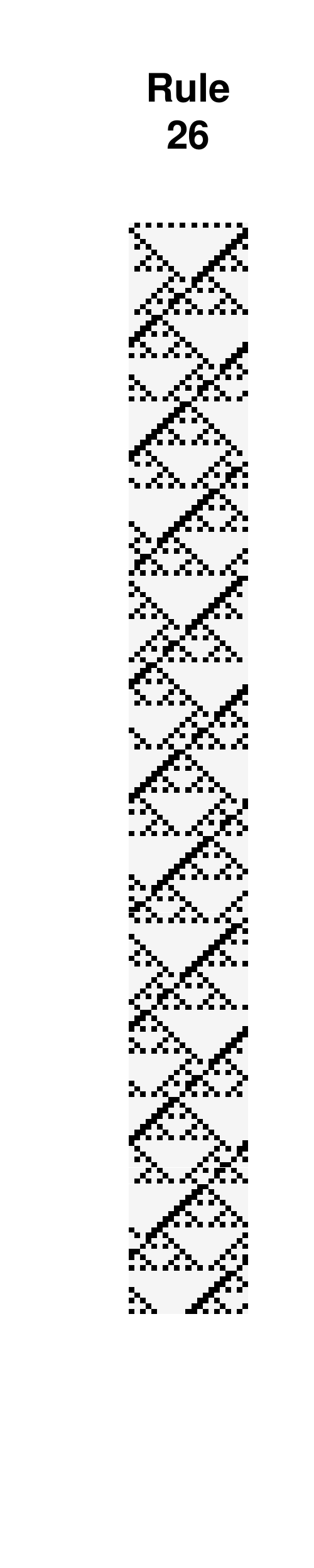}\hspace{-1.5cm}
		\includegraphics[width=0.24\textwidth]{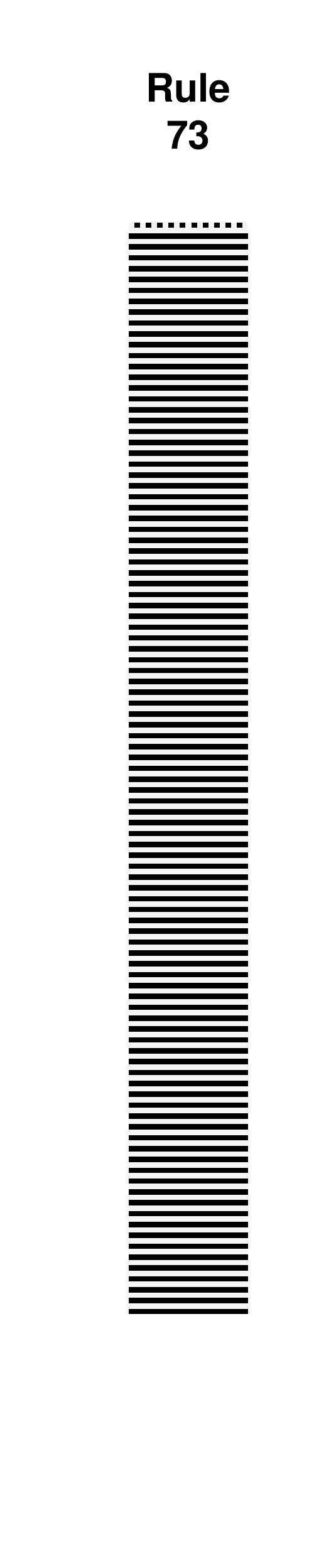}\hspace{-1.5cm}
		\includegraphics[width=0.24\textwidth]{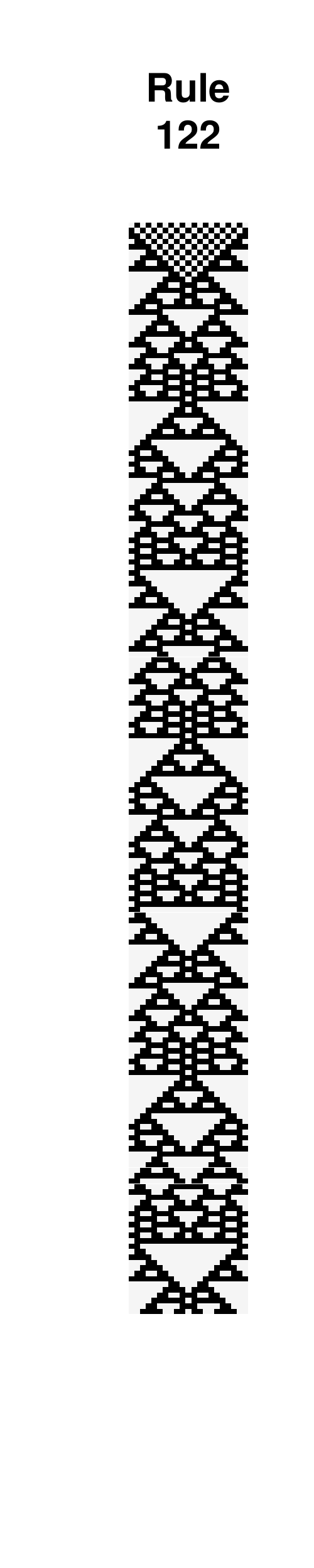}\hspace{-1.5cm}
		\includegraphics[width=0.24\textwidth]{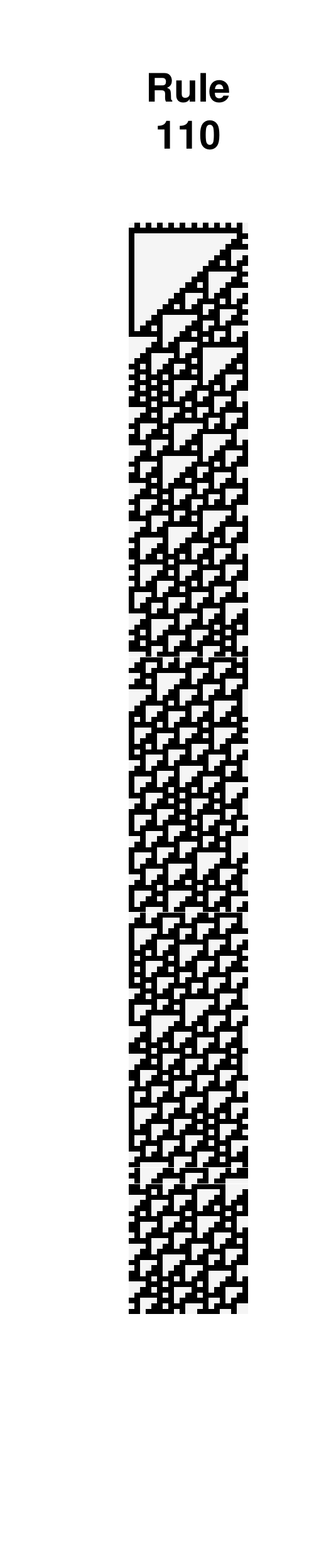}\hspace{-1.5cm}
		\includegraphics[width=0.24\textwidth]{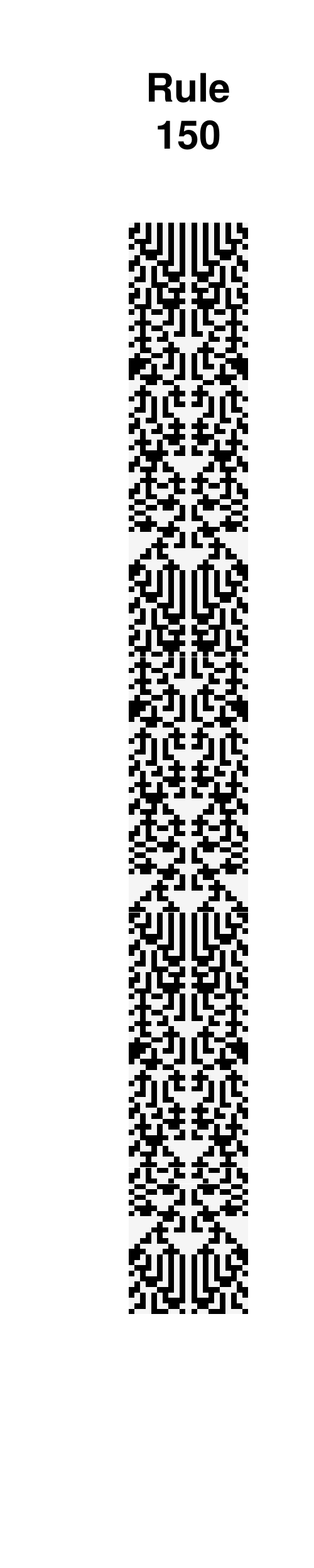}\hspace{-1.5cm}
		\includegraphics[width=0.24\textwidth]{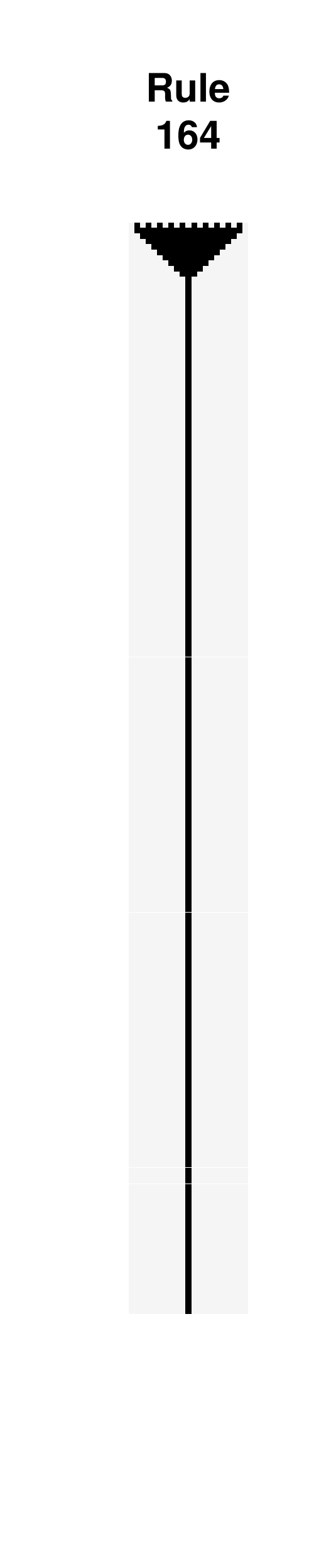}\hspace{-1.5cm}
		\includegraphics[width=0.24\textwidth]{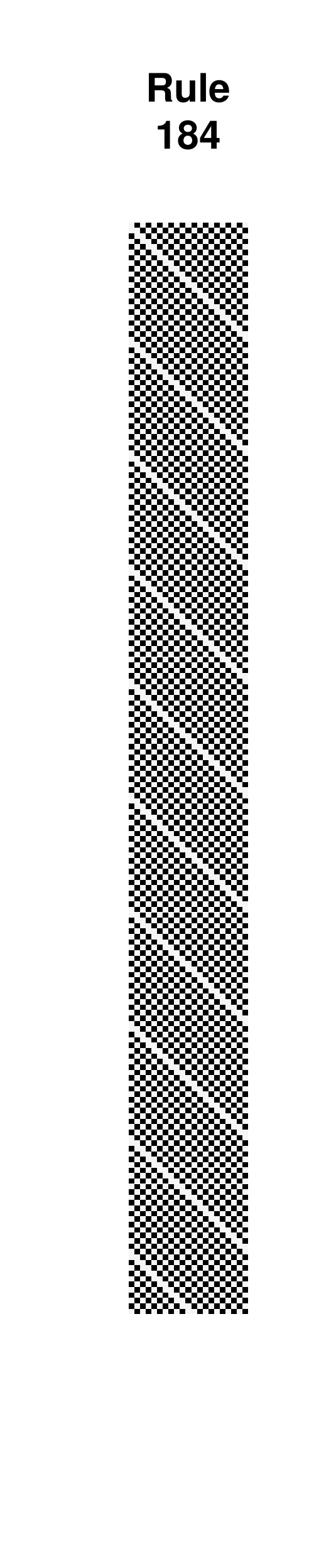}\hspace{-1.5cm}
		\vspace{-2.5cm}
	\caption{Space-time diagram (evolution for $t=200$ steps) of several elementary cellular automata without agent (or with an agent with a policy that always does actions of the form $\left\langle -, \text{\tt keep} \right\rangle$, as the empty policy). The initial array (seed) is always 010101010101010101010, whose length is 21 bits. }
	\label{fig:manyrules}
	\vspace{-1cm}
\end{figure}

Things change when we incorporate an active agent in the environment. Figure \ref{fig:manypolicies} shows how the environment with elementary cellular automaton number 110 varies for several agent policies. The resulting matrix patterns are different. Similar things (where differences are more visible) happen with rule number 164 (Figure \ref{fig:manypolicies2}).

\begin{figure}
	\centering
		\vspace{-1cm}
		\hspace{-1.5cm} 
		\includegraphics[width=0.22\textwidth]{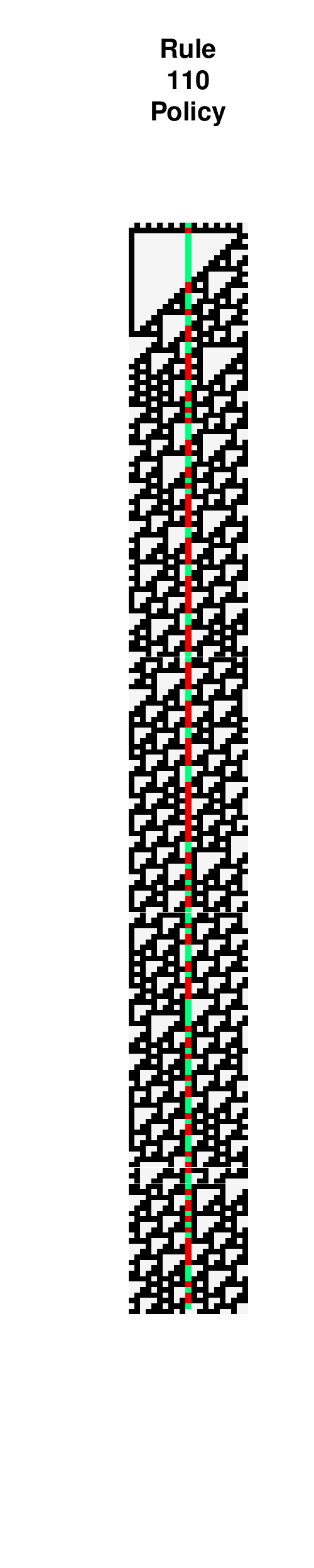}\hspace{-1.5cm}
		\includegraphics[width=0.22\textwidth]{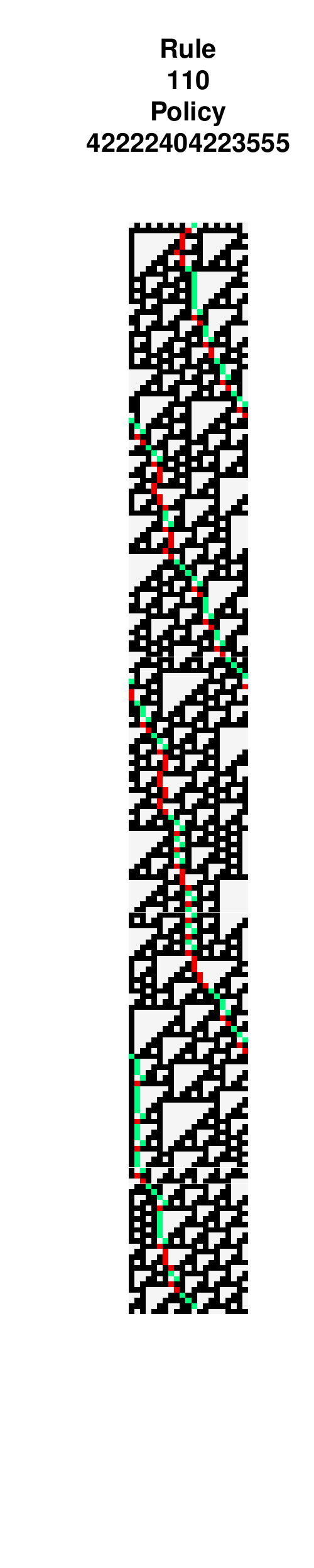}\hspace{-1.5cm}
		\includegraphics[width=0.22\textwidth]{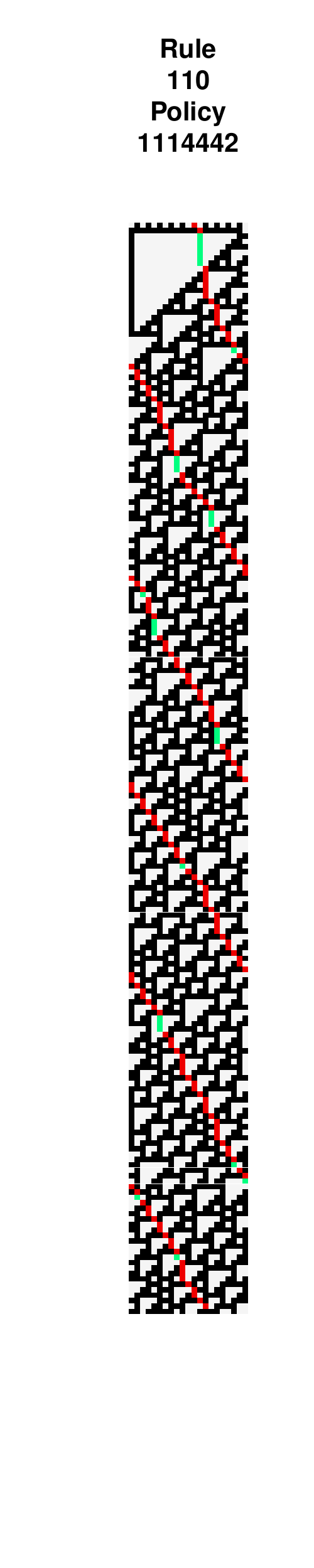}\hspace{-1.5cm}
		\includegraphics[width=0.22\textwidth]{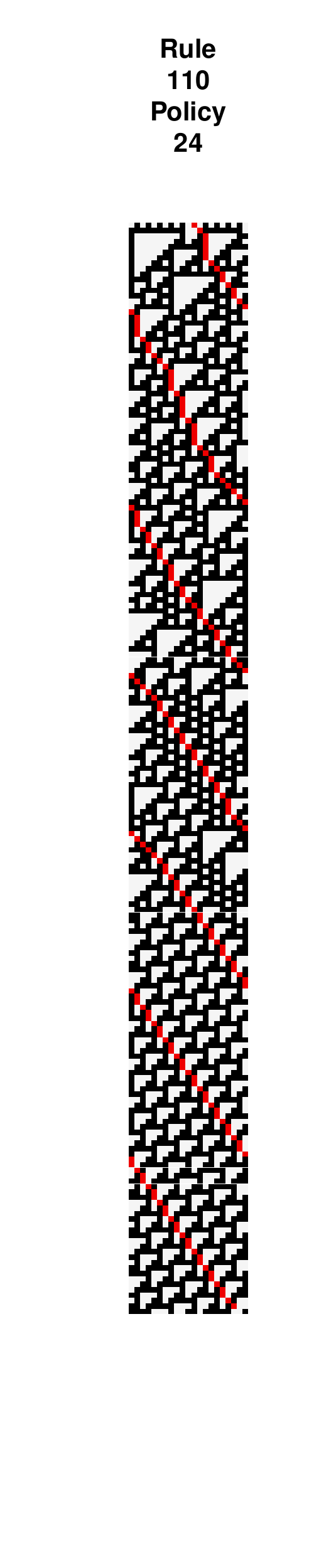}\hspace{-1.5cm}
		\includegraphics[width=0.22\textwidth]{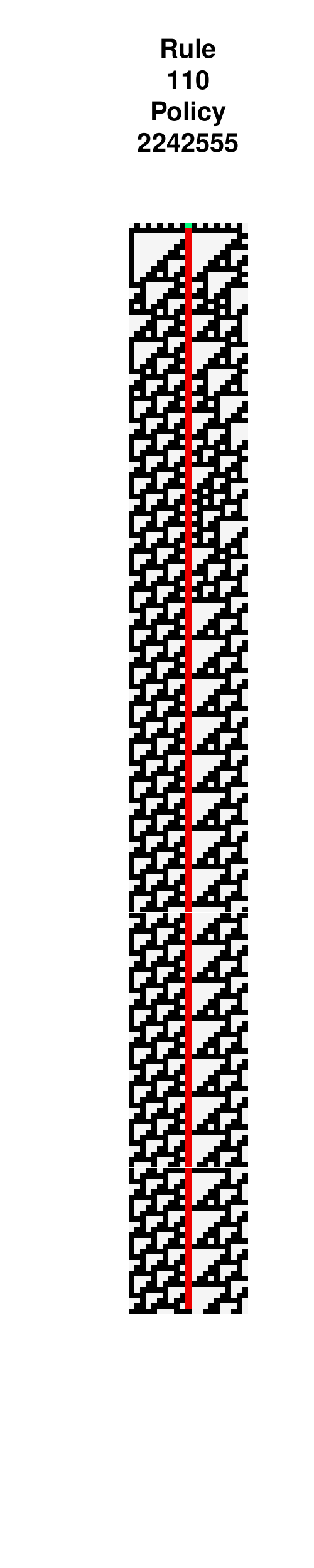}\hspace{-1.5cm}
		\includegraphics[width=0.22\textwidth]{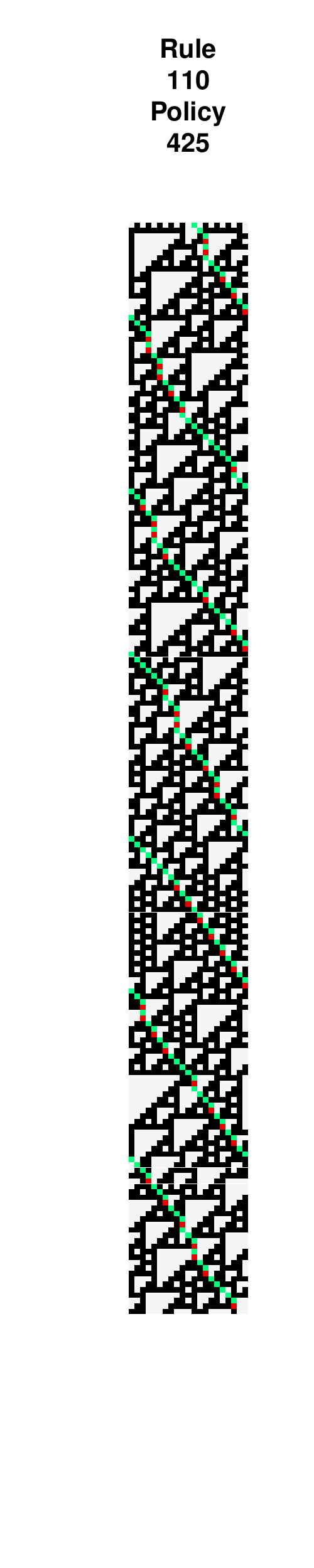}\hspace{-1.5cm}
		\includegraphics[width=0.22\textwidth]{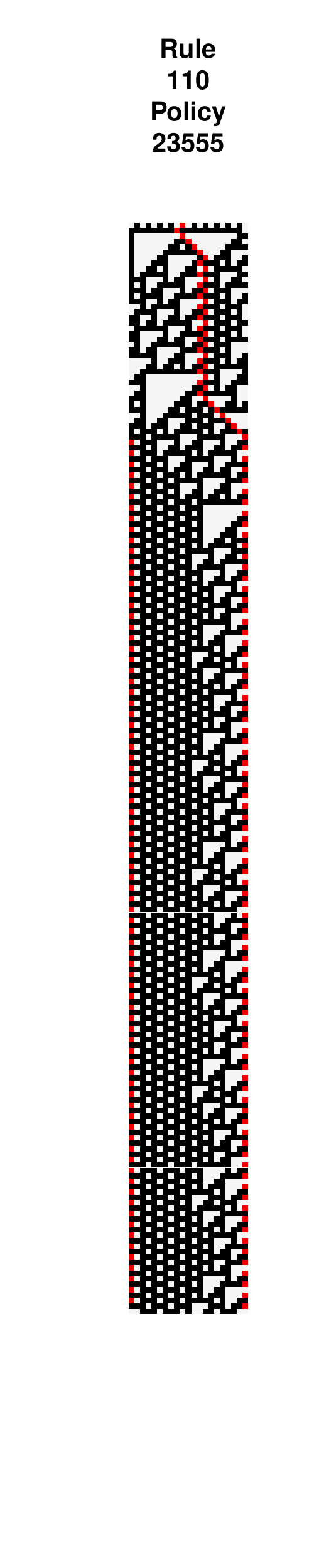}\hspace{-1.5cm}
		\includegraphics[width=0.22\textwidth]{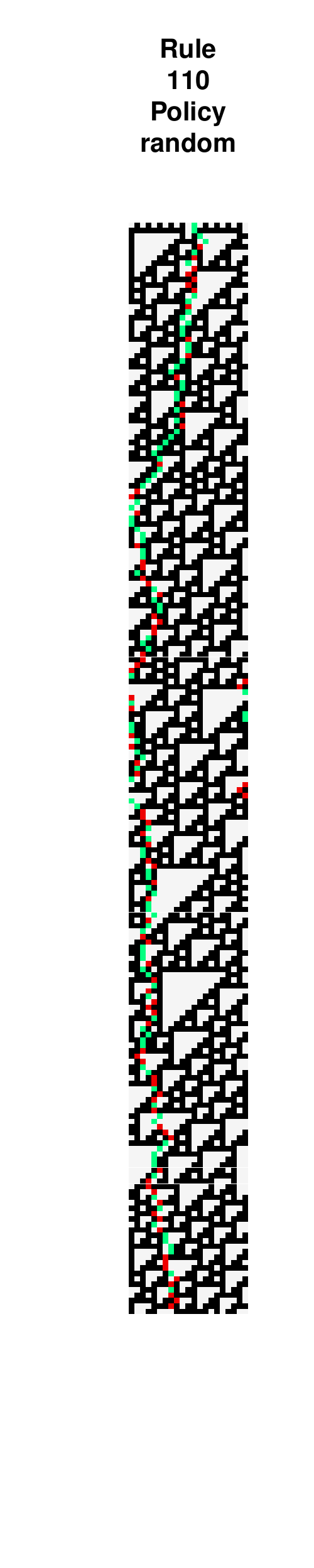}
		\hspace{-1cm}
		\vspace{-2.5cm}
	\caption{Space-time diagram (evolution for $t=200$ steps) with different agent policies for the elementary cellular automata with rule 110. The initial array (seed) is always 010101010101010101010, whose length is 21 bits. The agent is represented by a red dot when the cell has a 0 (like white ones) and by a green dot when the cell has a 1 (like the black ones). The leftmost diagram is the empty policy ($\left\langle \text{\tt stay}, {\tt keep} \right\rangle$) and the rightmost one is a random-walk policy (it cannot be described with {\sffamily APL}). }
	\label{fig:manypolicies}
	\vspace{-1cm}
\end{figure}

\begin{figure}
	\centering
		\vspace{-1cm}
		\hspace{-1.5cm} 
	  \includegraphics[width=0.22\textwidth]{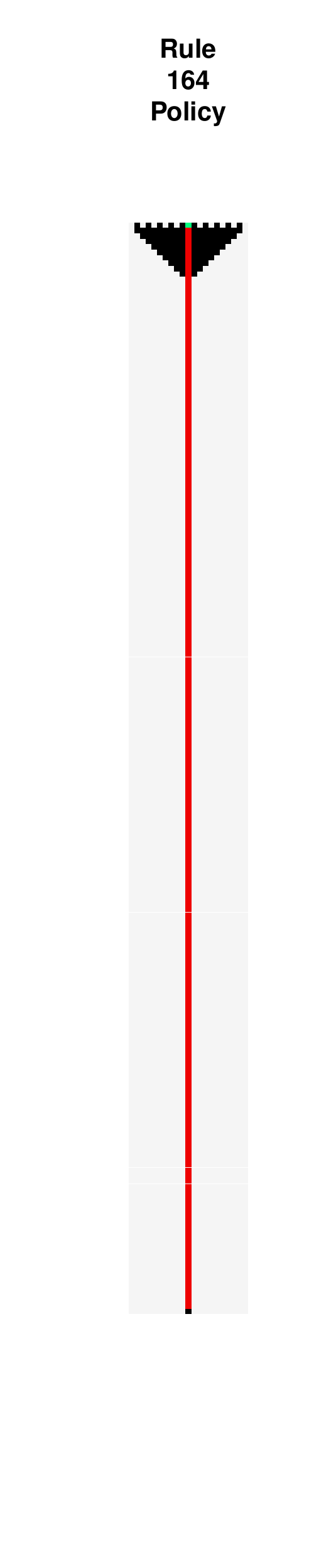}\hspace{-1.5cm}
		\includegraphics[width=0.22\textwidth]{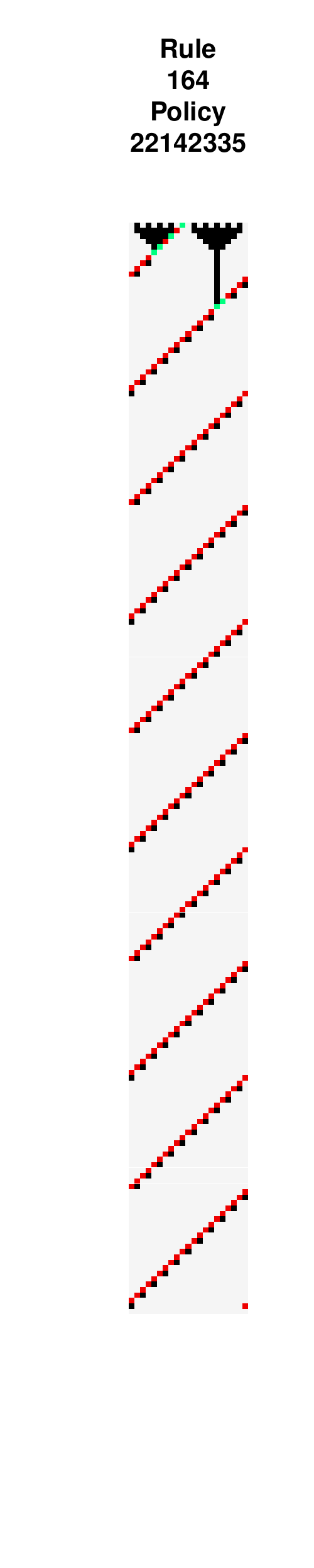}\hspace{-1.5cm}
    \includegraphics[width=0.22\textwidth]{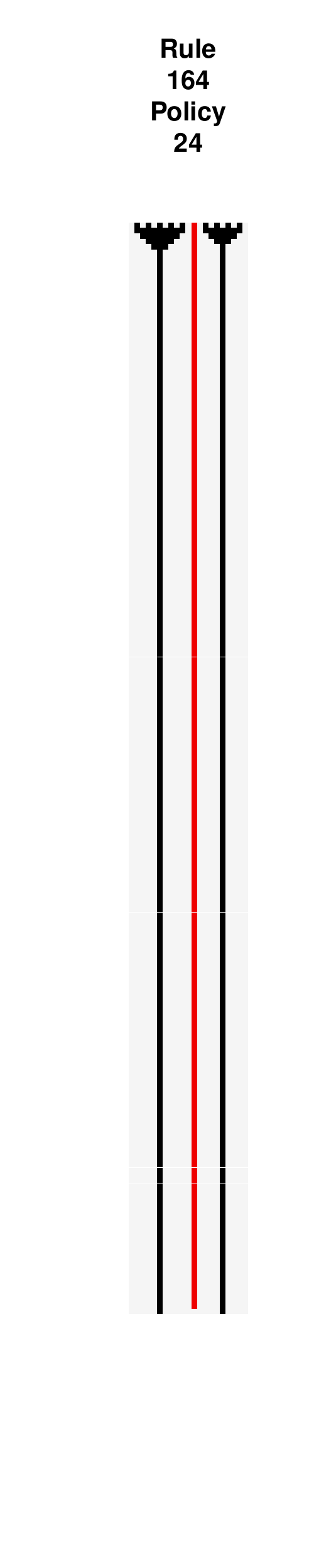}\hspace{-1.5cm}
		\includegraphics[width=0.22\textwidth]{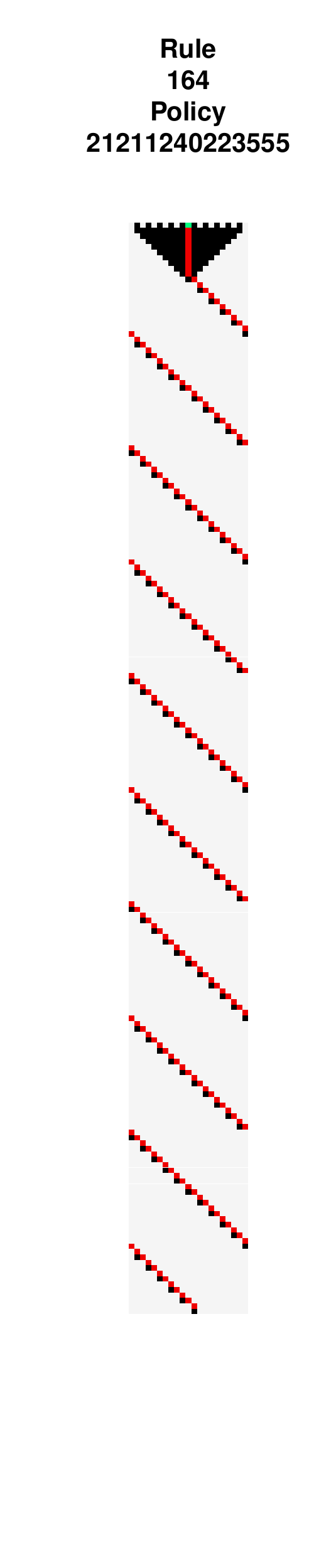}\hspace{-1.5cm}
    \includegraphics[width=0.22\textwidth]{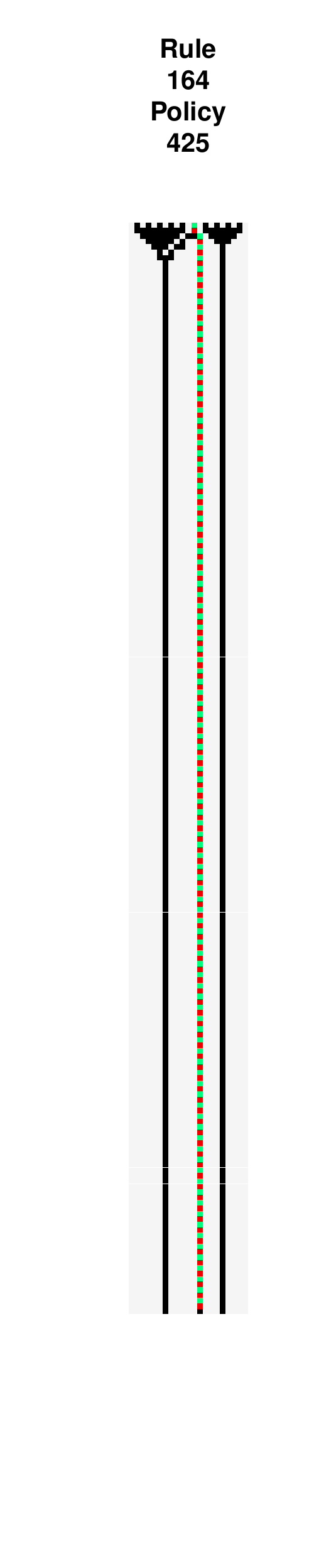}\hspace{-1.5cm}
		\includegraphics[width=0.22\textwidth]{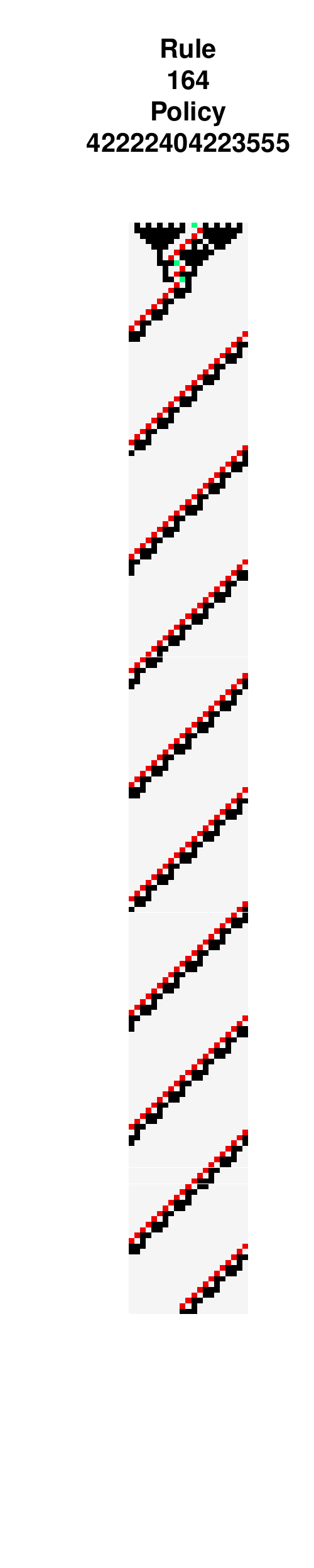}\hspace{-1.5cm}
		\includegraphics[width=0.22\textwidth]{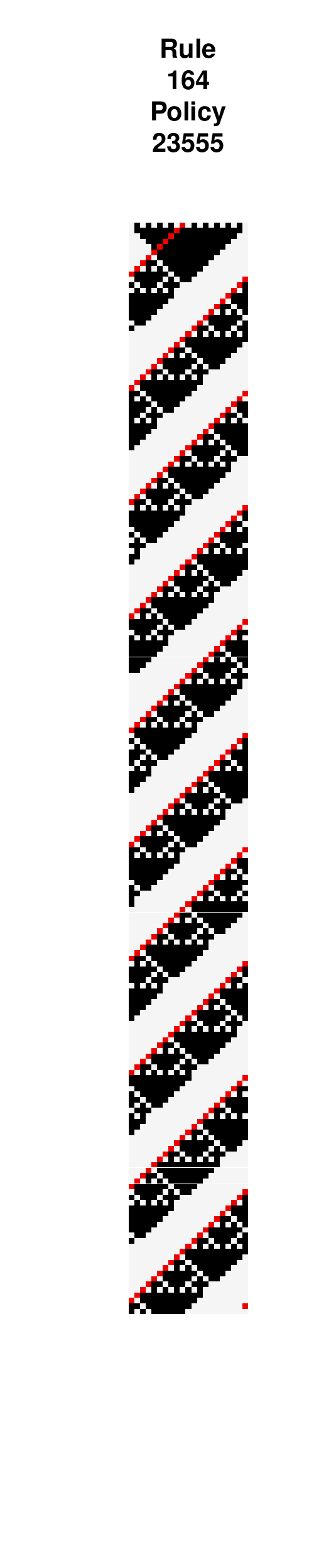}\hspace{-1.5cm}
		\includegraphics[width=0.22\textwidth]{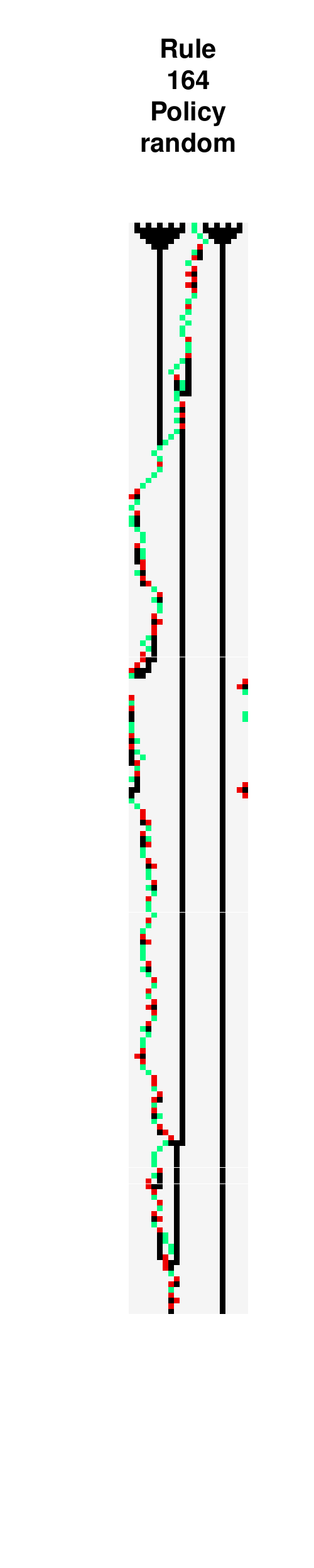}
		\hspace{-1cm}
		\vspace{-2.5cm}
	\caption{Same as Figure \ref{fig:manypolicies}, with rule 164 and other policies. }
	\label{fig:manypolicies2}
	\vspace{-1cm}	
\end{figure}

\subsection{Experimental setting}\label{sec:strategies}

And now let us see how the plots, indicators and curve introduced in section \ref{sec:distribution} can be applied to this particular environment class. The idea is to evaluate a population of policies $\Omega$ against several environments. Of course, in order to make this evaluation finite, we set a fixed number of steps $t$. This is meant to resemble a single agent trying different behaviours until it finds the optimal one. As usual, an agent trying to learn the optimal policy will do a mixture of exploration, exploitation and also `mind practising'. By `mind practising' we mean that an agent can construct a model of the environment and try some policies on that model `mentally' without really implementing them into real actions. Taking into account this variety of strategies for a learning agent, we may think in at least four possible ways of evaluating policies in an environment:

\begin{enumerate}
\item {\em Fresh}: We could think that whenever an agent tries a new policy in $\Omega$, it starts as if the environment were brand new, so the seed configuration is reinitialised for each policy. This is only realistic in environments which are ergodic or where bad actions do not lead to local heavens or hells. On the contrary, this option would be completely unrealistic if we just try a couple of policies and fall into a trap, making it impossible to try other policies.
\item {\em Random walk}: We start with a brand new environment each time, but we execute a sequence of random actions before trying each policy. In other words, we calculate $R^t$ from the environment after a random walk of $t'$ steps. 
\item {\em Chained}: given a set of policies $\Omega$ we want to consider, we just choose a random order and evaluate them sequentially, without resetting the environment to its initial state. We calculate $R^t$ for each policy as usual.
\item {\em Levin search}: same as above ({\em Chained policies}) but the policies are ordered ascendingly by the length of their description. 
\end{enumerate}

\noindent The three last strategies are able to `detect' when the environment is `spoilt' by a bad policy and set into an unresponsive state where rewards cannot be modified any more. If this is the case, this will be shown in the distribution plots. For instance, in some ECAs, a previous policy may leave the configuration without 1s and the environment will not be able to easily recover from there. 

The two last strategies can possibly use their memory and discard some policies according to their models, without actually trying them, as if their weight (their a priori policy probability $w$ as for definition \ref{def:polprob}). Although we think that the first strategy is unrealistic, we have done experiments with the four of them. In what follows, we only show results with the {\em Fresh} and the {\em Chained} strategies, as these are more representative. In both cases, we interweave a 5\% of random-walk policies (note that random-walk policies are non-deterministic and cannot be expressed by the language  {\sf APL}). This random actions will also be useful to get information about any bias of the environment. Finally, in the case of the {\em Chained} strategy, this percentage of random walks actually makes the experiment a hybrid between {\em Random walk} and {\em Chained}.

Once decided how the search for policies is modelled, we now describe more details about the experiments.
We set the number of steps to 300. The seed configuration is always ``010101010101010101010'' and the start-up agent position is $p_0= 11$.
The set of policies $\Omega$ contains 2,000 policies. As mentioned above, 100 are random-walk policies (these are interwoven regularly between the others in the {\em Chained} strategy). The remaining 1,900 policies are randomly generated using a uniform distribution (since $\Omega$ is finite we can do this for $w$ in definition \ref{def:polprob}) for the size of the {\sf APL} program (between 1 and 20 instructions), so making up about 95 policies each. Given the size of the program, we choose each instruction with a uniform distribution among the six instructions in {\sf APL}.
Except for small lengths, policies are rarely equal, but they can be {\em equivalent}. For instance, a {\tt back} instruction followed by a {\tt forward} instruction cancel, many instructions at the end are useless if there is no
{\tt Vaddm},  {\tt Vadd1},  {\tt Uaddm},  {\tt Uadd1} after them, as well as some other simplifications. Because of this, we built a simplifier that transforms programs (when possible) into shorter ones. The simplifier is not exhaustive and some programs may have shorter equivalent version that the simplifier is not able to detect. Nonetheless, we estimate that it covers most cases. The application of the simplifier makes that the frequencies of programs of short size is slightly increased and the number of programs of high size (e.g., $>15$) is reduced considerably. This phenomenon has to be taken into account when looking at the plots and results.

Finally, in order to approximate the Kolmogorov complexity for each policy, we applied a compression algorithm to all the programs. In particular, we used a common approach \cite{Zenil2010,AGI2011Evaluating}: we coded the program as a character string and compressed it (using the $memCompress$ function in R \cite{Rproject}, a GNU project implementation of Lempel-Ziv coding). As a normalisation, from the resulting complexity we subtracted the length of the compressed size of an empty string. This makes the complexities range between $\kmin=0$ and $\kmax=20$. Given that Lempel-Ziv works with bytes and {\sf APL} has only 6 instructions, the complexity has to be understood as being multiplied by $\log_2(6)= 2.58$ and not by $\log_2(256) = 8$. In any case, this factor is the same for all policies so it is irrelevant for our study. The important thing is that the joint use of the simplification and the compression processes can give a relatively acceptable (and efficient) approximation of Kolmogorov complexity for these short strings.

\subsection{Analysis of the distribution}

As we have mentioned above, we performed experiments using a {\em Fresh} strategy and a {\em Chained} strategy, with a 5\% of random-walk policies. Figure \ref{fig:dist} shows the distributions $R[k]$ for several environments (rules 184, 110, 122, 164) using the {\em Fresh} strategy. In the Figure, each small circle (in grey) shows a policy. The box plots and the envelopes are shown for the accumulated values ($[\leq k]$). We show complexity ($k$) on the \xaxis and aggregated reward ($R$) on the \yaxis. As mentioned above, there is a higher concentration of values of $k$ between 5 and 10, because policies are compressed from a uniform distribution between 0 and 20. 

We see that the plots are very different in terms of average, range and evolution.
The environment with rule 110 evolves very quickly for the maximum envelope, while rule 122 has a slower trend. The environment with rule 184 has very wide range and high dispersion. On the contrary, rule 164 has a small range (especially for rewards below the average). And there are also very important differences in terms of average aggregated reward ($\Rmean$) as well. 
Some environments show some clusters. For instance, the environment with rule 184 shows a cluster of values around $R=0.95$, as well as at $0.7$, at $0.5$, $0.3$, and $0.05$. There are some areas that are almost empty, even for high values of $k$, such as the range between $0.8$ and $0.9$.
Finally, the results for the random-walk policies (shown on the left of the plots) are usually very compact, showing that an analysis using random actions (instead of policies) is not able to show the structure of the problem.

\begin{figure}
	\centering
		\vspace{-1.0cm}
		\hspace{-0.3cm} 
\includegraphics[width=0.5\textwidth]{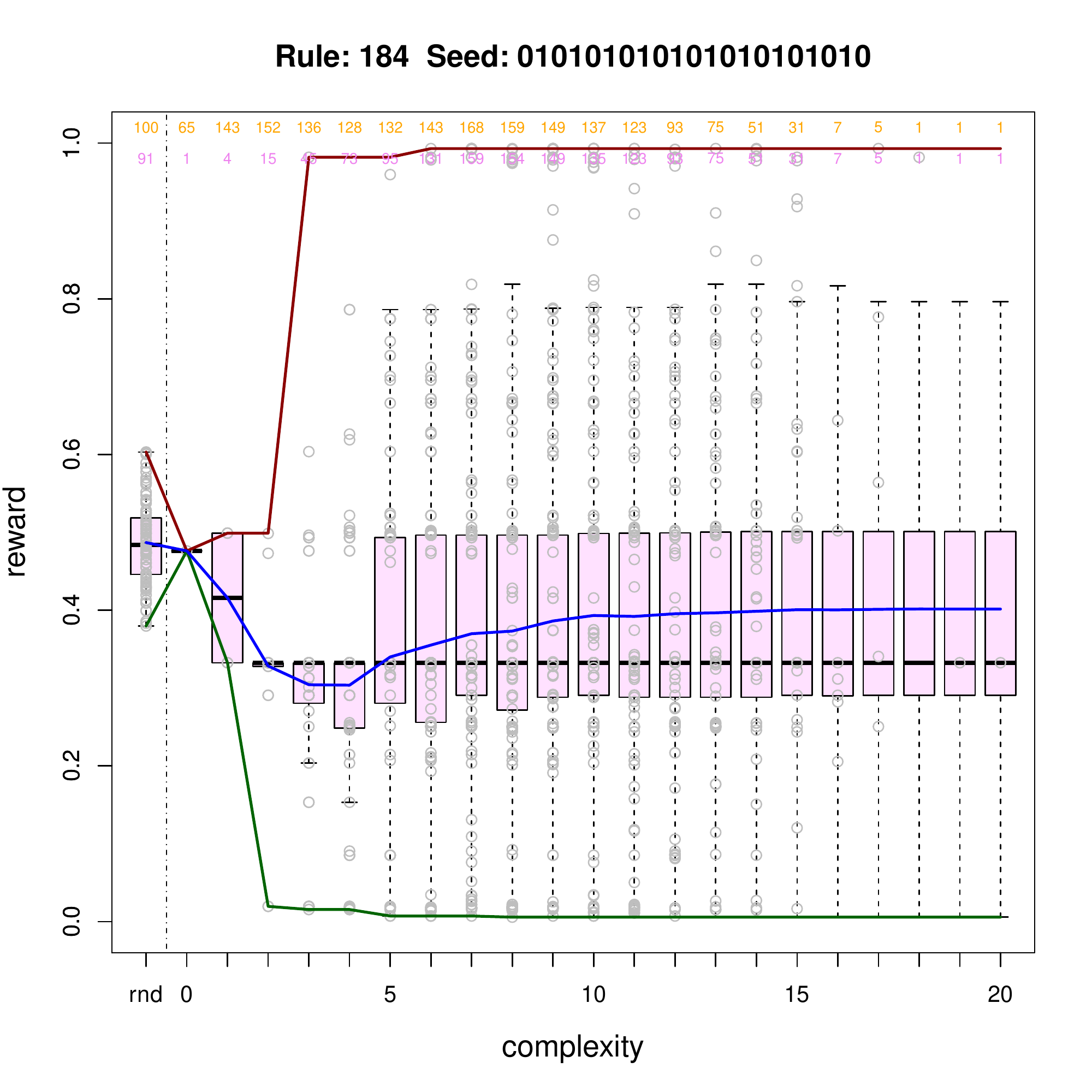}
\includegraphics[width=0.5\textwidth]{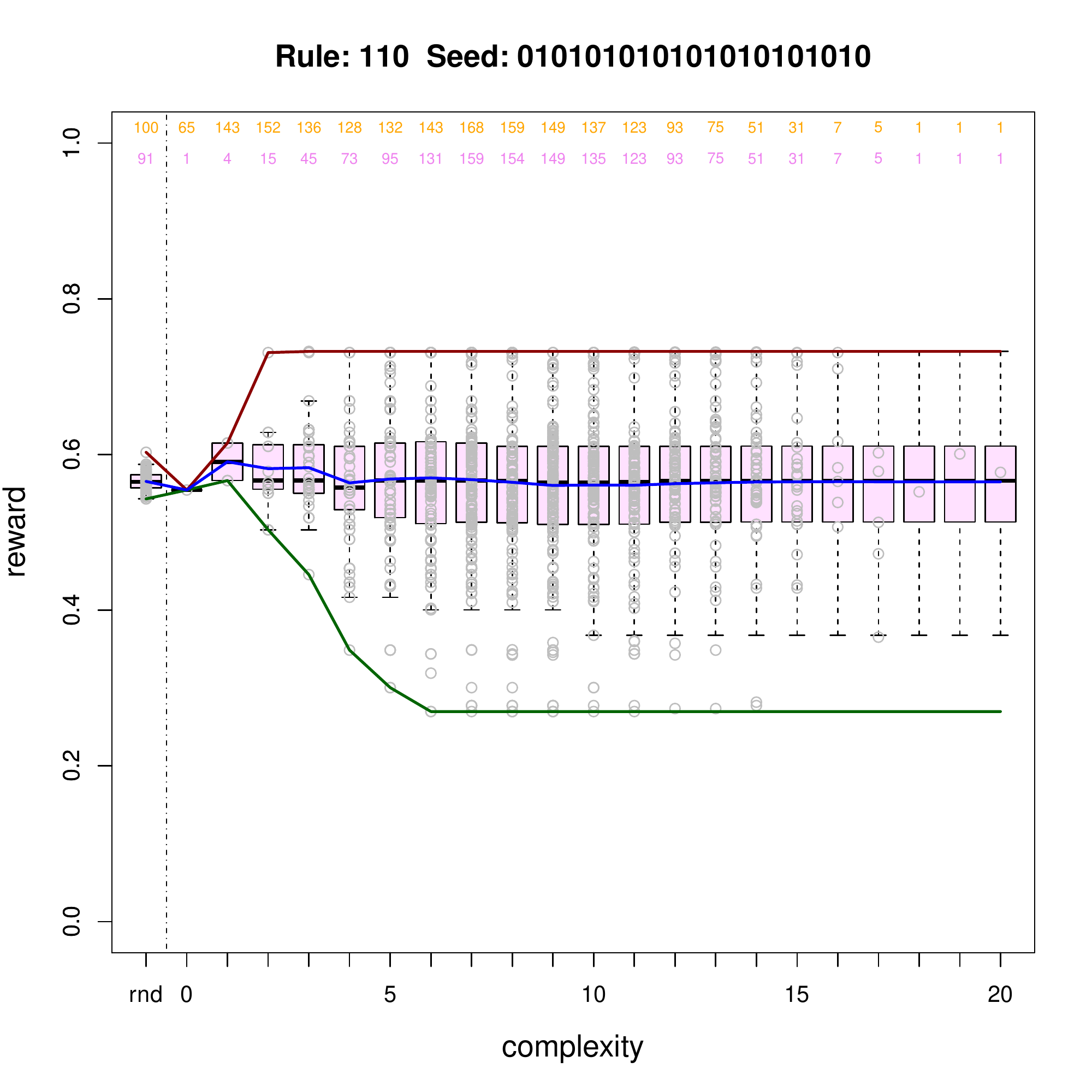} \\
		\hspace{-0.3cm} 
\includegraphics[width=0.5\textwidth]{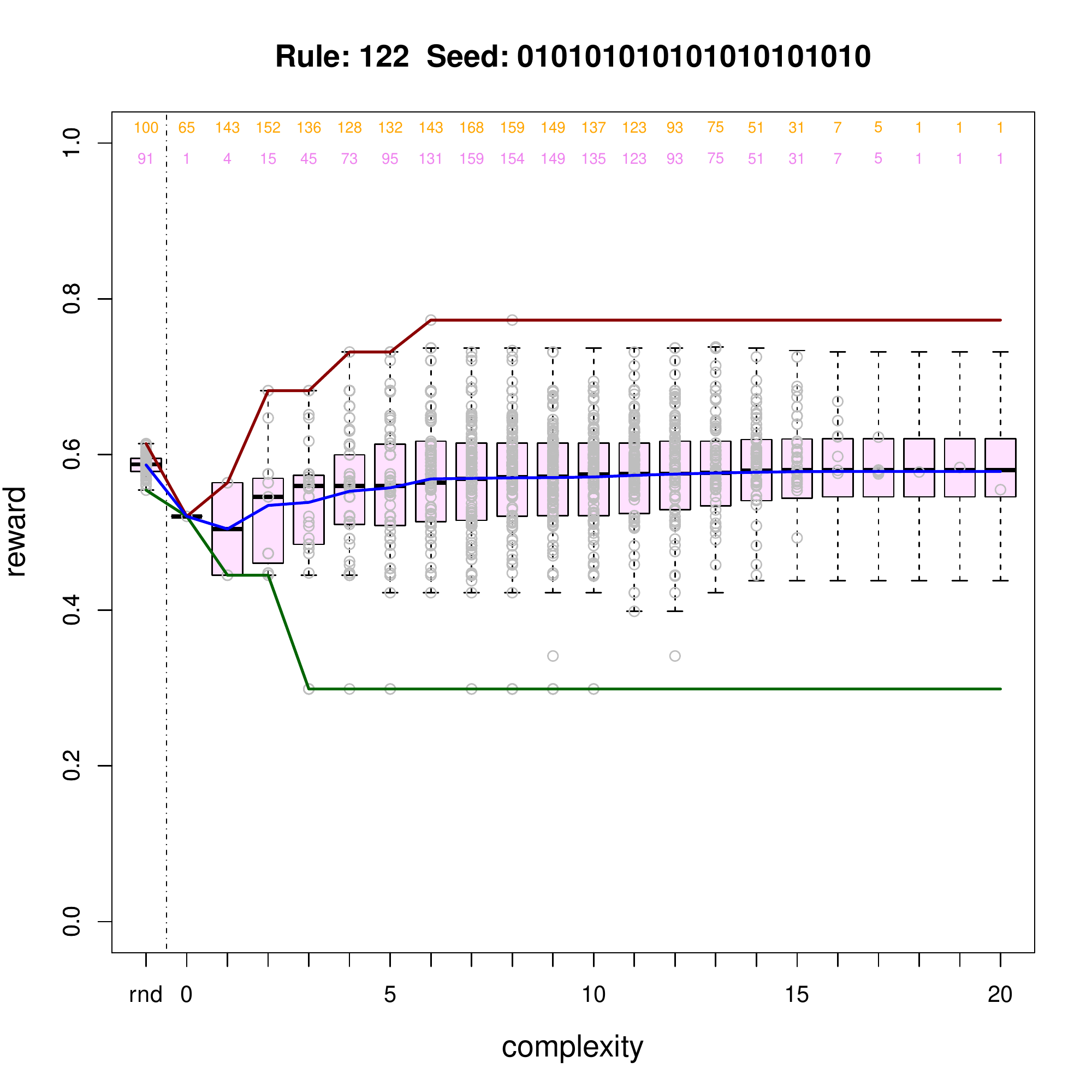}
		\includegraphics[width=0.5\textwidth]{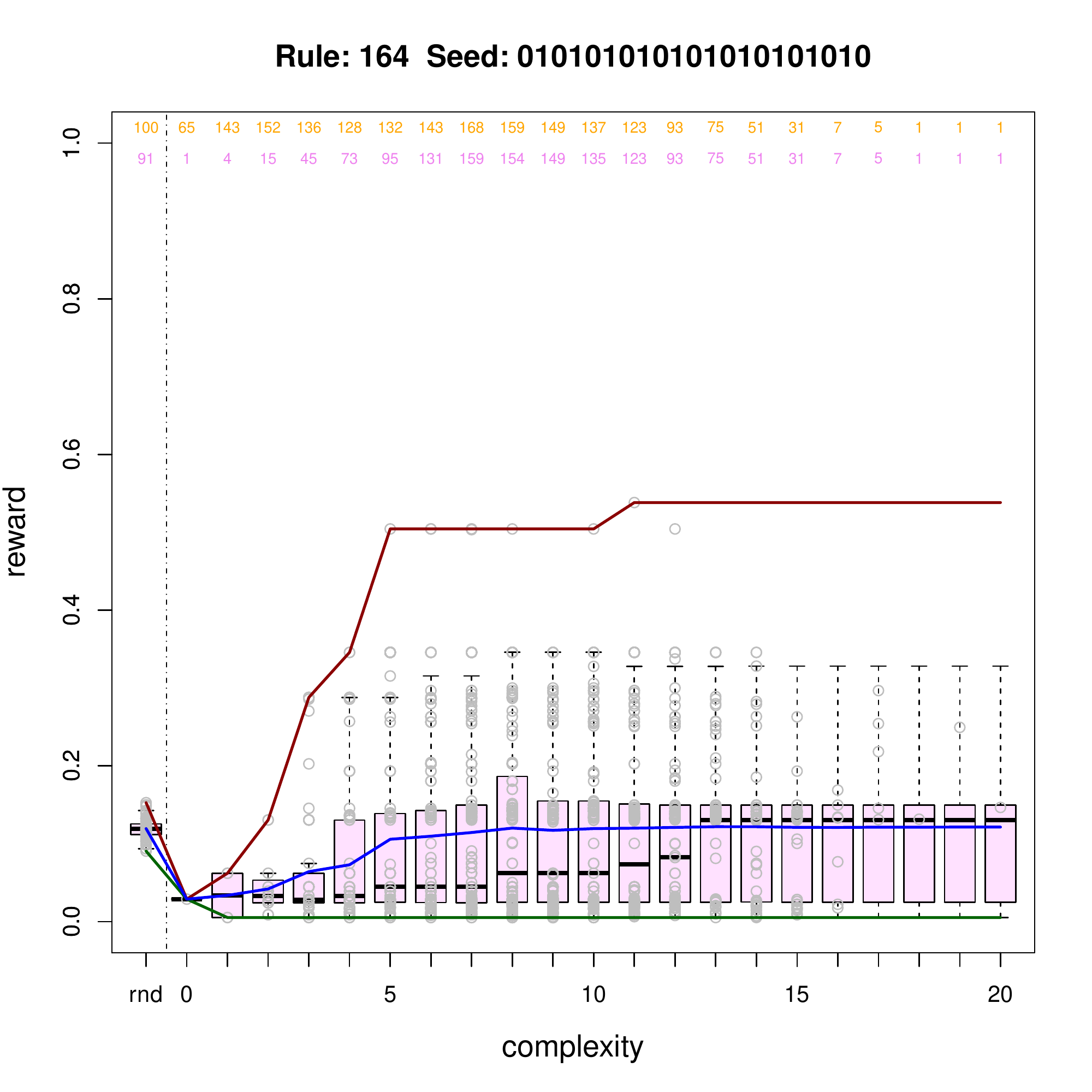}
		\hspace{-0.4cm}
		\vspace{-0.4cm}
	\caption{We show the distributions $R[k]$ for several environments (rules 184, 110, 122, 164) according to the complexity of the policy. We generate 2,000 agents (100 of them are random) and for each of them the environment runs during 300 iterations ($t=300$), starting from scratch for each agent ({\em Fresh} strategy), using the same seed configuration all the time: `010101010101010101010'. Some generated agents are equivalent. This is shown in the top part of the plots: in orange we have the number of agents of that complexity that we generated, in pink just counting the ones that are different. Note that this is just for information, the policies are not actually removed, but considered as many times as they appear (to preserve the original distribution). On the left of the plots (indicated on the \xaxis with ``rnd'') we also show the result for 100 random-walk policies.  For each value of $k$ (\xaxis) we show the aggregated reward results (\yaxis) for each policy (in grey, as small circles) and a summarised representation of the distribution using the accumulated values ($[\leq k]$) for the box (and-whiskers) plots and the envelopes (hence their monotonic appearance). Each box shows the lower and upper quartiles, the band showing the median and the ends of the whiskers showing 1.5 of the interquartile range (if this does not exceed the maximum and minimum). The dark red and dark green lines show the maximum and minimum envelopes and the blue line shows the average. Each environment shows a different figure, which provides insight about the difficulty and discriminating power of each environment. In all cases, note the low dispersion of the random-walk policy.}
	\label{fig:dist}
	\vspace{-0.5cm}
\end{figure}

The values for $\Rmax$ for the four environments (184, 110, 122, 164) are 0.99, 0.73, 0.77 and 0.54 
and the smallest complexities to reach these values ($\Hpolicy$) are 6, 3, 6 and 11. 
The  $\Rmean$ values are 0.40, 0.56, 0.58 and 0.12 respectively, 
and the $\Rmin$ values are 0.01, 0.27, 0.30 and 0.01 respectively.
The (Spearman) correlations $Cor_{\pi \in \Omega}(K(\pi),R(\pi))$ are 0.12, 0.03, 0.17 and 0.09 respectively\footnote{This small bias may be originated by some instructions setting 1s coming before the instructions setting 0s in the {\sffamily APL} instruction set.}, and the (Spearman) correlations 
$Cor_{i=1..\kmax}(i, \Rmax[\leq i])$ are  0.84, 0.70, 0.85 and 0.94 respectively.
If we take a look at normalisation, we see that the average results for random-walk policies are 0.49, 0.57, 0.59 and 0.12, which are very similar to the $\Rmean$ values except for environment 184. This suggests that environments can be normalised to have 0 expected aggregated reward for a random-walk agent (random actions), as we did with the notion of `balanced' environments in \cite[sec. 4.2]{HernandezOralloDowe2010,HernandezOrallo10b}, or we can do this to get 0 for a random policy, as we will do here, leading to different results.

Now we analyse similar experiments with the {\em Chained} strategy in Figure \ref{fig:disthist}. The shapes of the distributions are similar to the previous case, but there are some differences.
For instance, the environment with rule 164 is very sensitive to previous policies, suggesting that many policies lead to a low number of 1s from where it is difficult to recover. In fact, the existence of non-recoverable configurations full of 0s cannot be ruled out. In fact, it is symptomatic that we see a higher average performance with random-walk agents (action-random) than with a set of random policies.

\begin{figure}
	\centering
		\vspace{-1cm}
		\hspace{-0.3cm} 
		\includegraphics[width=0.5\textwidth]{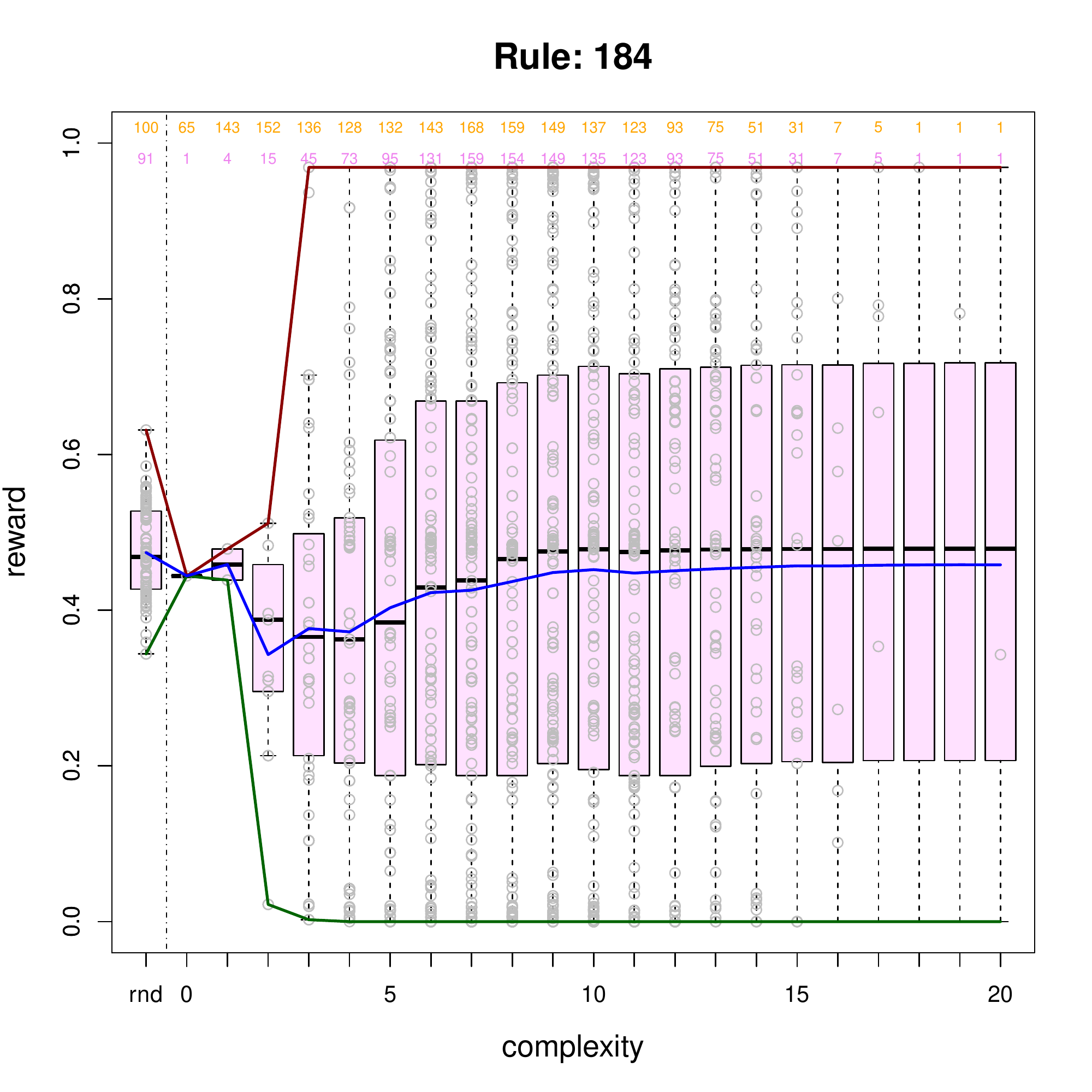}
		\includegraphics[width=0.5\textwidth]{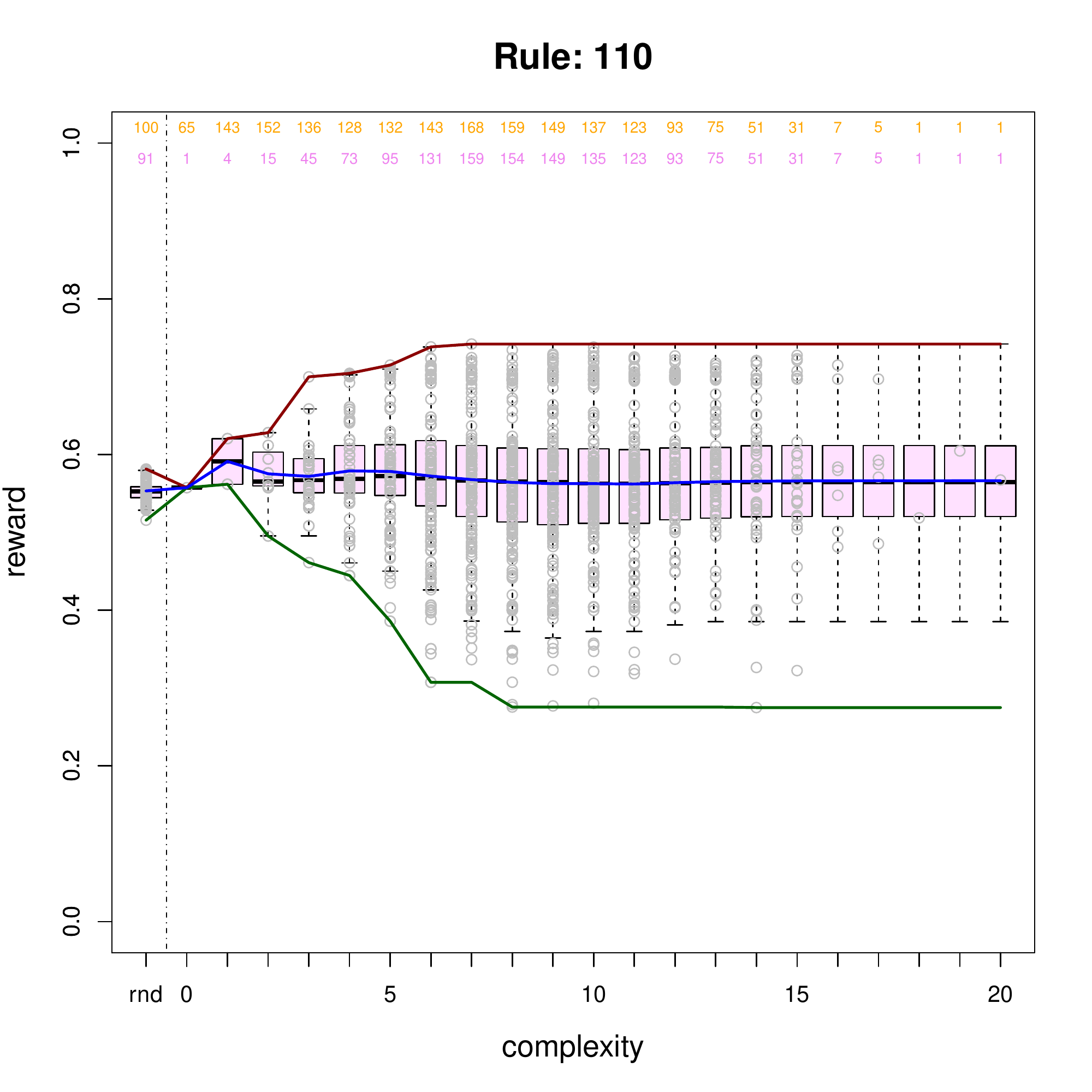} \\
		\hspace{-0.3cm} 
	  \includegraphics[width=0.5\textwidth]{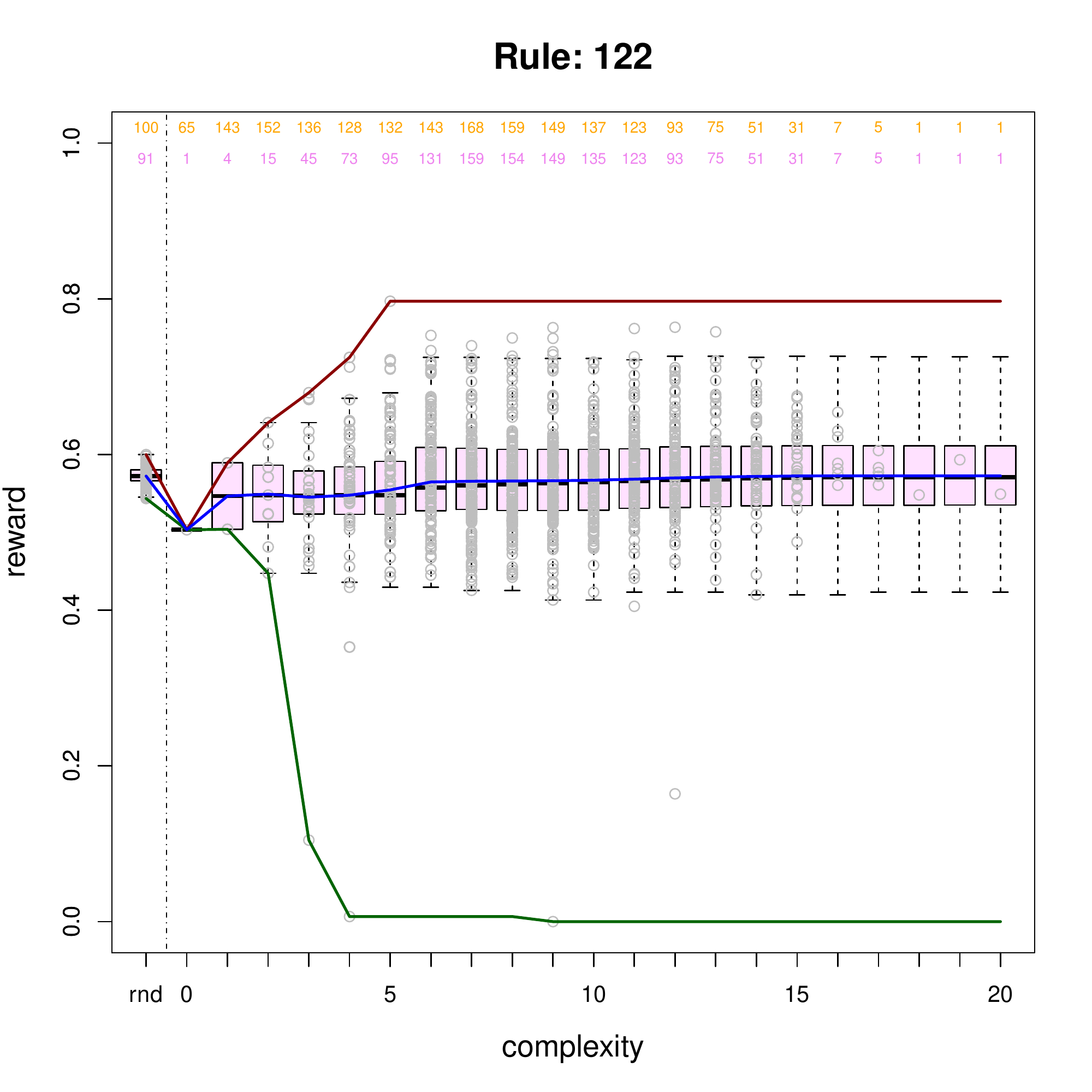}
		\includegraphics[width=0.5\textwidth]{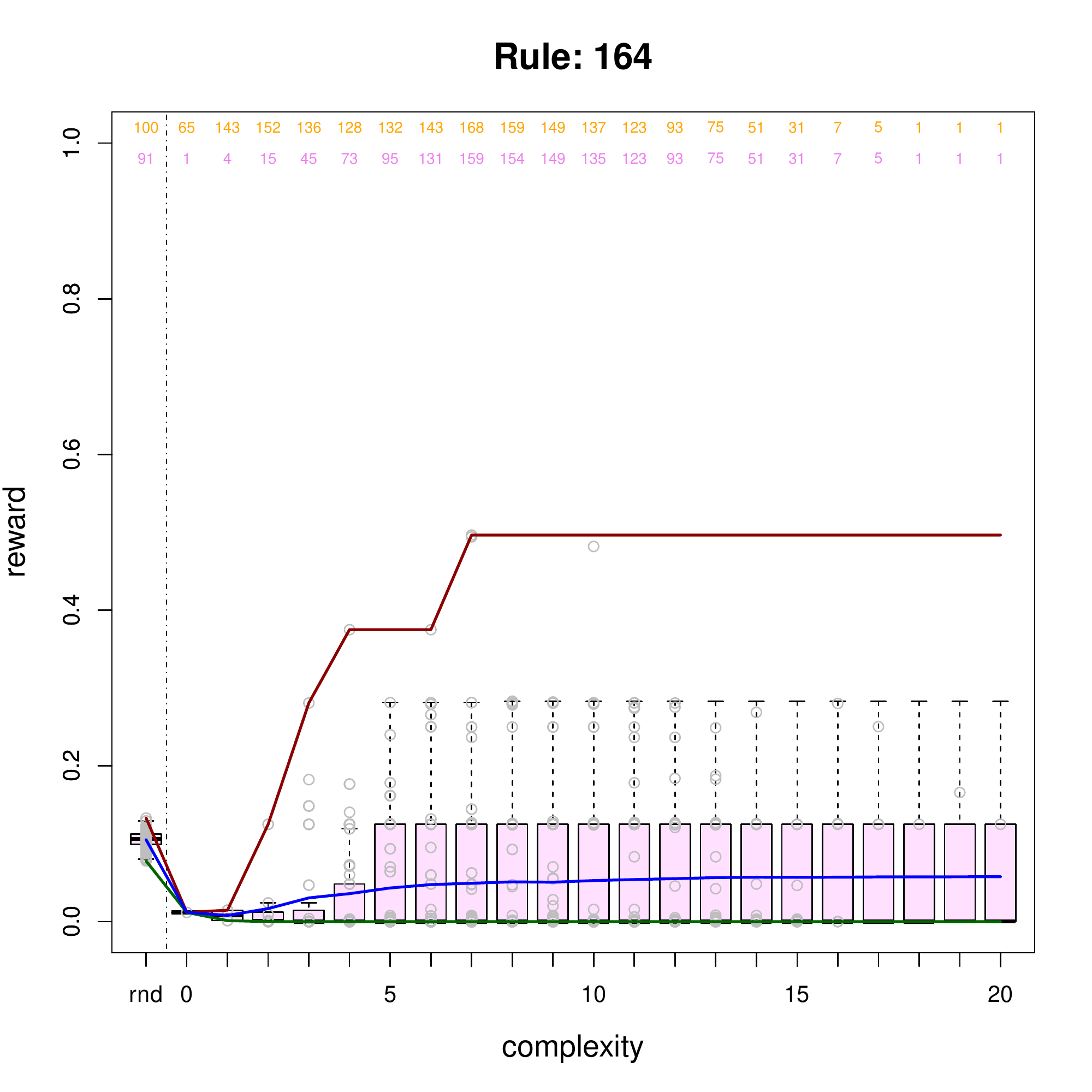}
		\hspace{-0.4cm}
		\vspace{-0.4cm}
	\caption{Same as Figure \ref{fig:dist} but using the {\em Chained} strategy for the experiments.}
	\label{fig:disthist}
	\vspace{-0.5cm}
\end{figure}

We now include some indicators for Figure \ref{fig:disthist} as well, shown in Table \ref{tab:size}. 
While the values of $Rmax$ are similar to the ones for the {\em Fresh} strategy, the $\Hpolicy$ values are very different, confirming that this value is not robust to small changes (in the sample or the strategy).
We see some relevant changes with respect to the {\em Fresh} strategy case in $\Rmin$, which again indicates that some policies may lead to malevolent environment states from where it is difficult to recover.

\begin{table}
	\centering\scriptsize
		\vspace{-0.3cm}
		\begin{tabular}{ c c c c c c c c}
		 Rule & $|\Omega|$ & $\Rmax$  & $\Rmean$ & $\Rmin$ & $\Hpolicy$ & $Cor_{\pi \in \Omega}(K(\pi),R(\pi))$ & $Cor_{i=1..\kmax}(i, \Rmax[\leq i])$ \\ \hline
		 184 & 1900 (1349) & 0.97 & 0.46 & 0 & 3 & 0.08 & 0.69 \\
		 184 & 9500 (5922) & 0.99 & 0.46 & 0 & 5 & 0.06 & 0.80 \\ \hline
		 110 & 1900 (1349) & 0.74 & 0.57 & 0.28 & 7 & 0.01 & 0.88 \\
		 110 & 9500 (5922) & 0.81 & 0.58 & 0.27 & 15 & -0.02 & 0.98 \\ \hline
		 122 & 1900 (1349) & 0.80 & 0.57 & 0 & 5 & 0.16 & 0.81 \\
		 122 & 9500 (5922) & 0.82 & 0.59 & 0 & 5 & 0.11 & 0.80 \\ \hline
		 164 & 1900 (1349) & 0.50 & 0.06 & 0 & 7 & 0.11 & 0.87 \\
		 164 & 9500 (5922) & 0.50 & 0.06 & 0 & 7 & 0.07 & 0.87 \\ \hline
		\end{tabular}
		\vspace{0.3cm}
	\caption{Comparison of indicators using the {\em Chained} strategy for two different sizes of the working sample $\Omega$ (in parentheses the number of policies after removing the repeated policies). Those for $|\Omega|= 1900$ correspond to Figure \ref{fig:disthist}. All correlations are rank (Spearman) correlations.}
	\label{tab:size}
		\vspace{-0.5cm}
\end{table}

We also did some experiments with a larger working sample $\Omega$ of 9,500 policies (10,000 minus the random walks). The results are also shown in Table \ref{tab:size}. As we can see, while there is a strong similarity in results for many indicators, some other indicators vary slightly, and one, $\Hpolicy$, is clearly not robust.

Finally, we show the whole distribution without slicing by complexity in Figure \ref{fig:histograms}. The narrow bars (usually higher and placed about in the middle) of the histograms show the 100 random-walk policies. The wide bars show the remaining 1,900 policies. We see that the histograms for the random-walk policies are peaked in the middle and with a normal-like shape. Apart from their location, random-walk policies do not provide too much information. However, this is not the case for the other policies. In the case of environments with rules 110 and 122 we see a normal-like shape, which suggests that, in their ranges, discriminating power is quite regular, with more policies around the average and very few on the extremes of the distribution. On the contrary, environments with rules 184 and 164 show a very different picture. Rule 184 has a high concentration of policies for aggregated rewards close to 0, 0.5 and 1, and valleys for the rest. This makes it difficult to understand whether the discriminating power is high or low. Also the notion of difficulty is blurred here, because it seems difficult to reach 0 and 1 because of their location, but then there are many policies with these scores. Rule 164 gives a very asymmetric picture, and many policies behave worse than random.

\begin{figure}
	\centering
		\vspace{-1cm}
		\hspace{-0.3cm} 
		\includegraphics[width=0.5\textwidth]{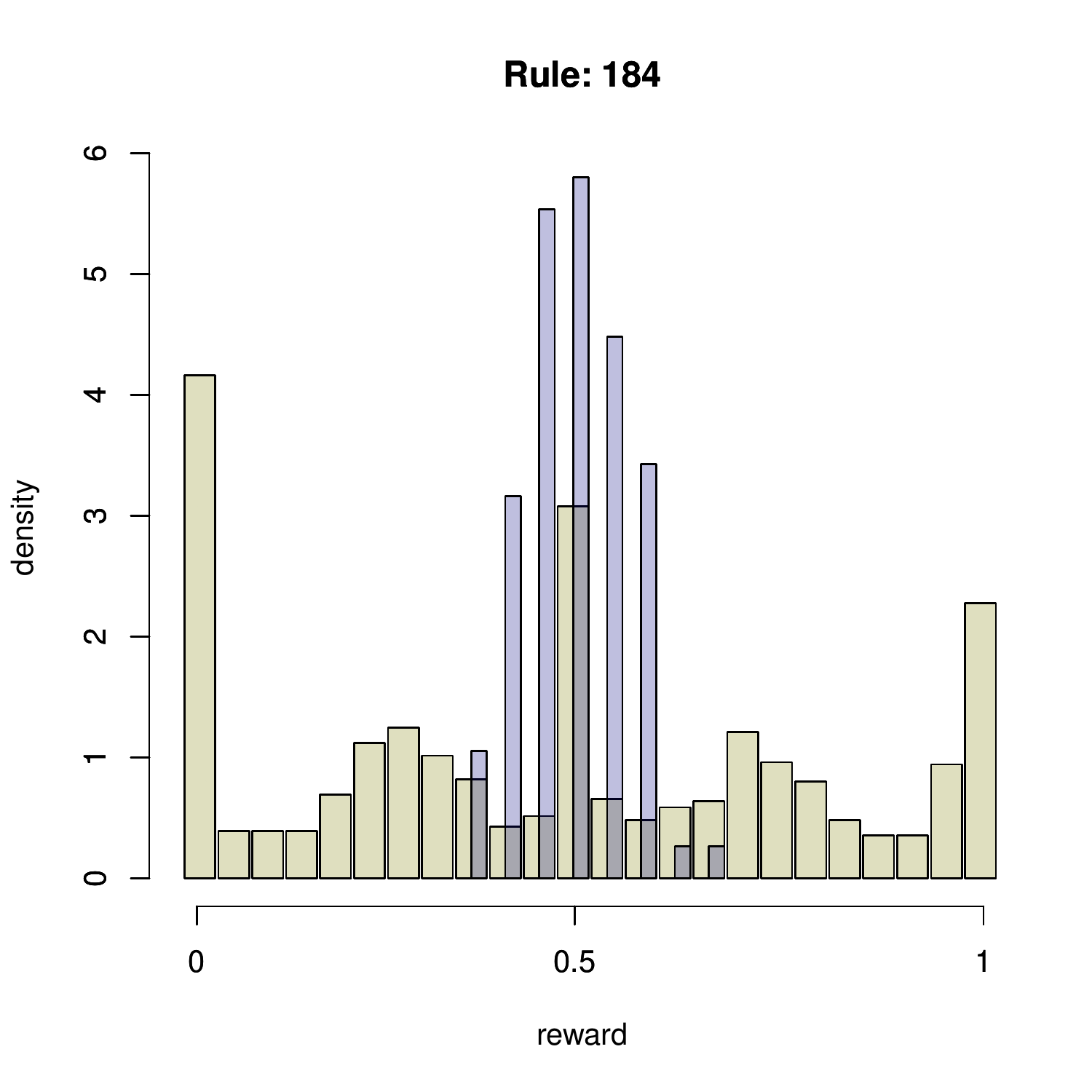}
		\includegraphics[width=0.5\textwidth]{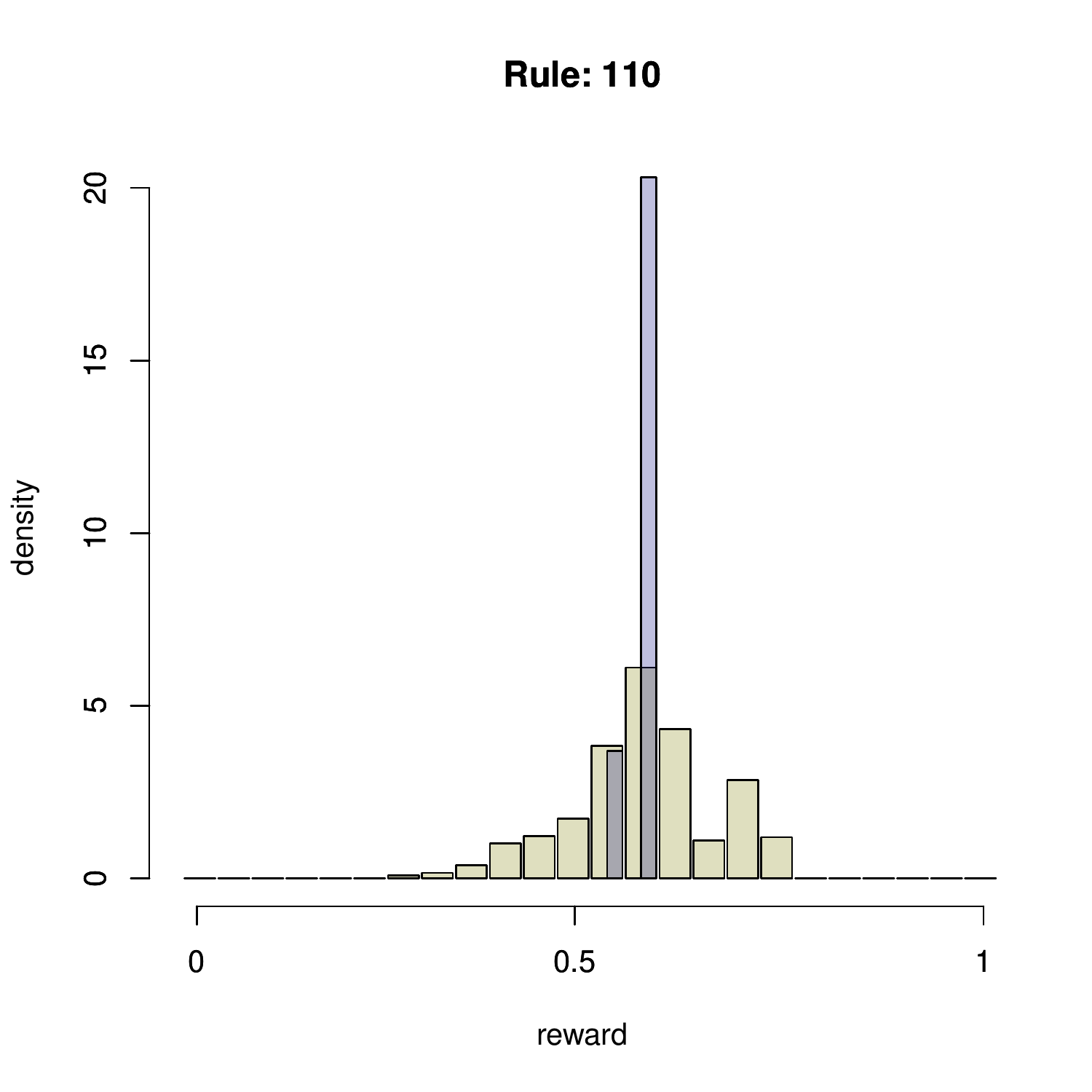} \\
		\hspace{-0.3cm} 
	  \includegraphics[width=0.5\textwidth]{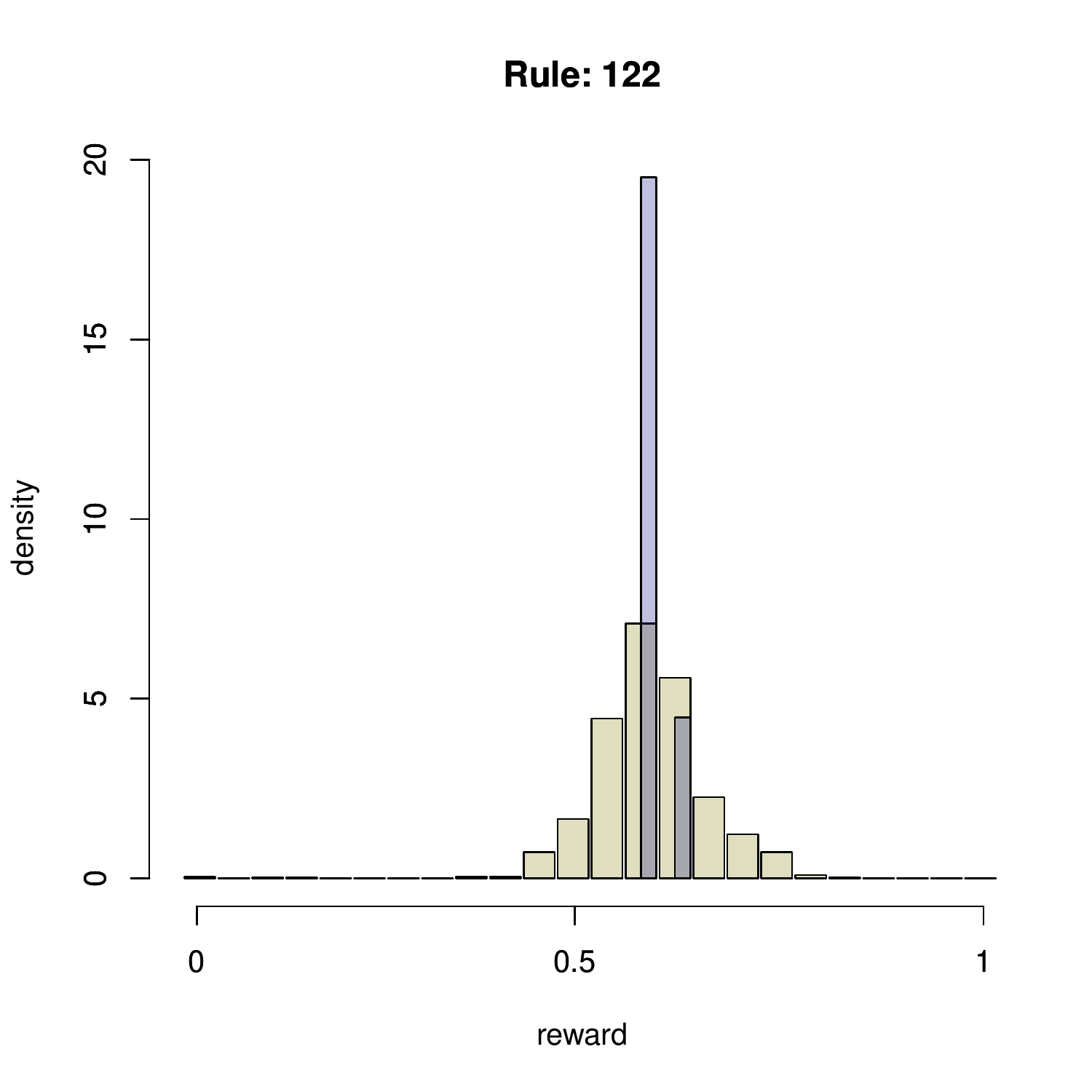}
		\includegraphics[width=0.5\textwidth]{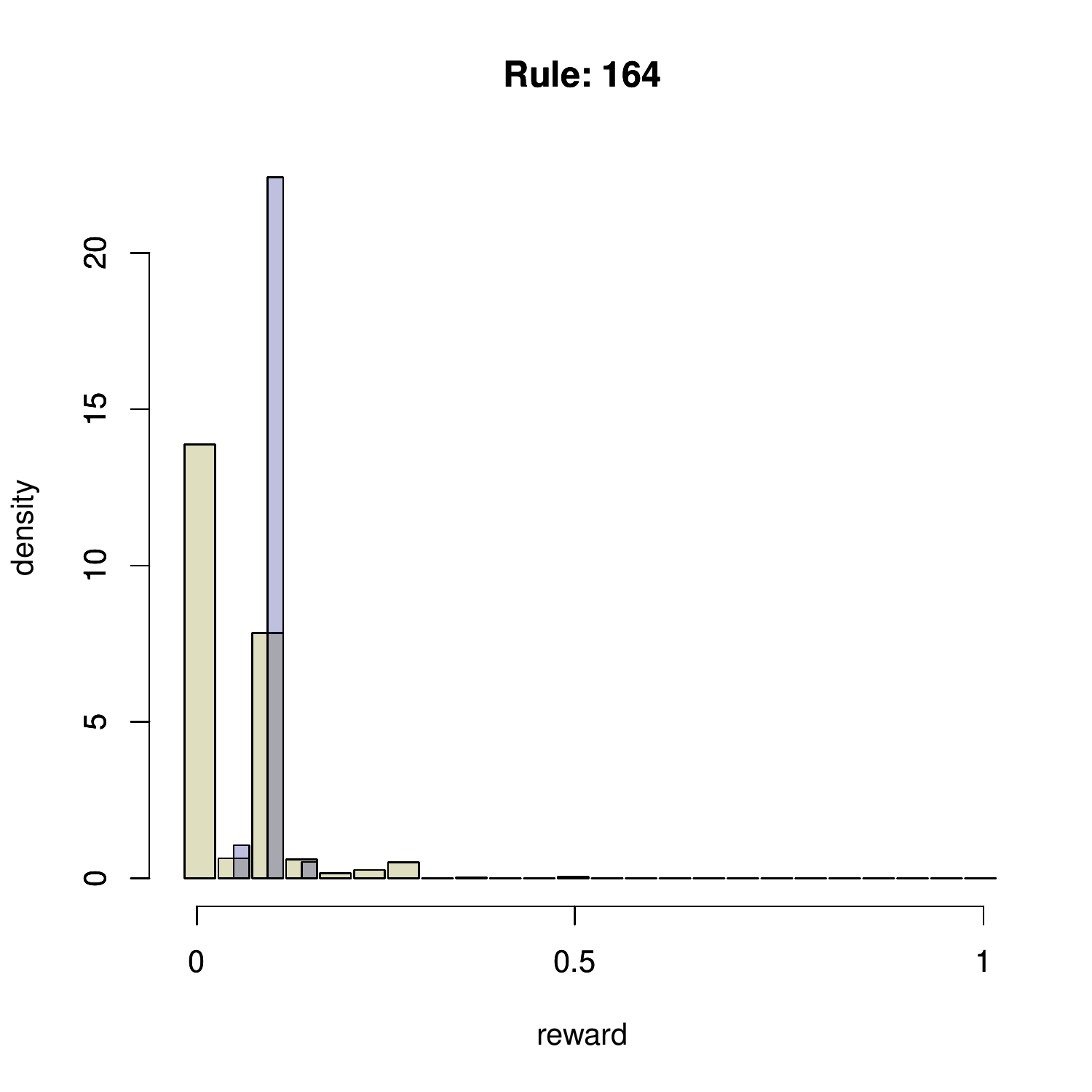}
		\hspace{-0.4cm}
		\vspace{-0.4cm}
	\caption{Histograms of $R$ for the 2,000 policies for the same experiments in Figure \ref{fig:disthist} ({\em Chained} strategy). Narrow bars (purple) represent all the results for the random-walk agents while wide bars (in beige) represent the rest.}
	\label{fig:histograms}
		\vspace{-0.5cm}
\end{figure}

In practice, we need a clearer way of determining the actual difficulty, normalisation and discriminating power of each problem. We introduced the environment response curves for that in section \ref{sec:erc} and we see them next for this setting.

\subsection{Estimation of the environment response curves}

First of all, we give a few indications of how the environment response curves are calculated. We estimate $w(\pi)$ in Eq. \ref{eq:w}, by using $k$ as an approximation of $K(\pi)$ and then we sample over the previously built population of $\Omega$, with 1,900 policies. Since we need to invert the functions $\Dpos$ and $\Dneg$, we calculate 101 different values of $\gamma$ at equal intervals from 0 to 1.  For each value of $\gamma$, we construct samples $S$ (without replacement, i.e., $S \leftarrow \Omega||_{w,N}$) with values of $N = 1..1900$, and we calculate $\max_{\pi \in S} R(\pi)$ for each sample $S$. We calculate whether this value is greater than or equal to $(1-\gamma)(\Rmax - \Rmean) + \Rmean$ as in Eq. \ref{eq:qpos} for each of these samples, leading to an array of Boolean values denoted by $B$, whose size\footnote{In the implementation, we do a trick to make things much faster: we check whether there are more than 10 `trues' in a row. In this case, we consider that the process has stabilised and we stop, filling the rest of values with `true'.} is 1,900. Note that this is not an estimation of $\qpos$ for each $N$, but just one case. Since we want to calculate $\Npos$ and $\Nneg$ (see equations \ref{eq:Npos} and \ref{eq:Nneg}), we need to estimate when $\qpos$ is greater than or equal to $1/2$. We do this by summing how many {\em falses} there are in $B$. 
This is exactly $\Npos$. 
Since our value for $\qpos$ was not a probability but just one case, we repeat the process 400 times in order to get a good estimate of $\Npos$. Finally, by calculating the binary logarithm we have $\Dpos$ as for Eq. \ref{eq:Dpos}. This leads to a table of 101 values of $\gamma$ and $\theta$. During the same procedure, we also calculate the $neg$ versions ($\Nneg$, $\Dneg$, etc.) and get a second table. From these two tables, we can just get the inverted functions $\Repos$, $\Reneg$ and their joint version $\Re$.

\begin{figure}
	\centering
		\vspace{-1cm}
		\hspace{-0.3cm} 
		\includegraphics[width=0.5\textwidth]{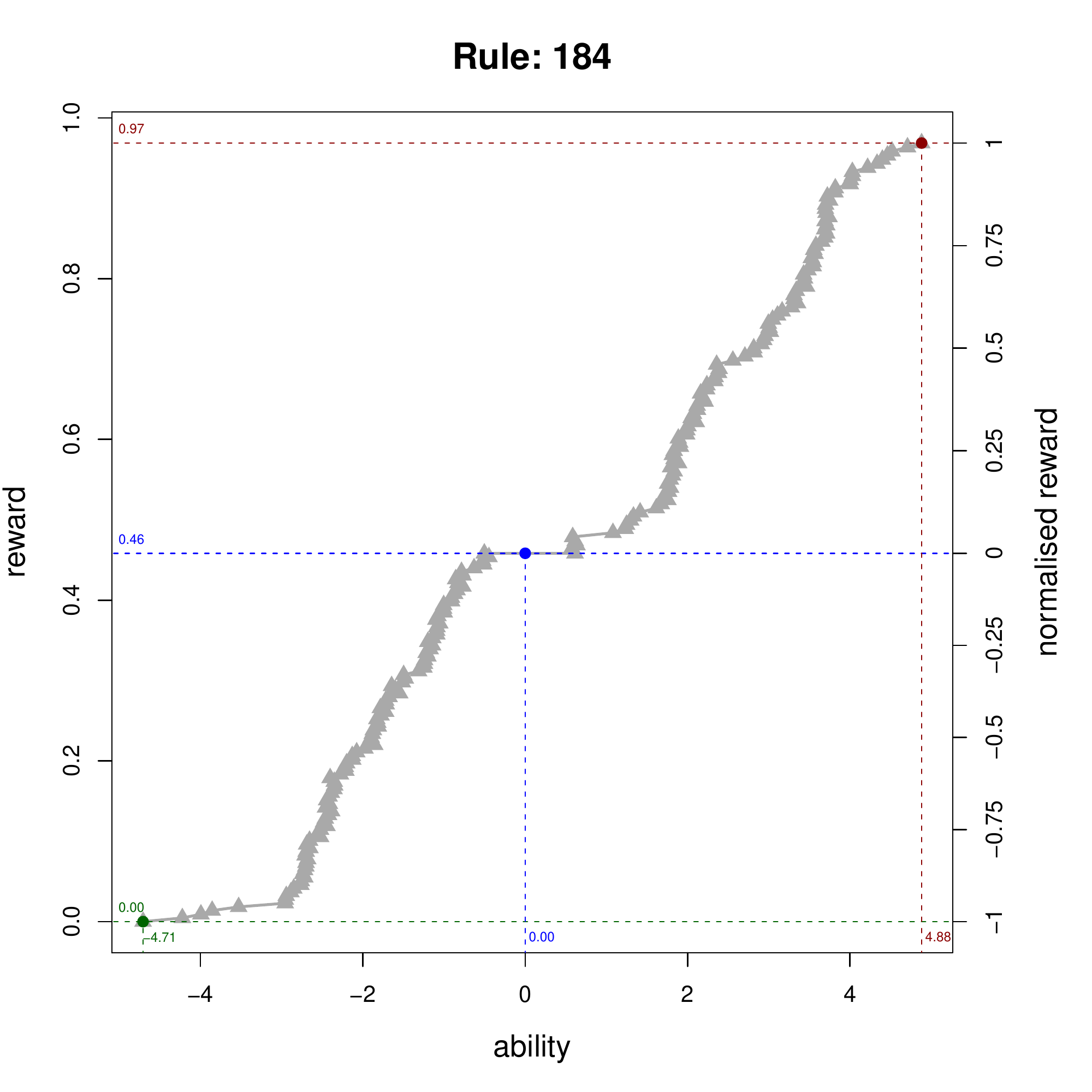}
		\includegraphics[width=0.5\textwidth]{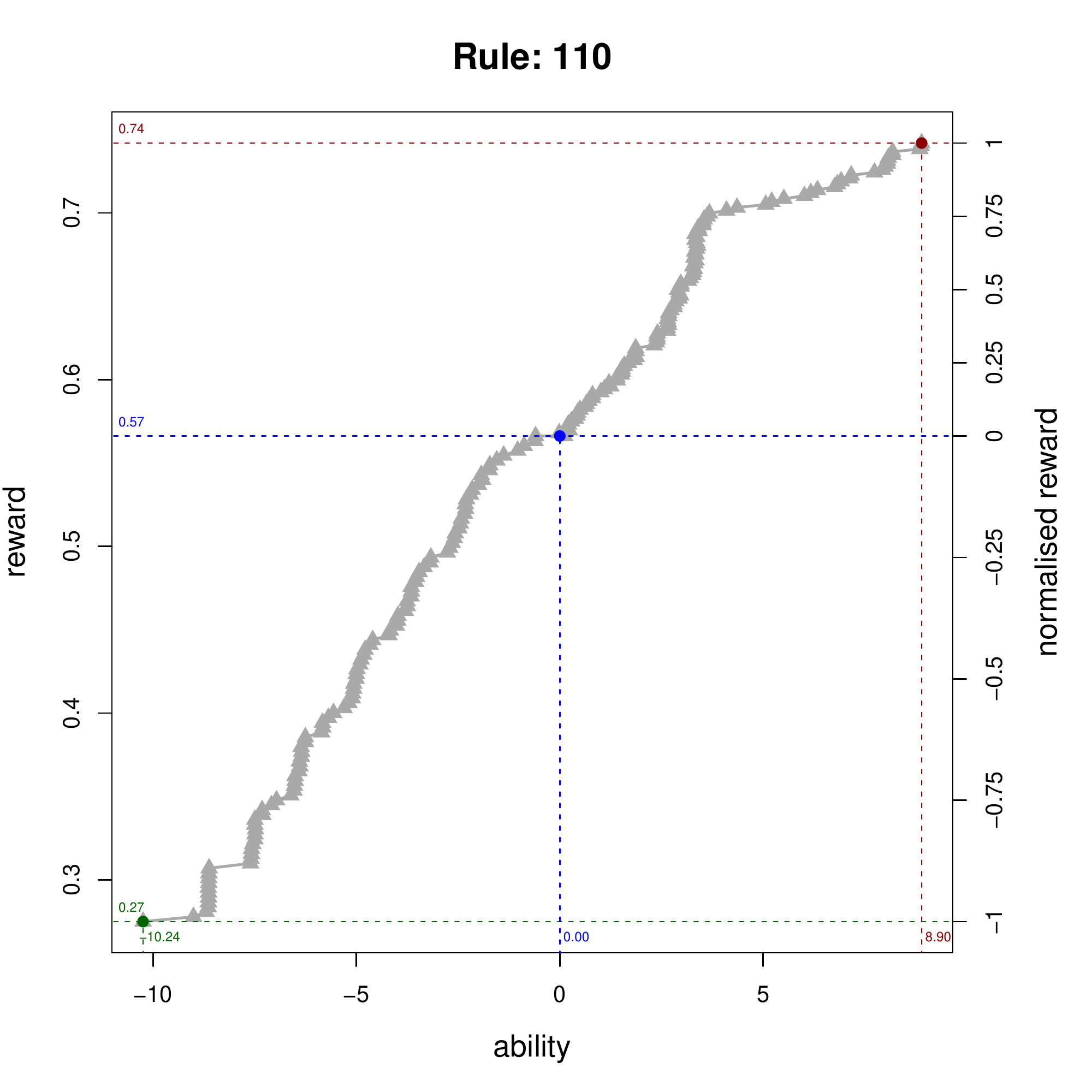} \\
		\hspace{-0.3cm} 
	  \includegraphics[width=0.5\textwidth]{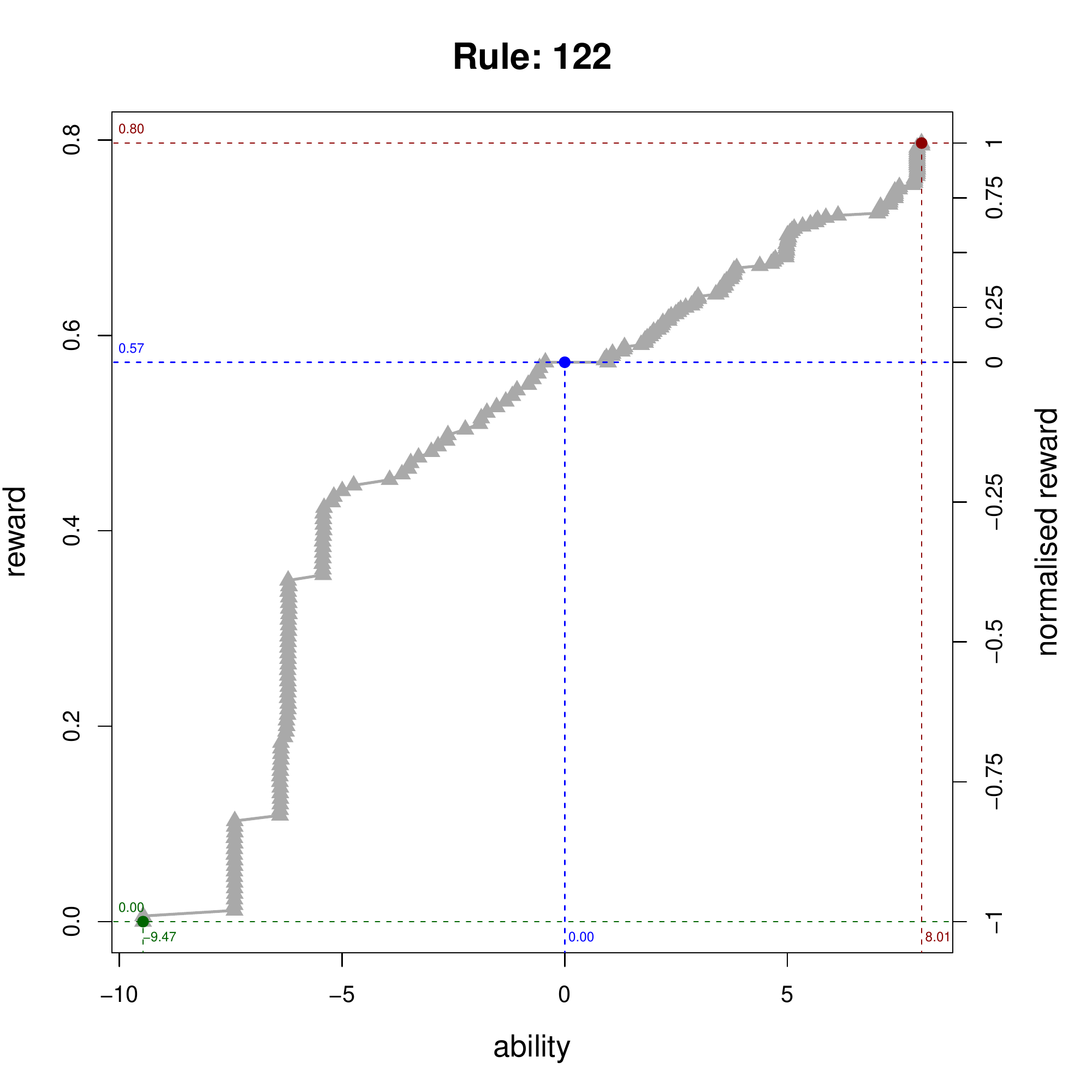}
		\includegraphics[width=0.5\textwidth]{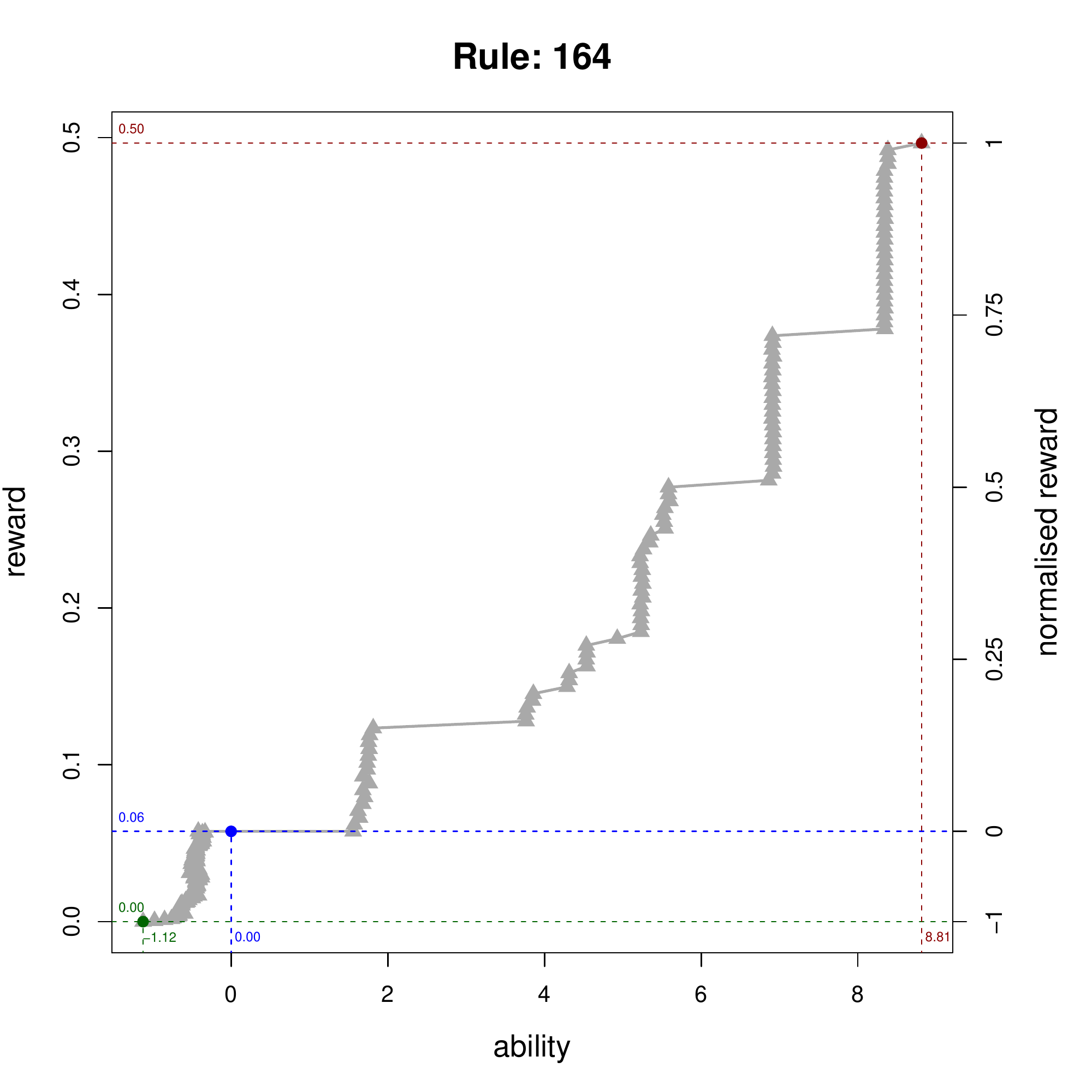}
		\hspace{-0.4cm}
		\vspace{-0.4cm}
	\caption{Environment response curves $\Re(\theta)$ for the four environments using 1,900 policies (2,000 minus the random walks), 9500which are reduced to 1349 after removing repeated ones. The \xaxis represents the ability ($\theta$) of an agent and the \yaxis (on the left) represents the expected aggregated reward $R$. The \yaxis on the right is a normalised version, where we can use the tolerance level $\gamma$ to locate any point. For instance, in order to get 1/2 probability of getting at least the maximum (0\% tolerance level). We need to look at normalised 1, which requires abilities 4.88, 8.90, 8.01 and 8.81 for rules 184, 110, 122 and 164 respectively. If we set the tolerance level at $\gamma=0.25$, we need to look at normalised reward 0.75, which gives values of 3.55, 3.10, 7.38, 8.34 for rules 184, 110, 122 and 164 respectively. More values in Table \ref{tab:sizediff}.}
	\label{fig:real-ercs}
		\vspace{-0.5cm}
\end{figure}

Figure \ref{fig:real-ercs} shows this function $\Re$. The tolerance levels can be easily seen on the \yaxis (right), which is a (linearly) normalised version of the aggregated rewards (\yaxis, left).
Apart from the convenient way of showing how the values are normalised (so all the environments become commensurable if we look at the right \yaxis), we can also get simple and clear indicators of difficulty and discriminating power.
For instance, from these curves we can get estimations of the difficulty of each environment for different levels of tolerance as calculated in the caption of Figure \ref{fig:real-ercs}. Also, the notion of discrimination can be analysed by looking at the shape of the response curves. For instance, rules 184, 122 and 164 have convex curves for the positive part (top right part of the curves) while 110 is concave. This indicates how difficulty varies for several values of tolerance. In fact, difficulty plunges from 8.90 to 3.10 when changing from 0\% to 25\% tolerance for rule 110.
Also, it is interesting to look at these plots in a dual way, because just by changing the rewards $r$ to $1-r$ in our environments, we would get environments with the structure given by the bottom left part of the curves.

Finally, Table \ref{tab:sizediff} shows the differences in the estimation of difficulties using a larger working sample: 9,500 (10,000 without the random walks). We see that these measures are more robust and convergent, especially for tolerance levels 0.05 and 0.10. The environment response curves are shown in Figure \ref{fig:real-ercs10000}, and also have the same shape as those in Figure \ref{fig:real-ercs}.

\begin{table}
	\centering \scriptsize
		\vspace{-0.3cm}
		\begin{tabular}{ || c | c || c | c | c | c | c || c | c | c | c | c || }
		\hline\hline
		 & & \multicolumn{5}{c||}{$\Dpos(\gamma: 0 .. 0.25) $} & \multicolumn{5}{c||}{$\Dneg(\gamma: 0 .. 0.25) $}  \\ \hline 
		 Rule & $|\Omega|\:$ & $\:0.00\:$ & $\:0.01\:$ & $\:0.05\:$ & $\:0.10\:$ & $\:0.25\:$ & $\:0.00\:$ & $\:0.01\:$ & $\:0.05\:$ & $\:0.10\:$ & $\:0.25\:$ \\ \hline\hline
		 184 & 1900 (1349) & 4.88  & 4.72  & 4.33  & 4.01  & 3.55 &      4.73 & 4.20  & 2.94  & 2.76 & 2.53 \\
		 184 & 9500 (5922) & 5.44  & 5.26  & 4.44  & 4.13  & 3.68 &      4.50 & 4.32  & 3.01  & 2.79 & 2.50 \\ \hline
		 110 & 1900 (1349) & 8.90  & 8.88  & 8.07  & 7.74  & 3.10 &     10.24 & 9.01  & 8.64  & 8.63 & 6.97 \\            
		 110 & 9500 (5922) & 12.42 & 12.08 & 10.15 & 10.01 & 8.87 &     11.76 & 11.61 & 10.22 & 9.07 & 8.30 \\ \hline
		 122 & 1900 (1349) & 8.01  & 8.01  & 7.91  & 7.91  & 7.38 &      9.47 & 9.47  & 7.42  & 7.41 & 6.39 \\   
		 122 & 9500 (5922) & 8.73  & 8.66  & 8.49  & 8.47  & 7.55 &     10.24 & 9.16  & 8.17  & 8.17 & 8.02 \\ \hline
		 164 & 1900 (1349) & 8.81  & 8.38  & 8.34  & 8.34  & 8.34 &      1.11 & 0.95  & 0.75  & 0.67 & 0.56 \\
		 164 & 9500 (5922) & 10.54 & 10.54 & 9.41  & 9.41  & 9.39 &      1.49 & 1.18  & 0.93  & 0.92 & 0.69 \\ \hline \hline
		\end{tabular}
		\vspace{0.3cm}
	\caption{Comparison of $\Dpos$ and $\Dneg$ using the {\em Chained} strategy for two different sizes of the working sample $\Omega$ (in parentheses the number of policies after removing the repeated policies). Those for $|\Omega|= 1900$ correspond to Figure \ref{fig:real-ercs} and those for $|\Omega|= 9500$ correspond to Figure \ref{fig:real-ercs10000}.}
	\label{tab:sizediff}
\end{table}


\begin{figure}
	\centering
		\hspace{-0.3cm} 
		\includegraphics[width=0.5\textwidth]{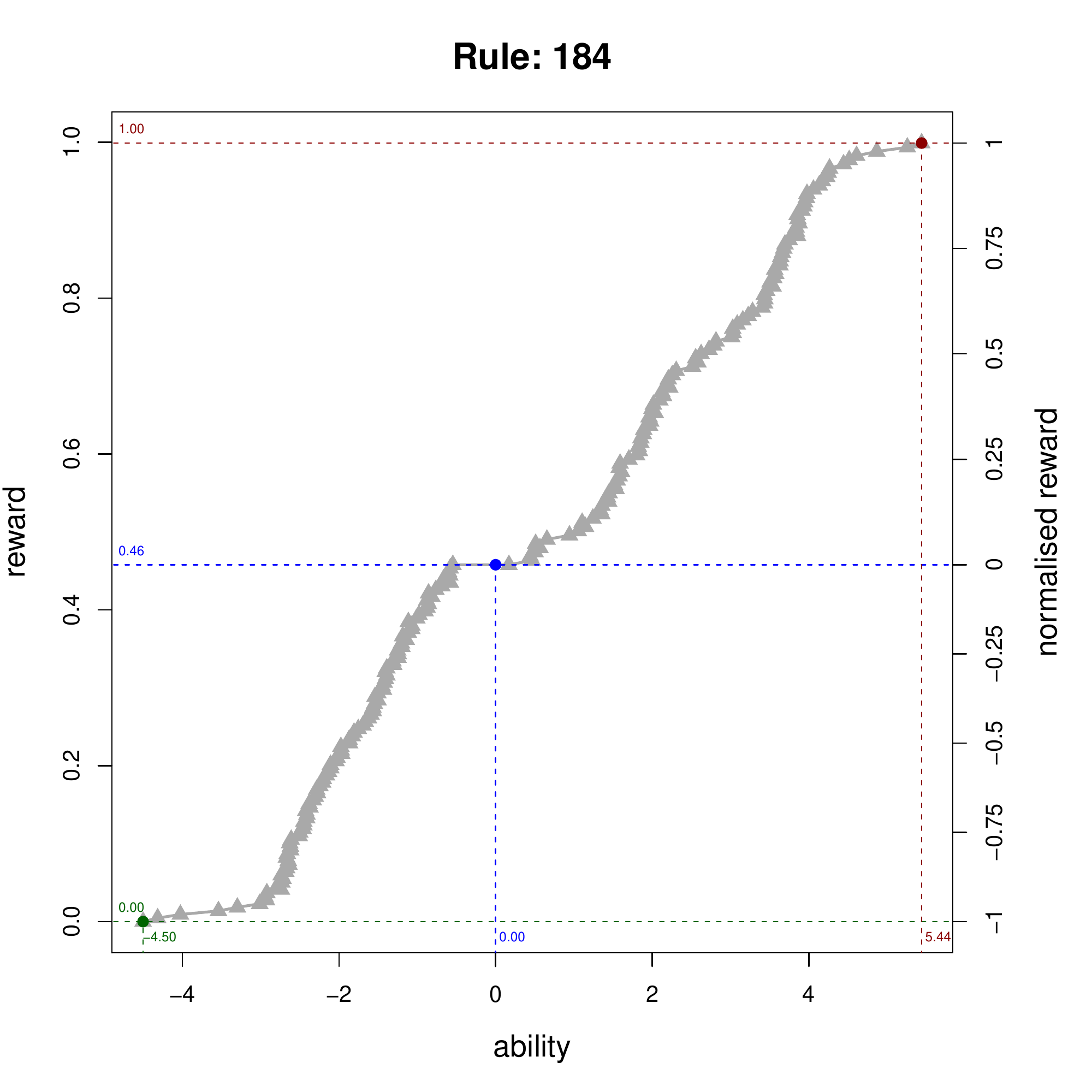}
		\includegraphics[width=0.5\textwidth]{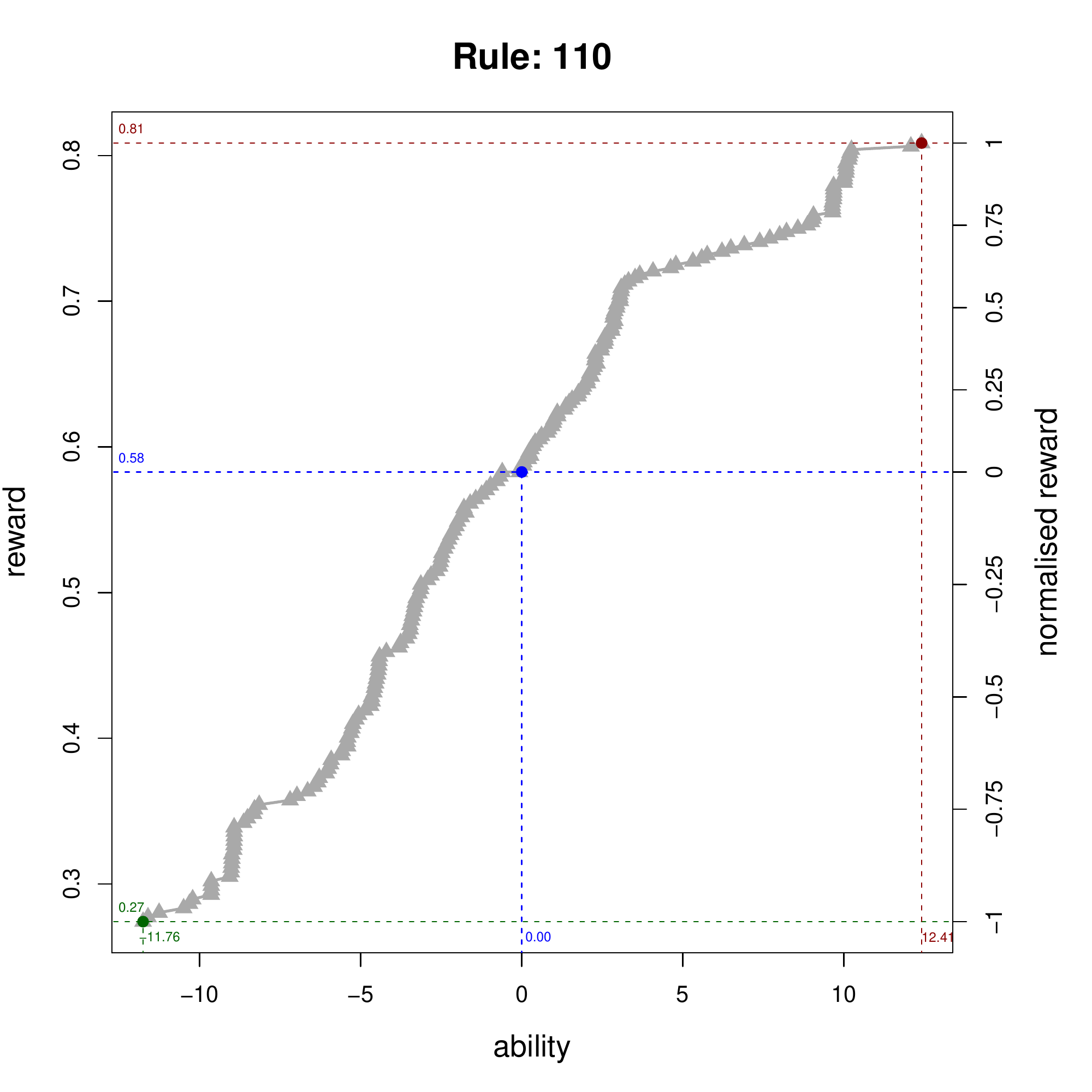} \\
		\hspace{-0.3cm} 
	  \includegraphics[width=0.5\textwidth]{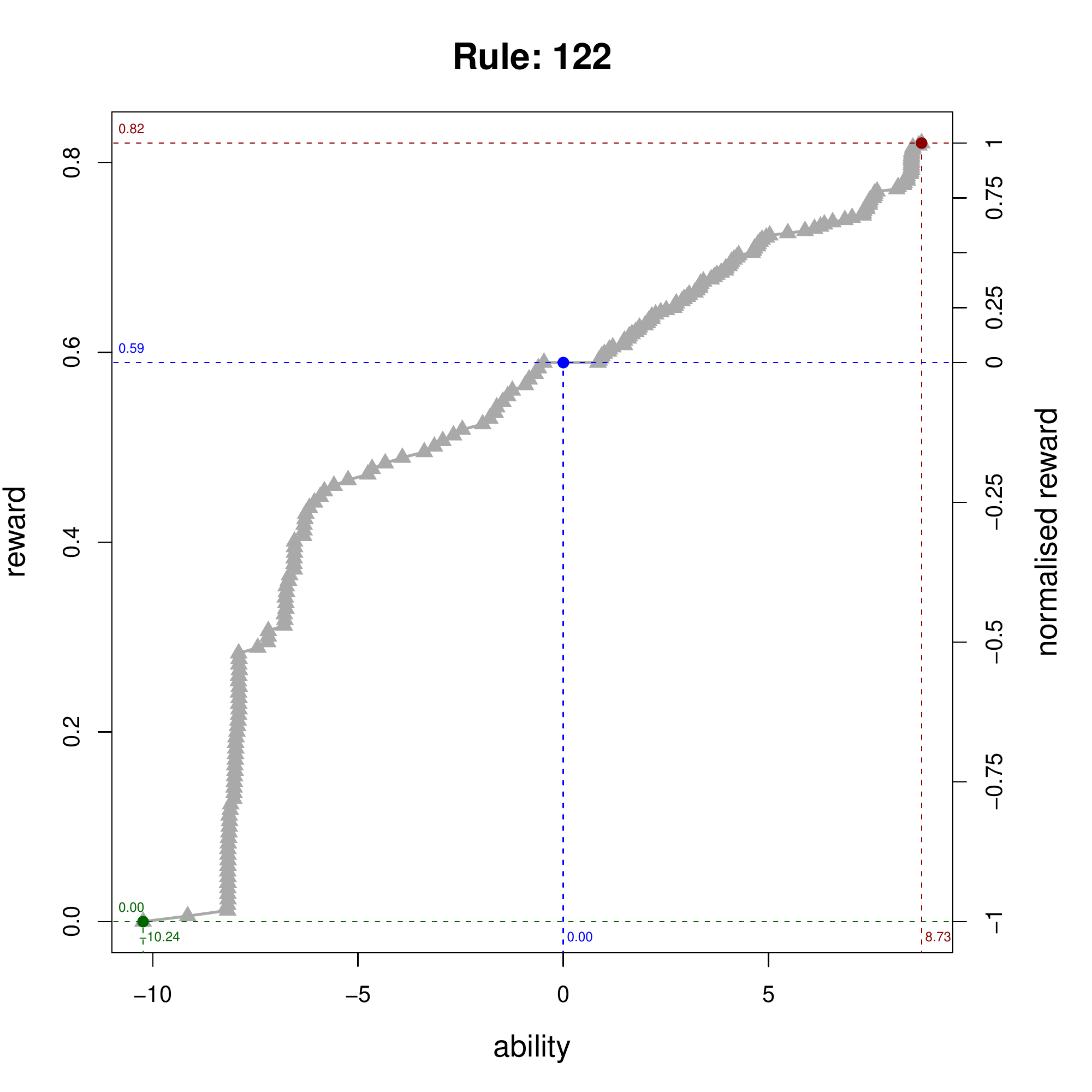}
		\includegraphics[width=0.5\textwidth]{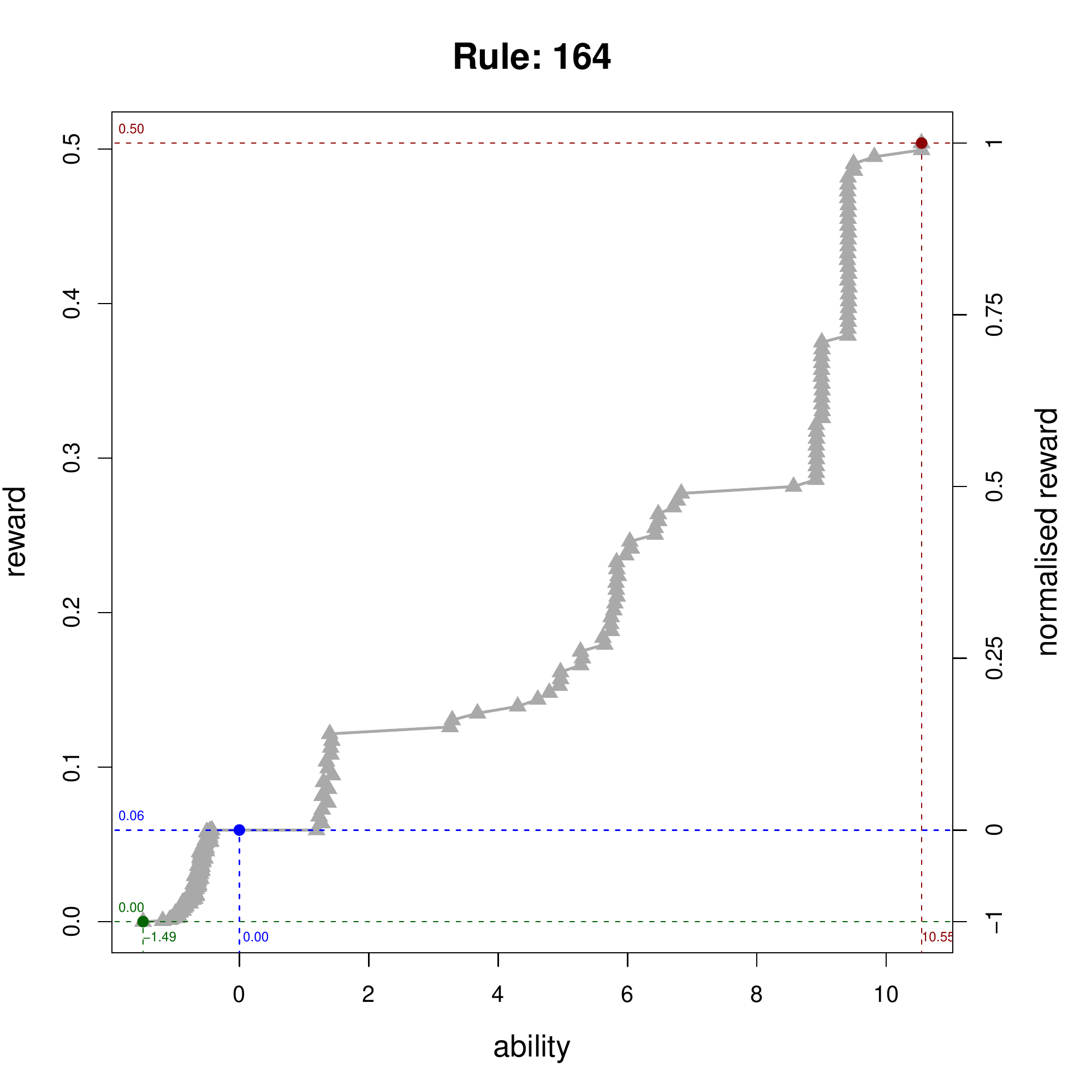}
		\hspace{-0.4cm}
		\vspace{-0.4cm}
	\caption{Same as Figure \ref{fig:real-ercs} using 9500 policies (10000 minus random walks), which are reduced to 5922 after removing repeated ones. The exact values of difficulty for some degrees of tolerance are shown in Table \ref{tab:sizediff}.}
	\label{fig:real-ercs10000}
		\vspace{-0.5cm}
\end{figure}


\section{Discussion}\label{sec:discussion}


In section \ref{sec:background} we mentioned many other approaches to the notions of difficulty and discriminating power. We will now compare our proposal and results with some of these previous notions  
and will put the contributions, limitations and future work in this context.

\subsection{Comparing to other kinds of complexity}

We have used approximations of Kolmogorov complexity to analyse the policies. In section \ref{sec:complexity} we derived some upper bounds by evaluating the complexity of the whole environment. In the setting seen in section \ref{sec:experiments}, this would have implied the compression of the definition of the configuration rules, the reward rules and the observation function. Approximating this as given by Eq. \ref{eq:dotK} would lead to a large number. For instance, the implementation of the environments and automata takes about hundreds lines of code in the language R. More compressed implementations are possible, but it is difficult to conceive an implementation whose length is less than a few thousand bits. So, proposition \ref{prop:upperbound1} is not useful as an approximation of difficulty in this case.

A way-out might have been to consider some underlying setting or class of environments, from which we just add some parameters or rules. For instance, we could assume that the reference machine contains the basics for evolutionary cellular automata as well as the way agents and rewards are handled in this setting. With this assumption, we would only need to calculate how complex each rule is. The problem here, and one of the reasons we have chosen elementary cellular automata, is that there are only 256 rules. Consequently, it is very unlikely that any analysis of the complexity of the description of the environment like this would lead to differences in complexity greater than $\log_2 256 = 8$ bits.

Another possibility might have been to look at the patterns each ECA rule generates and see whether this correlates with the difficulty we can find on them. Fortunately, a deep and insightful analysis of the space-time diagrams have recently been performed by several studies \cite{Zenil2010,zenil2012two}. 
Note that if we analyse and try to compress their space-time diagrams (e.g., those depicted in Figure \ref{fig:manyrules}) we are looking at some of their emergence properties, not at the automaton itself. In other words, we only need to code part of it (of its emergent properties), and not the whole of it. As a result, the complexity may be much lower than the whole automaton (Eq. \ref{eq:dotK}). 
If we just focus on the four rules we have used throughout the experimental section: 184, 110, 122 and 164, we have that \cite{Zenil2010}, which uses the compression method to approximate Kolmogorov complexity, sorts them (from lowest to highest complexity) as follows: 164, 184, 122, 110. A related study, \cite{zenil2012two}, which uses the coding theorem method through a `block matrix decomposition', reaches slightly different values but exactly the same order. If we compare this order with the results of difficulty (with 0 tolerance) as shown in Figure \ref{fig:real-ercs}, we have that the order is 184, 122, 164 and 110, and for (0.25 tolerance) it is 110, 184, 122, 164. In other words, the Kolmogorov complexity of how the ECA evolves in space-time does not seem to correlate with the difficulty that the agent may find in it. This is of course expected, since there are many other things (such as rewards or the influence of the agent in the ECA) that may lead to simple policies occasionally succeeding even if the configuration of the automaton becomes too messy (complex).

The take-away message of all this is something that we already pointed out in the introduction: the emergence of complexity is a very different thing to the emergence of difficulty. While some approaches to difficulty can work (and have worked) by looking at the environment complexity, any general account of difficulty must take an agent-centred point of view.

As discussed in previous sections, there are several options for this. One of them is to look at action sequences, as in Eq. \ref{eq:Haction}. Figure \ref{fig:actions} shows the aggregated reward results for 2,000 agents arranged by the complexity of their sequences of actions $K(\alpha)$ (each sequence is composed of 300 pairs of move and upshot, $\left\langle V, U \right\rangle$). As we see, the complexity of the actions does not correlate with the maximum envelope (not shown in the figure). Moreover, we see that the correlation ($Cor_{i=1..\kmax}(i, \Rmax[\leq i])$) is clearly negative. As more complex the actions are, the variance is lower. This is partially caused by the number of policies leading to high-complexity action sequences is relatively smaller than the number of policies leading to low-complexity action sequences.

\begin{figure}
	\centering
		\vspace{-1cm}
		\hspace{-0.3cm} 
		\includegraphics[width=0.5\textwidth]{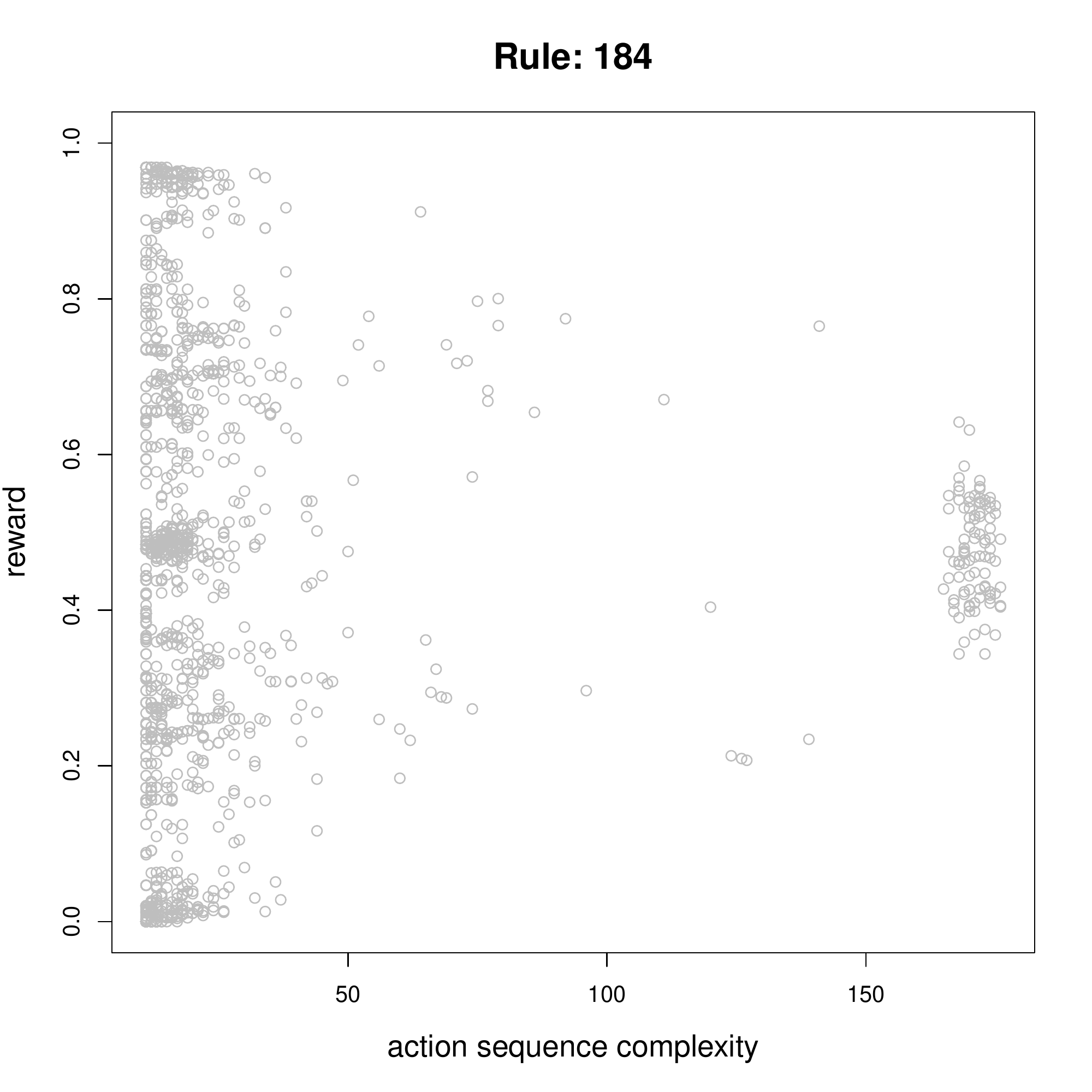}
		\includegraphics[width=0.5\textwidth]{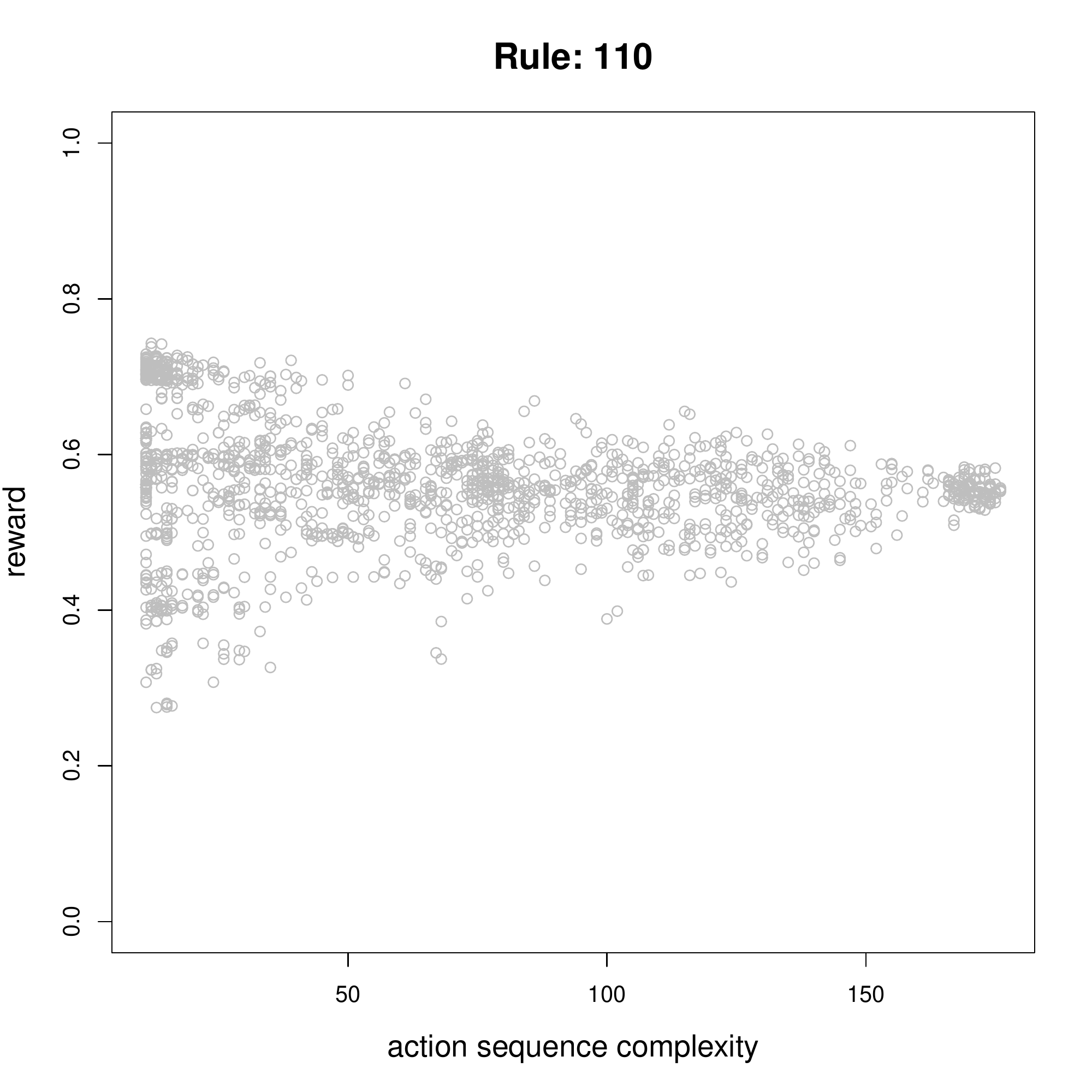} \\
		\hspace{-0.3cm} 
	  \includegraphics[width=0.5\textwidth]{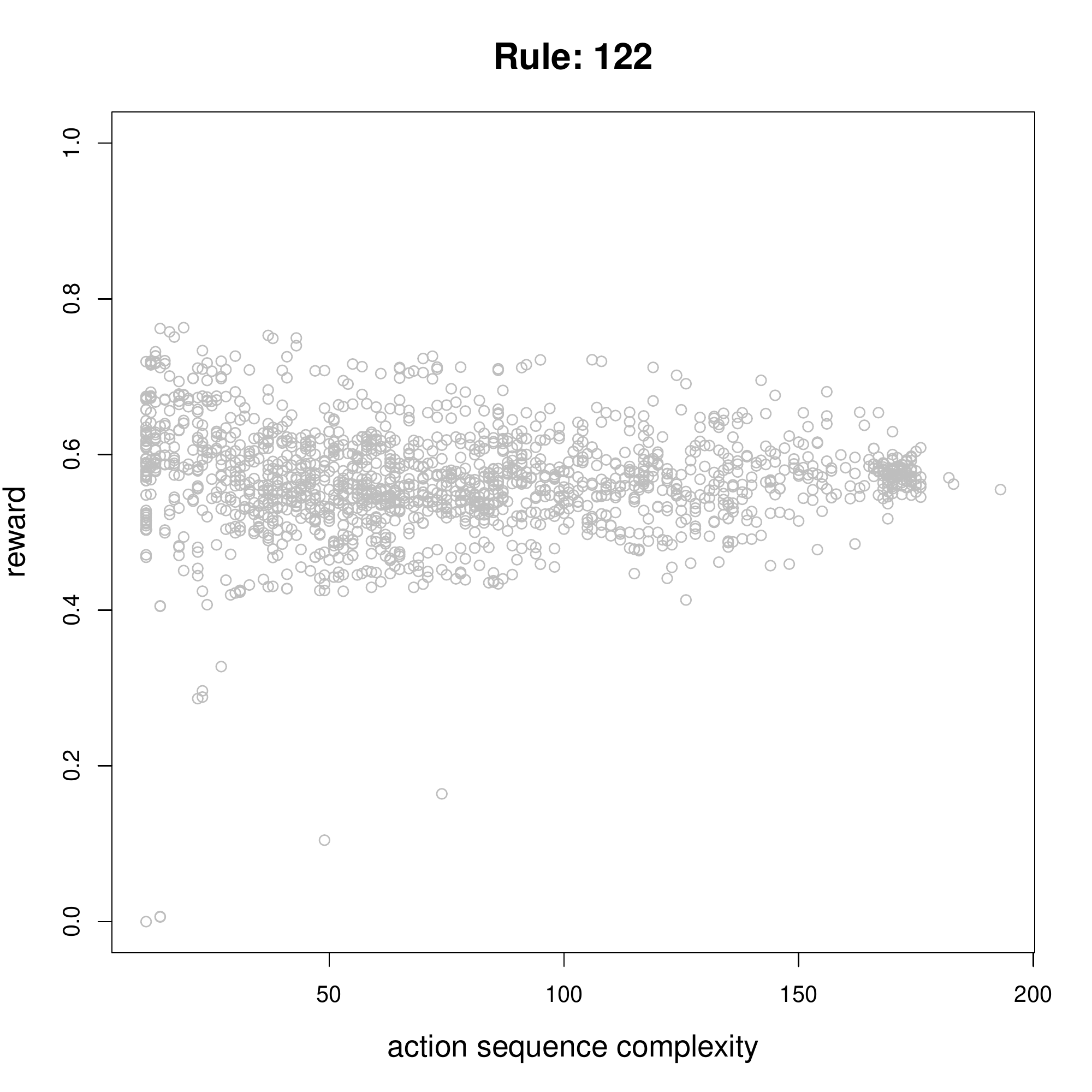}
		\includegraphics[width=0.5\textwidth]{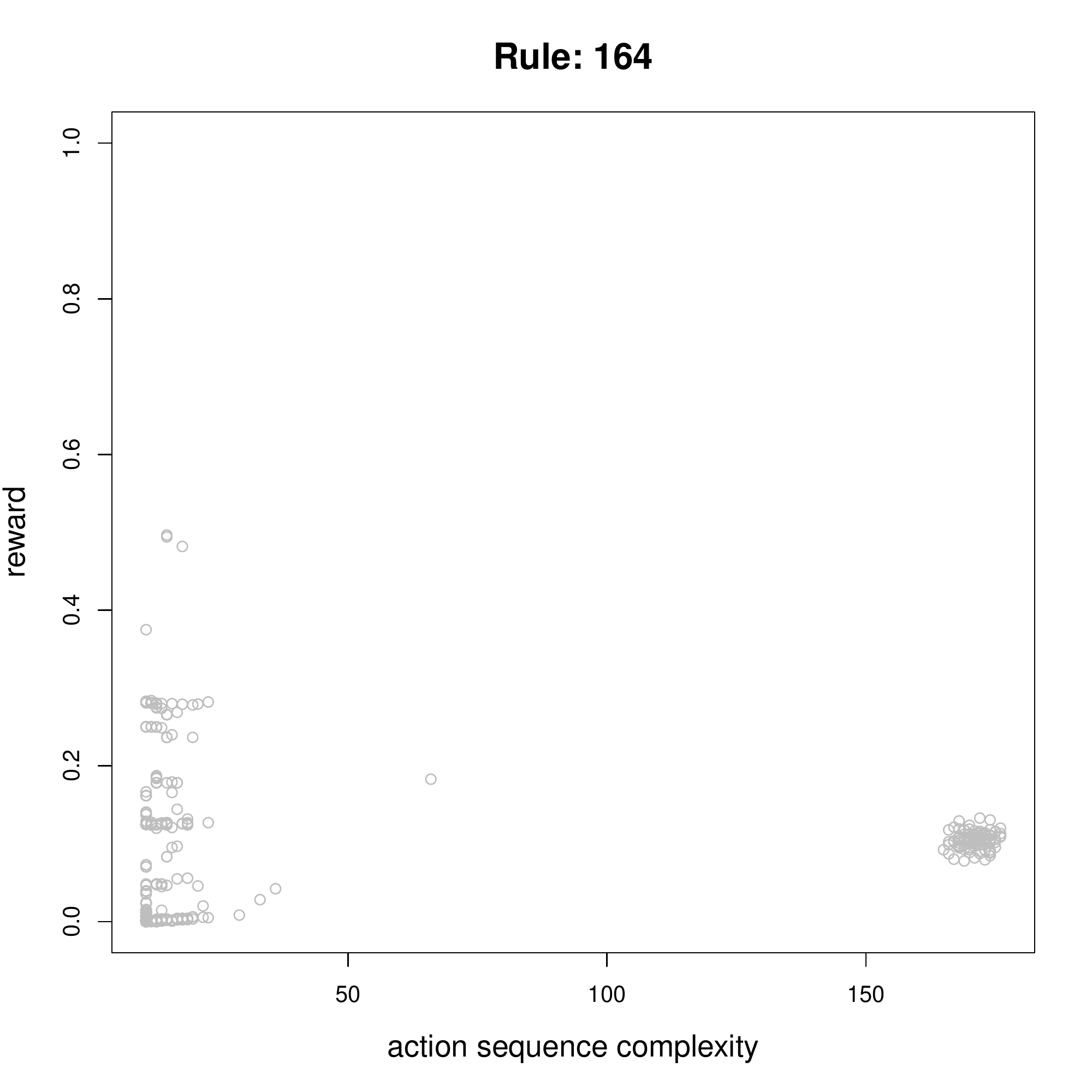}
		\hspace{-0.4cm}
		\vspace{-0.4cm}
	\caption{Same configuration as Figure \ref{fig:disthist} (2000 agents, 100 of them are random) but we plot agent performance on the \yaxis against {\em action sequence} complexity on the \xaxis. 
	The cases on the right of the plots typically show the random agents.}
	\label{fig:actions}
		\vspace{-0.5cm}
\end{figure}

Although Figure \ref{fig:actions} cannot be used as a good source to identify the difficulty or discriminating power of an environment, it confirms some of the findings shown in other plots. For instance, the distribution for rules 110 and 122 is quite regular and matches the most regular histograms in Figure \ref{fig:histograms}. Clearly, rules 164 and 184 lead to less complex space-time diagrams (they are the simplest according to \cite{Zenil2010,zenil2012two} and intuitively in Figure \ref{fig:manyrules}), so simple policies in these cases are not expected to create random-like (i.e., complex)  action sequences.

This does not mean at all that Kolmogorov complexity and other tools from algorithmic information theory should not be applied for determining the difficulty of an environment. The take-away message is that for understanding difficulty we need to apply them on the policies, as we have done in this paper.
In fact, the good thing about Kolmogorov complexity is that the results should be more independent of the language used to represent policies (our agent policy language {\sf APL} here).

\subsection{Taking performance-guided heuristics into account}

One of the features of the approach introduced in this paper is that it is mostly independent of what heuristic a learning agent might consider to evaluate the policies, i.e., to explore the policy space.
In fact, even sticking to one of the two last strategies ({\em Chained} and {\em Levin}) made in section \ref{sec:strategies}, we do a {\em sample} of policies. Any reasonable (and feasible) sample will typically build short programs first instead of long ones, so this will be closely related to Levin's optimal search \cite{Levin73} any way, which is known to asymptotically dominate any other search. This assumes a finite episode length $t$ and that the a priori probability $w$ in definition \ref{def:polprob} as set in Eq. \ref{eq:w} is not modified.  

However, it is of course highly debatable that having the opportunity of attempting several policies and seeing the ongoing reward for each of them, the exploration of the solution space still behaves blind on this, trying one policy after the other without modifying the a priori probability $w$. Also, we need to consider that some policies can be discarded by our mind models, as we also discussed under the term `mind practising'.
It looks reasonable that if we try a policy and gets high reward, we typically modify that policy instead of attempting another one completely from scratch. This is in the end what most heuristics do, guided by gradient ascent or some other kind of maximisation process.

Let us consider an example of this. Imagine a binary policy representation, i.e., strings over $\Sigma = \{0, 1\}$ and consider that we can just modify policies by adding or removing a bit at any position (Levenshtein distance $d$). Also consider that we explore each policy for a long time, so we can get a fairly good approximation of its performance. From here we can construct a graph, by connecting all the strings $s$ and $t$ such that $d(s,t) = 1$.
Imagine that we have a way to determine whether some strings correspond to the same policy, allowing us to merge them into equivalence classes, with $d(s,t)=0$.
This leads to a graph, where for each node in the graph we have an estimated reward $R$ and each edge requires the modification of one bit.
These graphs can be explored by some kind of gradient-descent or any other heuristic. 
%
%
%
%
Without setting all of the details about how the heuristics work it is hard to figure out a criterion to determine a measure of difficulty from a graph like this: the length of the shortest non-decreasing path (if it exists), the best policy after a breadth-first, with or without backtracking, ...
There are infinitely many heuristics taking the information into account. Some of them 
are hybrids, by trying to maximise results but also keeping solutions simple (e.g., MML-guided heuristics \cite{Wallace2005}).
In fact, this binary representation is just one of the infinitely many possibilities, and the resulting graph would highly depend on it and, as a result, our indicators (and possible plots) of difficulty. 
%
We could even define a distribution of heuristics and calculate difficulty by an appropriate weighting of all of them. However, this would be infeasible in practice and could possibly boil down to some of the approaches, such as the one used in the paper, that try to ignore heuristics or assume a very general one such as Levin's search.

Nonetheless, we think that the specialisation of some of the plots, indicators and curves for families of heuristics could be a useful area of research, especially for the analysis of agent algorithms and other learning techniques. This would also connect this work to all the previous analysis on difficulty made in the area of evolutionary computation \cite{he2007note}. In this sense, some recent approaches based on the notion of motif, such as `motif difficulty' \cite{liu2012motif} would be worth being investigated in relation to this work.


\subsection{Applications, limitations and future work}

The main contribution of this paper is the view of difficulty in terms of how a population of policies (i.e., behaviours) perform on a given task. While this is a traditional view in psychometrics, especially in Item Response Theory, here we propose to construct the policies, i.e., the population, according to some distribution over a policy language, and evaluate their results in terms of the (Kolmogorov) complexity of the policy.
By developing these ideas we have reached several findings. The alternative of using the (Kolmogorov) complexity of the environment leads to very loose upper bounds, which are useless in practice.
We have also seen that analysing a problem by how random actions perform (the random-walk policy) is not able to unveil the structure of the problem. Using policies instead leads to a much better understanding. The distribution of policies in terms of complexity, either by the use of graphical tools or statistical indicators is informative. However, it is sometimes hard to understand and does not yield clear notions of difficulty and discriminating power, or how several environments can be normalised so their results become commensurable. The development of the environment response curves gives a much cleaner picture of the notions of difficulty and discriminating power and make the analysis directly operative for constructing tests and benchmarks, as Item Response Theory does.

While the notion of difficulty derived here has intuitive interpretations, it is certainly difficult to expect that this could lead to a consensus about what difficulty is and how it can be measured. However, there seems to be more agreement in that assessing the difficulty of problems and tasks is a very important issue when evaluating agents and systems (either artificial or natural) and also when designing intelligent systems, artificial intelligence artefacts, heuristics, algorithms and almost any engineered thing that requires to solve a (computational) problem.

The first application of the tools introduced in this paper is the one that motivated it: the construction of cognitive tests of different abilities (for artificial or natural systems). With the tools developed here we can now choose an environment class representing a cognitive ability, and select a subset of environments to integrate a test. This selection can now be done in a much more principled way. First, we can normalise the results of many environments, so the aggregation of results of several environments (tasks) become more meaningful. Second, we know their difficulty at any tolerance level, so we can choose those tasks that are more appropriate to the subject we want to evaluate. Third, by using their discriminating power, we can get maximum information from a single environment, and hence accelerate the evaluation. In fact, we can adapt techniques from (computerised) adaptive testing to do this. The generality of this approach makes it applicable to any kind of subject or system, since environments are a most general approach, which can be instantiated to almost any problem presentation in computer science (see \cite[ch. 3]{Legg08} for an environment hierarchy). For instance, by using appropriate environment classes we can evaluate learning abilities, planning abilities, visual abilities, deductive abilities, etc. Note that a meaningful choice of each environment is much more effective than choosing them arbitrarily, can replace or complement those evaluations based on repositories and it is certainly more effective than choosing all the environments or selecting them by a (universal) distribution, as in \cite{LeggHutter07}.

A second application is the assessment of ill-specified problems and interfaces. Throughout the paper we have been assuming that we knew the complete description of the environment or interactive task at hand. This is necessary for $\dot{K}$ (Eq. \ref{eq:dotK}), but not for the other indicators, plots and measures of difficulty based on (actions or) policies. This means that we can apply this assessment to multi-agent systems, where the behaviour or the other agents are unknown, possibly some of them being natural agents (animals or humans). This makes it possible to evaluate any environment, provided we can establish an appropriate interface, including real environments, games, robotic environments, etc. In fact, we can also check whether the use of different interfaces for the same task leads to different policy distributions, curves and indicators.

A third application is to help develop better heuristics, especially in reinforcement learning. This can be done in two ways. First, we can analyse the kinds of environments (according to their response curves) that a given heuristic can achieve. This can provide useful information about the exact problems where the heuristic is struggling and also whether there is a neat correlation with the indicators of difficulty. Second, we can customise the techniques used here by changing the probability $w$ of each policy according to some of the features of the heuristic and see how the results change. In any of both ways, the relation of our approach with Levin's search must be fundamental here.

A fourth application is a better understanding of the notion of difficulty and its relation to existing problem solvers, especially in the area of natural and artificial evolution. Since our approach is based on the evaluation of population of agents, we can specialise this population according to a given natural population (as psychometrics does) or to some theoretical population (in theoretical biology, artificial life or evolutionary computation). How the distributions and measures of difficulty evolve generation after generation (when facing the same or different problems) could provide valuable information about the evolutionary process itself. 

The contributions and applications mentioned above also carry some limitations. 
One limitation is the use of Kolmogorov complexity, which has two important problems: it is incomputable (and needs to be approximated) and it involves constants that depend on the reference machine, which may be very important (relatively) for small objects. Fortunately, there has been a significant improvement in the empirical approximation of Kolmogorov complexity \cite{delahaye2011numerical,zenilPhD}, which has also devised methods which are less sensitive to the reference machine.
A second limitation is the computational effort required to calculate the distributions. Although the plots and measures for the minimalistic setting used in this paper can be calculated in a matter of minutes with a personal computer, things may grow exponentially for more elaborated environment classes and policy languages. Nonetheless, it is important to state that these evaluations do not have to be done in real time. They can be done beforehand and the results can be stored for successive applications of each environment. This is what psychometricians do; they construct a library or catalogue of items, whose response curves have been previously estimated (not effortless), and they can use them repeatedly on demand. Since our approach is not specialised to any kind of heuristic, environments that have already been evaluated can be used for a broader range of situations, also making the initial investment more beneficial.

There are also many new things to explore, some of them related to the issues above. 
For instance, we could use better approximations to Kolmogorov complexity or some other variants (such as $Kt$, already used in \cite{HernandezOrallo-MinayaCollado98,HernandezOrallo00a}), leading to possibly more accurate views of difficulty, since time must be taken into account more explicitly from the beginning. The way in which the generation of policies is done could be rethought in order to resemble the way the coding theorem method works and get the estimation of Kolmogorov complexity (or any variant) directly, without further processing.

Other settings have to be explored, such as the use of other kinds of agents, described with different (and universal) policy languages, using other kinds of environments, etc. For instance, we are considering applying this to mazes (in comparison with \cite{Zatuchna-Bagnall09}), the matching pennies game (in relation to \cite{Turing100}) or generalisations of cellular automata, such as Random Boolean Networks \cite{gershenson2012complexity}, in a setting that would highly resemble the one introduced in \cite{HernandezOrallo10b}.

There is also an interesting work to be done in terms of indicators and graphical representations. For instance, we have used an empirical approach to the derivation of the environment response curve, but some other IRT models of difficulty and discriminating power (\cite{ferrando2009difficulty,ferrando2012discriminating}) could be adapted to our case.

Overall, while there has been an enormous effort and significant progress in understanding what complexity is and how it emerges, the question about what difficulty is and how it emerges is more elusive. We hope that this work is a step towards a better understanding and measurement of difficulty.

\section*{Acknowledgements}
The implementation of the elementary cellular automata used in the environments is based on the library `CellularAutomaton' by John Hughes for R \cite{Rproject}. I am grateful to Fernando Soler-Toscano for letting me know about their work \cite{zenil2012two} on the complexity of 2D objects generated by elementary cellular automata.

{
\bibliographystyle{plain}

\bibliography{biblio}

}

\end{document}